\documentclass{article}

\usepackage[preprint]{neurips_2026}

\usepackage{times}
\usepackage{soul}
\usepackage{url}
\usepackage[hidelinks]{hyperref}
\usepackage[utf8]{inputenc}
\usepackage[small]{caption}
\usepackage{graphicx}
\usepackage{amsmath}
\usepackage{amsthm}
\usepackage{booktabs}
\usepackage{algorithm}
\usepackage{algpseudocode} 
\usepackage{amssymb}
\usepackage{xcolor}
\usepackage{longtable}

\newcommand{\sig}[1]{\ifmmode^{#1}\else\textsuperscript{#1}\fi}
\newtheorem{theorem}{Theorem}
\newtheorem{lem}{Lemma}

\title{PLATO: Pointer Learner for Agent and Task Openness}
\author{%
  Alireza Saleh Abadi \\
  School of Computing\\
  University of Nebraska-Lincoln\\
  \texttt{asalehabadi2@huskers.unl.edu} \\
  \And
  Leen-Kiat Soh \\
  School of Computing\\
  University of Nebraska-Lincoln\\
  \texttt{lksoh@unl.edu} \\
  \And
  Daniel Alan Redder \\
  School of Computing\\
  University of Georgia\\
  \texttt{Daniel.Redder@uga.edu} \\
  \AND
  Adam Eck \\
  Oberlin College\\
  \texttt{aeck@oberlin.edu} \\
  \And
  Prashant Doshi \\
  School of Computing\\
  University of Georgia\\
  \texttt{pdoshi@uga.edu} \\
}

\begin{document}

\maketitle
\begin{abstract}
Open agent systems (OASYS) are increasingly prevalent in real-world domains where the sets of agents and tasks change unpredictably over time. Such openness, including agent openness (AO) and task openness (TO), poses a fundamental challenge to multi-agent reinforcement learning (MARL), which typically assumes fixed state and action spaces. Existing methods address openness only partially: padding and masking approaches introduce artificial bounds, while recent graph-based or hypergraph methods handle one dimension of openness but still depend on restrictive assumptions. In this paper, we introduce Pointer Learner for Agent and Task Openness (PLATO), a pointer-network-based actor combined with a centralized graph neural network (GNN) critic, trained with multi-agent proximal policy optimization under a centralized training and decentralized execution paradigm. Our pointer-based actor outputs distributions directly over the current task set. This directly supports changing action spaces without masking or retraining. Our GNN critic encodes agent–task interactions as a graph that changes shape with task and agent composition. Together, these components consider AO and TO without the boundedness of existing approaches. We formalize PLATO in a Task-and-Agent-Open Markov Game (TaAgO-MG), extending prior task-open formulations, and prove it is well-defined over the resulting unbounded state and action spaces. We evaluate PLATO with the Methods for Open Agent Systems Evaluation Initiative (MOASEI) wildfire suppression domain, an environment designed for open multi-agent system evaluation, and we demonstrate strong performance and more consistent zero-shot generalization than state-of-the-art baselines in OASYS.

\end{abstract}



\section{Introduction}\label{sec:introduction}
In many real-world multi-agent systems (MAS), the world does not remain fixed while agents are operating. New tasks emerge unpredictably, agents may leave or join mid-operation and even their goals and capabilities can change. Consider a wildfire response scenario, where fires can ignite, be suppressed, or burn out, spread unpredictably across the landscape, and where firefighters may join or leave as their suppressant runs out. This dynamic nature, known broadly as openness \citep{aimag_eck}, defies the common assumptions of multi-agent reinforcement learning (MARL), where the sets of agents and tasks are instead assumed to be fixed during training and execution. Systems with these characteristics are referred to as \emph{open} agent systems (OASYS). Two common forms of openness \citep{aimag_eck} are agent openness (AO), where agents enter or leave, and task openness (TO), where tasks appear or disappear. Our work handles AO and TO together.

Handling openness is critical for robust, scalable deployment of MARL in practical settings because systems must continue to perform even as their composition changes. The challenge is that openness causes the underlying decision problem to change during execution: agents or tasks may appear or disappear, and policies that were optimal moments earlier may no longer apply. Environments are therefore \emph{unbounded} in the quantity and composition of agents and tasks with which agents must coordinate. Practical MARL methods must therefore adapt their policies to these unbounded changes without retraining.

A few notable MARL approaches attempt to address facets of openness in MAS but none has addressed AO and TO jointly without limitations. For example, \textbf{MOHITO}~\citep{anilmohito} explicitly targets \textit{task openness} under TaO-MG; while its interaction hypergraph is expressive enough to represent AO, MOHITO does not explicitly model or evaluate handling stochastic AO. \textbf{Deep Implicit Coordination Graphs (DICG)}~\citep{dicg}, though not explicitly designed for openness, introduce flexible coordination structures that provide limited support for AO and TO but still operate within bounded and padded inputs. \textbf{Graph convolutional reinforcement learning (DGN)}~\citep{dgn} performs coordination by graph convolutions over an agent graph and can, in principle, accommodate changing agent neighborhoods, but it is not designed to model task openness and assumes a fixed action set and agent-centric graph structure.  
Taken together, these methods represent important progress, yet none jointly addresses both AO and TO without restrictive assumptions or fixed bounds on the agent–task space. 


We address these challenges with \textbf{Pointer Learner for Agent and Task Openness}, namely \textbf{PLATO}, an actor-critic MARL algorithm with centralized training and decentralized execution (CTDE) \citep{lowe2017multi}. PLATO uses a pointer network-based \citep{ptr} actor with a centralized graph neural network (GNN) \citep{gnn} critic, and we train it with multi-agent proximal policy optimization (MAPPO) \citep{mappo}.
PLATO handles TO and AO by learning a content-based pointing policy--inspired by pointer networks \citep{ptr}--in which agents point to tasks via query–key attention over only the currently available tasks, so the action distribution adapts as agents or tasks enter and leave without retraining or bounding.



The main contributions of this work are: (1) \textbf{PLATO}, the first MARL architecture to handle both AO and TO jointly via content-based pointing, without fixed bounds or retraining; (2) a \textbf{TaAgO-MG} formalization extending prior task-open formulations, with formal guarantees that PLATO is well-defined and permutation-invariant over the resulting unbounded state and action spaces (Appendix~\ref{app:well-defined} and~\ref{app:stats-perm}); (3) an \textbf{empirical evaluation} of PLATO's performance and action efficiency across AO/TO setups and grid sizes in the wildfire domain, comparing against smart and heuristic baselines; and (4) a \textbf{zero-shot generalizability assessment} by deploying trained policies on unseen grid sizes without retraining.

\section{Background}\label{sec:background}



\subsection{Open Agent Systems (OASYS)}\label{oasys}

Openness in multi-agent systems refers to environments where either agents or tasks can change over time. A recent survey \citep{aimag_eck} and earlier discussions \citep{original-omas,Calmet04,jumadinova2013strategic} identify three main forms: (1) \emph{Agent openness (AO)} occurs when agents join or leave the system, altering team composition and available capabilities; (2) \emph{Task openness (TO)} arises when tasks appear or disappear during operation, requiring agents to adapt to new or vanishing opportunities; and (3) \emph{Frame openness (FO)} refers to changes in agents' internal configurations--e.g., goals, preferences, or capabilities, redefining how agents perceive and interact in the environment. 

Among the three categories, AO has received the most attention in prior work. It was first defined over two decades ago as the addition of agents beyond the initial set \citep{original-omas}, and later refined as the complexity of transforming a system to add or remove an agent \citep{Jamroga13}. Subsequent studies showed that explicitly reasoning about agent presence improves system behavior \citep{Chandrasekaran,cohen2017open,aimag_eck,Kakarlapudi22}. 


By comparison, TO is less studied than AO, and joint AO+TO settings are rarely addressed in a single learning framework. Some model task arrivals and departures alongside agent dynamics \citep{jumadinova2013strategic,chen2016collaborative}, and in MARL, MOHITO explicitly addresses TO \citep{anilmohito}. Many domains exhibit exogenous TO while the agent set can also change online, requiring policies that remain well-defined under both forms of openness.


While FO captures deeper behavioral and cognitive variability, AO and TO make the environment structure dynamic fundamentally. Consequently, our work focuses on AO and TO, which directly determine the system’s composition and coordination patterns and leave FO for future work. 

\subsection{Task Open Markov Games}
\label{sec:tao-mg}
\noindent Task-Open Markov Games model stochastic multi-agent decision-making problems, extending Markov games \citep{Littman} for task openness \citep{anilmohito}. A TaO-MG is defined as:
$$
\text{TaO-MG} \triangleq \langle {M}, X, \Psi \rangle 
$$
\noindent where $\bullet\ M$ the base model, is a Markov game, where $N$ is the set of agents, $S$ the state space, $A_i$ the action space of agent $i$, $T: S \times \prod_{i \in N} A_i \to \Delta(S)$ the transition function, $R_i: S \times \prod_{i \in N} A_i \to \mathbb{R}$ the reward function for agent $i$, $\gamma \in [0,1)$ the discount factor, and $s_0$ the initial state. $\bullet\ X$ is the multiset of current tasks; $\bullet\ \Psi$ is a generator function whose components, $(\Psi_S,\Psi_A,\Psi_T,\Psi_R)$ for a MG, adapt $M^t$ to $M^{t+1}$ for a new set of tasks, $X^{t+1}$, at time $t+1$. 

Reinforcement learning in a TaO-MG constitutes each agent $i\in N$ learning a policy $\pi_i$ that maximizes that agent's expected sum of discounted rewards,
$$
\mathbb{E}_{\tau\sim\pi} \left [ \sum^\infty_{t=0} \gamma^t R^t_{i}(s^t,\mathbf{a}^t | s=s_0) \right ],
$$
\noindent where trajectories are in the context of the current model, $\tau=\langle s^t,\mathbf{a}^t,X^t\rangle$, and $R_i^t=\Psi_R(R_i^{t-1},X^t)$. 
This formalization makes task-driven changes explicit, but it leaves open how policies should act over a variable task set (and does not capture agent openness), motivating our TaAgO-MG extension and set-conditioned action selection in Section~\ref{sec:open-mas}.




\subsection{Pointer Networks}\label{pointer_background}

Pointer networks~\citep{ptr} generate outputs by selecting from the \emph{current} input set rather than from a fixed vocabulary, making them naturally suited to variable-sized action spaces. This is directly relevant for OASYS: as tasks appear and disappear (TO), the selectable action set changes; as agents enter and leave (AO), the team context conditioning those selections changes. PLATO uses a single-step pointer at each timestep---each agent forms a query summarizing its team context and scores each available task via additive attention~\citep{bahdanau2014neural}---yielding a distribution whose support matches the current task set. The formal mechanism is detailed in Section~\ref{actor}.

\noindent

\section{Open Agent-Task in MAS}\label{sec:open-mas}
In this section, we define the \textbf{Task–Agent-Open Markov Game (TaAgO-MG)}, an extension of TaO-MG (see Section 2.2), to consider both AO and TO. Since both the task set and the agent set may change over time, the feasibility and utility of each task-action depend on the current tasks and team composition. Thus, a practical policy should evaluate the currently available tasks and the current team context, rather than assume fixed action labels with permanent meaning. 

In the following, we instantiate the Markov game using the popular wildfire suppression domain \citep{Chandrasekaran,Kakarlapudi22,anilmohito,eck2020scalable,patino2025inaugural} that exhibits both TO and AO. Briefly, the environment is a grid where firefighter agents are stationed in fixed positions and each fire is a task with a location and size \citep{patino2025inaugural}. Fires may ignite (e.g., from lightning or spread), burn out, or become extinguished by firefighters, changing the set of active tasks and creating \emph{task openness}. At each timestep, an agent chooses to either fight a fire or remain idle (\textsc{no-op}); it may also exit the environment when out of suppressant, creating \emph{agent openness}. Rewards also vary: larger fires are more challenging and rewarding, new fires create suppression opportunities, while burnouts and spread alter potential gains and losses.

Consider:
\[
\text{TaAgO-MG} \triangleq \langle M, X, N, \Psi \rangle,
\]
where $N$ and $X$ are the changing sets of agents and tasks, respectively, rather than a fixed set. We extend the generator, $\Psi$ from MOHITO \citep{anilmohito} to also cover $N$. Specifically, we treat $\Psi$ as a generator that takes the base game $M$ and the current open sets $(X^t,N^t)$ and returns the time-indexed game, i.e., $M^t := \Psi(M, X^t, N^t)$. This lets us account for arbitrary deterministic change which new agents entail for the game. Formally the generator is defined as
$$
M^t=\big\langle \Psi_S(M,X^t,N^t),\Psi_A(M,X^t,N^t),\Psi_T(M,X^t,N^t),\Psi_R(M,X^t,N^t)\big\rangle .
$$
$\Psi_S$ updates the state, $\Psi_A$ updates the action set, 
$\Psi_T$ generates the transition function, and $\Psi_R$ generates the reward functions.
This decomposition allows the environment to update all components when the task set or agent set changes. Now, let $A_i^t$ denote the action space of agent $n_i$ at time $t$, and let \textsf{suppress}$(x)$ denote the action of suppressing fire $x \in X^t$. (Note that in our wildfire instantiation, each task has one action, so selecting a task uniquely determines the non-\textsc{no-op} action.)

\textbf{Task Arrival.}
If a new fire appears at time $t$, $X^{t+1} = X^t \cup \{x_{\text{new}}\}$. The generator induces the action update, $A_{i}^{t+1} = A_{i}^t
\cup
\{\textsf{suppress}(x_{\text{new}})\},
\qquad \forall n_i \in N^t$, corresponding to $\Psi_A(M,X^{t+1},N^t)$.

\textbf{Agent Departure.}
If agent $n_i$ depletes its suppressant and exits, $N^{t+1} = N^t \setminus \{n_i\}$. The generator produces the updated game components: 
$S^{t+1} = S^t \setminus \{n_i\}$ and $R^{t+1} := \{R_j^{t}\}_{n_j \in N^{t+1}}$, which correspond to $\Psi_S(M,X^{t+1},N^{t+1})$ and
$\Psi_R(M,X^{t+1},N^{t+1})$, respectively.

\section{Methodology}\label{sec:methodology}

PLATO is a CTDE method for multi-agent systems with TO and AO, as shown in Figure~\ref{fig:architecture}. PLATO's core novelty is reformulating action selection from fixed-index classification to \emph{content-based pointing}: each agent forms a context vector (query) and selects among the currently available tasks by matching this query against task feature vectors (keys). The policy output size is determined by the current task set, so tasks and agents can appear or disappear without retraining or architectural changes, keeping the policy well-defined as openness changes the environment. 

\begin{figure}[htbp]
  \centering
  \includegraphics[width=0.7\linewidth]{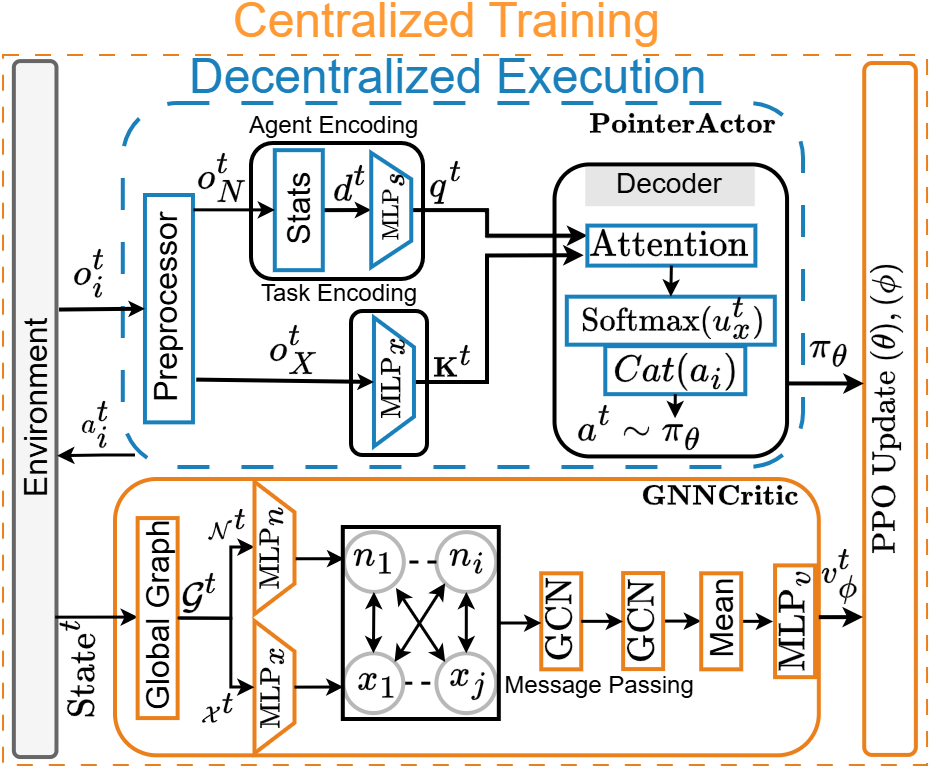}
  \caption{\small \textbf{PLATO architecture for OASYS.} It uses a pointer actor approach and CTDE training. 
  }
\label{fig:architecture}
\end{figure}

\subsection{Pointer Actor}\label{actor}
The pointer actor is the decentralized policy used at execution time. It supports TO by producing a policy whose support matches the currently available tasks, and supports AO by conditioning decisions on a permutation-invariant summary of the current team. The actor (1) deterministically factors $o_i^t$ into an agent-feature matrix and a task-feature matrix, (2) encodes the current agent set into a fixed-dimensional query, (3) encodes each available task into a key, and (4) applies additive cross-attention to obtain a categorical distribution over the available task-actions. 

\subsection{Observation Factorization}\label{factorization}  In TaAgO-MG, the observation of the $i^{\text{th}}$ agent at time $t$, $o_i^t$, is its local view of the global state, $\mathrm{State}^t$, restricted to the current agent set $N^t$ and task set $X^t$ that are visible or relevant to agent $i$.  
The preprocessor factors $o_i$ into an agent-feature matrix $o_N^t \in \mathbb{R}^{|N^t|\times f_N}$ and a task-feature matrix $o_X^t \in \mathbb{R}^{|X^t|\times f_X}$. In these feature matrices, rows correspond to $i$'s observation of tasks and other agents, similar to node factorization for MOHITO's critic graph \citep{anilmohito}. 


\begin{algorithm}[t]
\caption{\textsc{PLATO: Pointer Actor, GNN Critic and training}}
\label{alg:plato}
\footnotesize
\begin{algorithmic}[1]
\State \textbf{Inputs:} env $e$, horizon $T$, actor $\pi_\theta$, critic $V_\phi$, initial observation $O^0$
\For{$t=1\ldots T$}
  \State $a^t_i\sim \textsc{PointerActor}(O_i^t) \quad \forall i \in N^t$
  \State $O^t,\text{State}^t,r^t\gets e(a^t)$
  \State $\hat V^t \gets \textsc{GNNCritic}(\text{State}^t;\phi)$;\; store $(\cdot)$
  \State compute GAE~\citep{gae} $\hat A^t$;\; update $\theta$ by PPO-clip;\; update $\phi$ by MSE$(\hat V^t,\hat R^t)$
\EndFor
\end{algorithmic}
\noindent\rule{\linewidth}{0.3pt}
\begin{minipage}[t]{0.52\linewidth}
\begin{algorithmic}[1]
\setcounter{ALG@line}{7}
\Procedure{\textsc{PointerActor}}{$O_i^t$}
  \State $o_N^t,o_X^t\gets \textsc{Preprocessor}(O_i^t)$
  \State $d^t\!\gets\![\mu,\mathrm{Var},\min,\max](o^t_{N})$
  \State $q^t \gets \textsc{MLP}_s(d^t)$
  \State $K^t \gets \{\mathrm{MLP}_x(o) \mid o \in o_X^t\} \cup \{\mathrm{MLP}_x(o_{\textsc{no-op}})\}$
  \State $u_x^t\!\gets\!\mathbf{v}^\top\tanh(\mathbf{W}_K k_x^t+\mathbf{W}_q q^t) \quad \forall k_x^t \in K^t$
  \State \textbf{return} $\mathrm{Cat}(\mathrm{softmax}([u_x^t]_{x \in X^t \cup \{\textsc{no-op}\}}))$
\EndProcedure
\end{algorithmic}
\end{minipage}%
\hfill
\begin{minipage}[t]{0.46\linewidth}
\begin{algorithmic}[1]
\setcounter{ALG@line}{15}
\Procedure{\textsc{GNNCritic}}{$\text{State}^t$}
  \State form $\mathcal{G}^t=\langle\mathcal{X}^t,\mathcal{N}^t,\mathcal{E}^t\rangle$
  \State $h_0^t \gets \text{Concat}\left[\text{MLP}_{\mathcal{N}}\!\left(\mathcal{N}^t\right),\text{MLP}_\mathcal{X}\!\left(\mathcal{X}^t\right)\right]$
  \State $h^t_l \gets \textsc{GCN}(h^t_{l-1}, \mathcal{E}^t) \quad \text{for each layer }l$
  \State $M^{(L)}\!\gets\!\frac{1}{|\mathcal{N}^t|+|\mathcal{X}^t|}\sum_i(h^t_L + h^t_0)_i$ \label{line:critic_pool}
  \State \textbf{return} $V=W_v\,\mathrm{MLP}_v(M^{(L)})+b_v$ \label{line:critic_value}
\EndProcedure
\end{algorithmic}
\end{minipage}
\end{algorithm}

\textbf{Preprocessor.}\label{actor:preprocessor}
Each agent $i$ receives a local observation $o_i^t \in O^t$ covering the agents and tasks relevant to $i$ at time $t$. A deterministic preprocessor factors this into an agent-feature matrix and a task-feature matrix:
$$
    o^t_N \in \mathbb{R}^{|N^t|\times f_N}, \quad o^t_X \in \mathbb{R}^{|X^t|\times f_X}.
$$
\textbf{Encoder.}\label{actor:encoder}
This step constructs one \emph{query} from the agent set and a variable-sized set of \emph{keys} from the task set. We use MLP encoders since the actor makes a single decision per timestep from fully observable features; sequential decoding is not required.
 We address AO by aggregating four statistics (i.e., mean, variance, minimum, and maximum) over agents in $o^t_N$ to produce a fixed-length vector $d^t$ (Alg.~\ref{alg:plato}, line~10) and then encoding it as $q^t = \mathrm{MLP}_s(d^t)$ (Alg.~\ref{alg:plato}, line~11).
 In Lemma~\ref{lem:stats}, we show the encoder is permutation invariant to agent order (proof in Appendix~\ref{app:stats-perm}).
 We use multiple statistics to reduce aliasing, where different agent sets could otherwise produce identical representations. We address TO through attention. We construct $|X^t|+1$ keys $K^t$ from $o^t_X$ (Alg.~\ref{alg:plato}, line~12) in the task encoder: one key $k_x^t = \mathrm{MLP}_x(o_x^t)$ per task, plus a \textsc{no-op} key. We assume one action per task with a persistent \textsc{no-op}, so $a_i^t=\{\textsc{no-op}\}\cup\{a_x\mid x\in X^t\}$. 
\begin{lem}[Encoder Permutation Invariance]
    \label{lem:stats}
    Each team statistic is permutation invariant, $MLP_s$ operates on the aggregate, and $MLP_x$ operates on task observations independently. Therefore, the agent and task encodings are permutation invariant.
\end{lem}
\paragraph{Pointer Decoder.}\label{actor:decoder} The Pointer Decoder is a one-step variant of a pointer network \citep{ptr} that maps the \emph{query} and \emph{keys} to a stochastic policy over currently available \emph{task-actions}.  We use one-step, additive attention \citep{bahdanau2014neural} to decode $(q^t,k_x^t)$ to $u_x^t$, a scalar attention score (Alg.~\ref{alg:plato}, line~13), in which, \(\mathbf{W}_K\) and \(\mathbf{W}_q\) are learnable projection matrices, and
\(\mathbf{v}\) is a learnable vector projecting the compatibility space to a scalar score. We then normalize over all attention scores using SOFTMAX to obtain the policy (Alg.~\ref{alg:plato}, line~14).  

\subsection{GNN Critic} \label{critic}
The critic is centralized during training and estimates $V_\phi(s^t)$ under CTDE. It represents the open global state as a bipartite agent--task graph whose node sets vary with $|N^t|$ and $|X^t|$. Agent and task features are embedded into a shared latent space, message passing propagates relational information across agent--task edges, and pooling yields a fixed-size graph summary that is mapped to a scalar value. This avoids fixed-size concatenation and supports variable numbers of agents and tasks.


$\mathcal{G}^t = \langle\mathcal{X}^t, \mathcal{N}^t, \mathcal{E}^t\rangle$
encodes the state, $\text{State}^t$, at time $t$, where
(1) $\mathcal{X}^t \in \mathbb{R}^{|X^t|\times f_X}$ are \emph{task nodes} and contain $f_X$ many features, describing each task at time $t$; (2) $\mathcal{N}^t \in \mathbb{R}^{|N^t|\times f_N}$ are \emph{agent nodes} and contain $f_N$ many features, describing each agent at time $t$; and (3) 
$\mathcal{E}^t$ includes undirected edges between all agent-task pairs which have an action, $\{\langle n,x\rangle | n\in \mathcal{N}^t \text{ and }x\in \mathcal{X}^t\}$.  


We build a critic network using a GCN \citep{kipf2016semi} to propagate information between agent and task nodes by message passing. First, we use MLP encoders to map $\mathcal{N}^t$ and $\mathcal{X}^t$ into a shared latent space and concatenate them to form $\mathbf{h}_0^t$ (Alg.~\ref{alg:plato}, line~18), where $\mathbf{h}_0^t$  is the embedding of all nodes in $\mathcal{G}^t$ and $f_h$ is the hidden dimension. We then perform message passing (Alg.~\ref{alg:plato}, line~19) with a skip connection \citep{he2016deep} that adds $\mathbf{h}_0^t$ to the final layer output before pooling, preserving individual node features that would otherwise be smoothed out by message passing. We use two layers to allow information to propagate between agents and tasks through the bipartite graph.
Finally, we map the node embeddings to a scalar value by mean pooling over nodes followed by an $\mathrm{MLP}_v$ value head (Alg.~\ref{alg:plato}, lines~\ref{line:critic_pool}--\ref{line:critic_value}). We outline in Theorem~\ref{theorem:well-defined} and show in Appendix~\ref{app:well-defined} that the actor and critic are well defined functions over the open task and agent spaces.

\begin{theorem}[PLATO is well defined]
    \label{theorem:well-defined}
    The actor and critic are permutation invariant, as shown in Lemma~\ref{lem:stats} and Appendix~\ref{app:well-defined}. The MLPs act on observations independently or on aggregates. The pointer decoder and critic GCNs are defined over a single unbounded dimension. Therefore, the actor and critic are well defined functions over the countably infinite open agent--task space.
\end{theorem}

\subsection{Training}\label{training}

We train PLATO under CTDE using MAPPO (\citep{mappo,ppo}) with a centralized critic and a single shared policy, while at execution time only the decentralized pointer actor is used. The actor is optimized with the PPO clipped objective (\citep{ppo}). The critic uses the standard PPO clipped value regression loss, and we compute advantages with generalized advantage estimation (GAE) \citep{gae} (Alg.~\ref{alg:plato}, line~6). All actor ($\theta = \{MLP_S, MLP_X, \mathbf{W}_K, \mathbf{W}_q\}$) and critic ($\phi = \{MLP_{\mathcal{X}}, MLP_{\mathcal{N}}, MLP_v, \text{GCNs}\}$) components are standard differentiable modules optimized jointly. The per-timestep training cost holds all neural network dimensions constant and scales only with the number of active agents $|N^t|$ and active tasks $|X^t|$, as formalized in Theorem~\ref{thm:analysis} (proof details in Appendix~\ref{app:complexity}).

\begin{theorem}[Complexity Analysis for PLATO]
    \label{thm:analysis}
    Treating all neural network dimensions and the number of GCN layers as constants, the per-timestep cost of one PLATO training step---actor forward pass over all $|N^t|$ agents and critic forward pass combined---is $\mathcal{O}(|N^t|^2+|N^t||X^t|)$.
\end{theorem}

\section{Experiments}\label{sec:experiments}
We evaluate PLATO against three smart baselines and four naive heuristic-based baselines across TO and AO configurations on $2\times 3$, $3\times 3$, and $4\times 4$ grids in our experiments. 
Note that we distinguish \textit{native} evaluation (training and testing on the same grid) from \textit{zero-shot} evaluation (deploying a trained policy on an unseen grid without retraining); see Appendix~\ref{app:Setup_figures} and~\ref{app:zs} for setup illustrations and zero-shot details. Additional native and zero-shot results are in Appendix~\ref{Detailed_Results}.




\textbf{Domain and Setups.}\label{sec:to_ao_settings}
We use the \textsc{Wildfire} domain~\citep{patino2025inaugural} (Appendix~\ref{wildfire}): a $r\times c$ grid where firefighter agents at fixed positions suppress fires under finite suppressant. \emph{Endogenous TO} arises from fire \emph{spread} (existing fires probabilistically ignite neighbors); \emph{exogenous TO} from \emph{random ignitions}; and \emph{AO} from stochastic suppressant depletion triggering an exit-and-refill cycle. At each step, agents may suppress an in-range fire or select \textsc{no-op}; extinguishing a fire of size $s$ yields $+2^s$ while a burnout costs $-2^s$. We define four setups (S0--S3) over a fixed 100-step horizon (Table~\ref{tab:setup_summary_app}, Appendix~\ref{app:domain_setups}): S0 is a \emph{closed} baseline (no spread, no ignitions, unlimited suppressant); S1 adds \emph{endogenous TO} (spread prob.\ $0.6$, refill $1.0$); S2 introduces \emph{exogenous TO} (ignition prob.\ $0.6$, refill $0.8$); S3 combines \emph{both with refill} $0.6$, the most demanding joint AO+TO condition. Openness event counts and grid configuration details are in Appendix~\ref{app:A1} and~\ref{app:configuration_tables}.


%

\textbf{Models and Baselines.}\label{sec:models_metrics} Our method is \textbf{PLATO}, as defined in Algorithm~\ref{alg:plato}.
The three smart baselines are: \textbf{DGN}~\citep{dgn} and \textbf{DICG}~\citep{dicg} represent state-of-the-art MARL methods for dynamic coordination but assume a fixed set of agents and tasks during training; and \textbf{MOHITO}~\citep{anilmohito} is the closest prior method explicitly designed for TO.  We also use four naive baselines to contextualize gains from coordination versus simple rule-based behavior: \textbf{NOOP} always selects no-action, providing a lower bound; \textbf{Random} samples uniformly from available actions; \textbf{Strongest} prioritizes the highest-intensity fires, representing an urgency-driven heuristic; and \textbf{Weakest} prioritizes the lowest-intensity fires, reflecting an immediacy-driven strategy.

\textbf{Metrics.}
We evaluate policy performance using \textbf{episode return (average cumulative reward per episode)} and policy efficiency using \textbf{reward per fight action}.
 Charts report mean$\pm$std across 150 evaluation episodes (50 seeds per checkpoint, 3 checkpoints per method). We determine significance with a Shapiro-Wilk-adaptive two-sided test~\citep{wilcoxon1945individual,shapiro1965analysis} (two-sided paired $t$-test or Wilcoxon signed-rank, selected by Shapiro-Wilk normality pre-test at the same Bonferroni-corrected $\alpha$; see Appendix~\ref{sec:stat_testing}), with Bonferroni correction~\citep{hochberg1987multiple} ($\alpha=0.05/\text{methods}$). We report the best results for each learned policy from Optuna hyperparameter tuning \citep{akiba2019optuna}. Details on tuning, seeding, and compute resources are in Appendix~\ref{app:reprod} and~\ref{sec:convergence_appendix} (Table~\ref{tab:bonferroni_families}).


\section{Results}\label{sub:Results}


\par{\textbf{Performance}.}
\label{par:performance}
Table~\ref{tab:native_reward} shows PLATO achieves the highest return in S2 and S3 on the $3\times 3$ grid with statistical significance, while performing as well as top baselines in S1. On $4\times 4$, PLATO outperforms all baselines in S1--S3 with statistical significance. In S0 (no openness), all smart methods match. Results on $2\times 3$ are in Appendix~\ref{Detailed_Results}.

\begin{table*}[!t]
\centering
\setlength{\tabcolsep}{1.4pt}
\renewcommand{\arraystretch}{1.02}
\resizebox{\linewidth}{!}{%
\begin{tabular}{@{\hspace{2pt}}l|r|r|r|r||r|r|r|r@{\hspace{2pt}}}
\toprule
 & \multicolumn{4}{c||}{trained and tested on $3\times3$} & \multicolumn{4}{c}{trained and tested on $4\times4$} \\
\cmidrule(lr){2-5}\cmidrule(lr){6-9}
Model & \multicolumn{1}{c|}{S0} & \multicolumn{1}{c|}{S1} & \multicolumn{1}{c|}{S2} & \multicolumn{1}{c||}{S3} & \multicolumn{1}{c|}{S0} & \multicolumn{1}{c|}{S1} & \multicolumn{1}{c|}{S2} & \multicolumn{1}{c}{S3} \\
\midrule
PLATO    & 8.00$\pm$0.00 & 7.52$\pm$1.46  & \textbf{112.83$\pm$12.62} & \textbf{80.53$\pm$19.66} & 12.00$\pm$0.00 & \textbf{19.47$\pm$9.33} & \textbf{228.81$\pm$21.99} & \textbf{149.35$\pm$27.33} \\
DICG     & 8.00$\pm$0.00 & 5.97$\pm$2.53  & 59.75$\pm$13.69           & 40.83$\pm$15.90          & 12.00$\pm$0.00 & 12.11$\pm$3.95          & 128.63$\pm$21.97          & 74.39$\pm$23.54           \\
DGN      & 7.63$\pm$1.17 & 5.52$\pm$2.81  & 44.59$\pm$15.02           & 31.56$\pm$16.62          & 11.79$\pm$0.90 & 10.83$\pm$4.83          & 82.99$\pm$20.57           & 64.16$\pm$19.27           \\
MOHITO   & 8.00$\pm$0.00 & 3.04$\pm$3.64  & 42.36$\pm$19.89           & 47.24$\pm$18.24          & 12.00$\pm$0.00 & 6.11$\pm$6.36           & 41.97$\pm$28.71           & 33.91$\pm$19.85           \\
\midrule
NOOP     & $-$8.00$\pm$0.00 & $-$8.00$\pm$0.00 & $-$12.00$\pm$0.00 & $-$12.00$\pm$0.00 & $-$12.00$\pm$0.00 & $-$12.00$\pm$0.00 & $-$16.00$\pm$0.00 & $-$18.00$\pm$0.00 \\
Random   & 6.00$\pm$2.46  & 4.48$\pm$3.49  & 11.48$\pm$16.06  & 0.88$\pm$12.96   & 9.20$\pm$3.16  & 3.20$\pm$4.44  & 49.76$\pm$21.11  & 35.52$\pm$22.49 \\
Weakest  & 8.00$\pm$0.00  & 7.76$\pm$0.96  & 70.36$\pm$22.66  & 37.44$\pm$1.72   & 12.00$\pm$0.00 & 10.36$\pm$3.12 & 137.36$\pm$20.15 & 49.40$\pm$0.93  \\
Strongest& 8.00$\pm$0.00  & 7.28$\pm$1.75  & 24.24$\pm$29.13  & 13.92$\pm$20.38  & 12.00$\pm$0.00 & 7.68$\pm$4.75  & 161.24$\pm$12.37 & 80.64$\pm$11.78 \\
\bottomrule
\end{tabular}%
}
\caption{Performance: Episode return (mean $\pm$ std). 
Bold indicates statistically significant best performance within each setup (Shapiro-Wilk-adaptive two-sided test; Bonferroni-corrected, $p{<}0.00625$).}
\label{tab:native_reward}
\end{table*}

\par{\textbf{Efficiency}.}
\label{par:efficiency}
Table~\ref{tab:native_rpf} shows PLATO attains the highest reward per fight on the $3\times 3$ grid in S0, S2, and S3 with statistical significance; S1 is led by \textbf{Weakest}. On $4\times 4$, PLATO leads S2--S3, MOHITO leads S0, and DICG leads S1, all with statistical significance. PLATO remains positive across all grids and setups. Results on $2\times 3$ are in Appendix~\ref{Detailed_Results}.

\begin{table*}[!t]
\centering
\setlength{\tabcolsep}{1.4pt}
\renewcommand{\arraystretch}{1.02}
\resizebox{\linewidth}{!}{%
\begin{tabular}{@{\hspace{2pt}}l|r|r|r|r||r|r|r|r@{\hspace{2pt}}}
\toprule
 & \multicolumn{4}{c||}{trained and tested on $3\times3$} & \multicolumn{4}{c}{trained and tested on $4\times4$} \\
\cmidrule(lr){2-5}\cmidrule(lr){6-9}
Model & \multicolumn{1}{c|}{S0} & \multicolumn{1}{c|}{S1} & \multicolumn{1}{c|}{S2} & \multicolumn{1}{c||}{S3} & \multicolumn{1}{c|}{S0} & \multicolumn{1}{c|}{S1} & \multicolumn{1}{c|}{S2} & \multicolumn{1}{c}{S3} \\
\midrule
PLATO    & \textbf{1.25$\pm$0.23} & 0.60$\pm$0.38          & \textbf{0.74$\pm$0.08} & \textbf{0.64$\pm$0.10} & 1.09$\pm$0.25          & 0.41$\pm$0.21          & \textbf{0.73$\pm$0.06} & \textbf{0.66$\pm$0.07} \\
DICG     & 0.73$\pm$0.39          & 0.35$\pm$0.21          & 0.45$\pm$0.09          & 0.37$\pm$0.16          & 0.71$\pm$0.35          & \textbf{0.56$\pm$0.22} & 0.51$\pm$0.07          & 0.41$\pm$0.11          \\
DGN      & 0.61$\pm$0.41          & 0.48$\pm$0.46          & 0.40$\pm$0.11          & 0.31$\pm$0.19          & 0.44$\pm$0.34          & 0.42$\pm$0.23          & 0.43$\pm$0.08          & 0.39$\pm$0.13          \\
MOHITO   & 0.97$\pm$0.12          & 0.25$\pm$0.39          & 0.23$\pm$0.59          & 0.39$\pm$0.16          & \textbf{1.21$\pm$0.10} & 0.27$\pm$0.28          & 0.25$\pm$0.33          & 0.24$\pm$0.33          \\
\midrule
NOOP     & $-$8.00$\pm$0.00 & $-$8.00$\pm$0.00 & $-$12.00$\pm$0.00 & $-$12.00$\pm$0.00 & $-$12.00$\pm$0.00 & $-$12.00$\pm$0.00 & $-$16.00$\pm$0.00 & $-$18.00$\pm$0.00 \\
Random   & 0.73$\pm$0.36          & 0.38$\pm$0.39          & 0.03$\pm$0.49          & $-$0.22$\pm$0.47       & 0.81$\pm$0.38          & 0.12$\pm$0.19          & 0.36$\pm$0.22          & 0.29$\pm$0.20          \\
Weakest  & 1.03$\pm$0.13          & \textbf{0.87$\pm$0.44} & 0.48$\pm$0.07          & 0.35$\pm$0.02          & 1.07$\pm$0.32          & 0.35$\pm$0.17          & 0.54$\pm$0.04          & 0.42$\pm$0.01          \\
Strongest& 1.03$\pm$0.13          & 0.71$\pm$0.57          & $-$0.10$\pm$0.62       & $-$0.16$\pm$0.57       & 1.07$\pm$0.32          & 0.18$\pm$0.12          & 0.53$\pm$0.04          & 0.41$\pm$0.03          \\
\bottomrule
\end{tabular}%
}
\caption{Efficiency: Reward per fight (mean $\pm$ std). 
Bold indicates statistically significant best performance within each setup (Shapiro-Wilk-adaptive two-sided test; Bonferroni-corrected, $p{<}0.00625$).}
\label{tab:native_rpf}
\end{table*}

\par{\textbf{Generalizability}.}
\label{par:generalizability}
PLATO’s actor and critic are proved well-defined over any unbounded state space (Appendix~\ref{app:well-defined}). Here we empirically evaluate zero-shot generalization by deploying learned policies on a different unseen grid without retraining, where the same team face different fires and novel positions (see Appendix~\ref{app:zs}). Table~\ref{tab:zs_reward} reports the zero-shot episode returns for both training sizes. When trained on $2\times 3$, PLATO leads S1--S2 on $3\times 3$ and S2 on $4\times 4$ with statistical significance; DICG leads S1 and S3 at $4\times 4$. MOHITO collapses to negative returns in S2 at $4\times 4$. When trained on $3\times 3$, PLATO leads S1--S3 on the $5\times 5$ grids with statistical significance; DICG leads S1--S2 and MOHITO leads S0 at $4\times 4$. DGN degrades consistently under openness. Additional zero-shot results are in Appendix~\ref{sec:zeroshot_appendix}.

\begin{table*}[!t]
\centering
\setlength{\tabcolsep}{1.4pt}
\renewcommand{\arraystretch}{1.02}
\resizebox{\linewidth}{!}{%
\begin{tabular}{@{\hspace{2pt}}l|r|r|r|r||r|r|r|r@{\hspace{2pt}}}
\toprule
 & \multicolumn{4}{c||}{trained on $2\times3$ and tested (zero-shot) on $3\times3$} & \multicolumn{4}{c}{trained on $2\times3$ and tested (zero-shot) on $4\times4$} \\
\cmidrule(lr){2-5}\cmidrule(lr){6-9}
Model & \multicolumn{1}{c|}{S0} & \multicolumn{1}{c|}{S1} & \multicolumn{1}{c|}{S2} & \multicolumn{1}{c||}{S3} & \multicolumn{1}{c|}{S0} & \multicolumn{1}{c|}{S1} & \multicolumn{1}{c|}{S2} & \multicolumn{1}{c}{S3} \\
\midrule
PLATO  & 9.97$\pm$0.33  & \textbf{6.83$\pm$2.87} & \textbf{80.75$\pm$12.65} & 4.95$\pm$28.12    & 11.73$\pm$2.51 & 5.31$\pm$9.89           & \textbf{64.11$\pm$15.58} & 13.68$\pm$36.62 \\
DICG   & 10.00$\pm$0.00 & 3.44$\pm$3.15          & 60.63$\pm$15.37          & 25.31$\pm$18.21   & 11.27$\pm$1.97 & \textbf{8.17$\pm$10.17} & 59.04$\pm$16.85          & \textbf{48.69$\pm$18.48} \\
DGN    & 7.07$\pm$2.77  & 2.91$\pm$2.46          & 51.84$\pm$12.29          & 27.16$\pm$18.51   & 4.11$\pm$3.64  & 4.23$\pm$8.86           & 39.61$\pm$25.76          & 44.88$\pm$21.18 \\
MOHITO & 3.87$\pm$2.88  & 0.96$\pm$3.12          & 24.39$\pm$26.92          & $-$4.63$\pm$13.48 & 3.41$\pm$2.08  & $-$1.65$\pm$4.33        & $-$13.17$\pm$9.09        & 23.24$\pm$22.84 \\
\specialrule{1.5pt}{0pt}{0pt}
 & \multicolumn{4}{c||}{trained on $3\times3$ and tested (zero-shot) on $4\times4$} & \multicolumn{4}{c}{trained on $3\times3$ and tested (zero-shot) on $5\times5$} \\
\cmidrule(lr){2-5}\cmidrule(lr){6-9}
Model & \multicolumn{1}{c|}{S0} & \multicolumn{1}{c|}{S1} & \multicolumn{1}{c|}{S2} & \multicolumn{1}{c||}{S3} & \multicolumn{1}{c|}{S0} & \multicolumn{1}{c|}{S1} & \multicolumn{1}{c|}{S2} & \multicolumn{1}{c}{S3} \\
\midrule
PLATO  & 9.65$\pm$2.70           & 5.23$\pm$9.58           & 49.17$\pm$25.38          & 57.36$\pm$14.29 & 7.76$\pm$2.98 & \textbf{2.48$\pm$6.17} & \textbf{80.79$\pm$21.55} & \textbf{50.40$\pm$26.18} \\
DICG   & 10.00$\pm$1.47          & \textbf{8.13$\pm$8.42} & \textbf{61.88$\pm$13.04} & 51.36$\pm$16.28 & 9.23$\pm$1.85 & 0.07$\pm$4.47          & 41.51$\pm$18.04          & 21.51$\pm$19.15          \\
DGN    & 6.80$\pm$3.88           & $-$0.57$\pm$5.45        & 47.85$\pm$19.51          & 35.93$\pm$26.71 & 6.69$\pm$3.78 & $-$2.80$\pm$4.19       & 28.29$\pm$14.36          & 14.21$\pm$18.18          \\
MOHITO & \textbf{11.20$\pm$1.84} & $-$1.72$\pm$6.56        & 34.68$\pm$30.94          & 55.77$\pm$15.25 & 9.52$\pm$1.95 & $-$0.43$\pm$4.14       & 40.79$\pm$22.12          & 25.36$\pm$21.85          \\
\bottomrule
\end{tabular}%
}
\caption{Generalizability: Episode return (mean $\pm$ std), zero-shot transfer. Bold indicates statistically significant best performance within each setup (Shapiro-Wilk-adaptive two-sided test; Bonferroni-corrected, $p{<}0.0125$).}
\label{tab:zs_reward}
\end{table*}

\textbf{Ablation.} For our ablation studies, we test four variants from a factorial design crossing two encoder choices (MLP, LSTM) with two scorer choices (additive, dot-product):  \textbf{(mlp+add)}, \textbf{(mlp+dot)}, \textbf{(lstm+add)}, and \textbf{(lstm+dot)}. Table~\ref{tab:ablation_reward} reports the native and zero-shot episode-return results. Additional metrics are in Appendix~\ref{sec:ablation_appendix}. \emph{Q1: Does temporal memory in the encoder improve performance?}
We compare \textbf{(mlp+add) vs.\ (lstm+add)} and \textbf{(mlp+dot) vs.\ (lstm+dot)} to isolate the encoder across both scoring functions. In native evaluation, there are no statistically significant differences between LSTM and MLP variants under either scorer. In zero-shot transfer, the encoder's impact depends on the scorer and setup. With additive scoring, (lstm+add) achieves statistically significantly higher returns than (mlp+add) in S3 across both target grids; 
(lstm+dot) is also numerically higher than (mlp+dot) in S3 
but without statistical significance. 
Together, these results show that the benefit of temporal memory (LSTM) is not universal: it is most pronounced under combined AO+TO (S3) in zero-shot settings when paired with additive scoring.  \emph{Q2: How does additive attention scoring perform with respect to dot-product scoring?} 
We compare \textbf{(mlp+add) vs.\ (mlp+dot)} and \textbf{(lstm+add) vs.\ (lstm+dot)}. In native evaluation, no scorer achieves statistically significant advantage under either encoder. In zero-shot transfer, the scorer's advantage also depends on the encoder and setup. With the LSTM encoder, (lstm+add) achieves statistically significantly higher returns than (lstm+dot) in S3 across both target grids 
— additive scoring wins in the most open zero-shot setting. With the MLP encoder, (mlp+dot) achieves statistically significantly higher returns than (mlp+add) in S0 at $4\times4$ 
— dot-product scoring wins in the simplest zero-shot setting. Taken together with Q1, these results point to a consistent interaction: (lstm+add) is the best combination under combined AO+TO (S3), while (mlp+dot) is the best combination under closed environment (S0) in zero-shot transfer.

\begin{table*}[!t]
\centering
\setlength{\tabcolsep}{1.4pt}
\renewcommand{\arraystretch}{1.02}
\resizebox{\linewidth}{!}{%
\begin{tabular}{@{\hspace{2pt}}l|r|r|r|r||r|r|r|r@{\hspace{2pt}}}
\toprule
 & \multicolumn{4}{c||}{trained and tested on $3\times3$} & \multicolumn{4}{c}{trained and tested on $4\times4$} \\
\cmidrule(lr){2-5}\cmidrule(lr){6-9}
\shortstack[l]{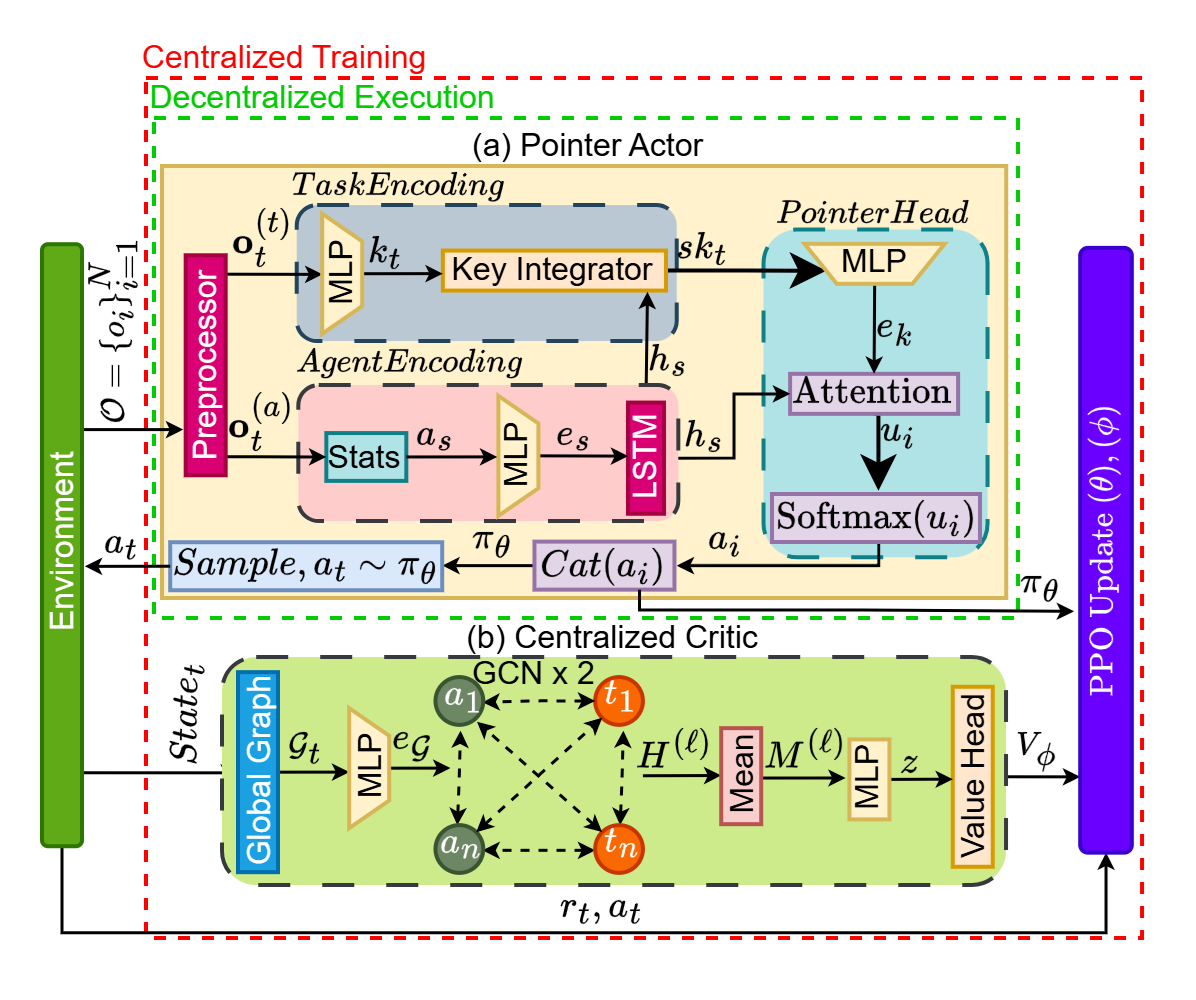} & \multicolumn{1}{c|}{S0} & \multicolumn{1}{c|}{S1} & \multicolumn{1}{c|}{S2} & \multicolumn{1}{c||}{S3} & \multicolumn{1}{c|}{S0} & \multicolumn{1}{c|}{S1} & \multicolumn{1}{c|}{S2} & \multicolumn{1}{c}{S3} \\
\midrule
(mlp+add)      & 8.00$\pm$0.00 & 7.52$\pm$1.46 & 112.83$\pm$12.62 & 80.53$\pm$19.66 & 12.00$\pm$0.00 & 19.47$\pm$9.33 & 228.81$\pm$21.99 & 149.35$\pm$27.33 \\
(mlp+dot)      & 8.00$\pm$0.00 & 7.28$\pm$1.54 & 109.65$\pm$13.60 & 84.37$\pm$20.79 & 12.00$\pm$0.00 & 19.76$\pm$9.51 & 221.04$\pm$24.03 & 146.63$\pm$25.99 \\
(lstm+add)     & 8.00$\pm$0.00 & 7.47$\pm$1.36 & 108.56$\pm$14.33 & 82.12$\pm$22.23 & 12.00$\pm$0.00 & 16.12$\pm$6.75 & 227.59$\pm$24.53 & 151.33$\pm$24.35 \\
(lstm+dot)     & 8.00$\pm$0.00 & 7.20$\pm$1.61 & 113.67$\pm$15.30 & 88.16$\pm$17.87 & 12.00$\pm$0.00 & 19.49$\pm$8.57 & 218.79$\pm$23.32 & 148.31$\pm$26.50 \\
\specialrule{1.5pt}{0pt}{0pt}
 & \multicolumn{4}{c||}{trained on $2\times3$ and tested (zero-shot) on $3\times3$} & \multicolumn{4}{c}{trained on $2\times3$ and tested (zero-shot) on $4\times4$} \\
\cmidrule(lr){2-5}\cmidrule(lr){6-9}
\shortstack[l]{PLATO} & \multicolumn{1}{c|}{S0} & \multicolumn{1}{c|}{S1} & \multicolumn{1}{c|}{S2} & \multicolumn{1}{c||}{S3} & \multicolumn{1}{c|}{S0} & \multicolumn{1}{c|}{S1} & \multicolumn{1}{c|}{S2} & \multicolumn{1}{c}{S3} \\
\midrule
(mlp+add)      & 9.97$\pm$0.33  & 6.83$\pm$2.87          & 80.75$\pm$12.65 & 4.95$\pm$28.12           & 11.73$\pm$2.51          & 5.31$\pm$9.89  & 64.11$\pm$15.58 & 13.68$\pm$36.62 \\
(mlp+dot)      & 10.00$\pm$0.00 & 6.25$\pm$3.01          & 78.19$\pm$12.02 & 4.52$\pm$27.48           & \textbf{12.96$\pm$1.93} & 4.09$\pm$8.57  & 63.17$\pm$16.43 & 17.89$\pm$36.82 \\
(lstm+add)     & 9.84$\pm$0.91  & 6.52$\pm$2.64          & 71.79$\pm$22.24 & \textbf{18.49$\pm$27.48} & 11.09$\pm$2.78          & 7.37$\pm$9.37  & 66.88$\pm$11.67 & \textbf{37.97$\pm$32.54} \\
(lstm+dot)     & 9.89$\pm$0.80  & 6.31$\pm$3.39          & 82.87$\pm$16.74 & 6.19$\pm$28.53           & 10.32$\pm$2.95          & 4.56$\pm$6.30  & 65.76$\pm$18.18 & 25.56$\pm$36.38 \\
\bottomrule
\end{tabular}%
}
\caption{Ablation: Episode return (mean $\pm$ std). \textit{Top:} native evaluation; no variant achieves statistically significant best performance. \textit{Bottom:} zero-shot transfer; bold indicates statistically significant best performance. Shapiro-Wilk-adaptive two-sided test; Bonferroni-corrected, $p{<}0.0125$.}
\label{tab:ablation_reward}
\end{table*}

\section{Related Work}\label{sec:related}

Graph Neural Networks (GNNs) \citep{gnn} model multi-agent interactions flexibly by representing agents as graph nodes and using message-passing to aggregate neighbor information, naturally handling variable group sizes but often limiting openness by imposing bounded assumptions. Several dynamic multi-agent approaches partially address related challenges. DGN~\citep{dgn} uses graph convolution with fixed-size neighbor communication for cooperation and can generalize to larger teams but restricts interaction through a predefined neighborhood and assumes static goals. DICG~\citep{dicg} learns dynamic coordination graphs for multi-agent coordination, handling large fixed teams well but requires known agent and task sets in advance, limiting true openness. GPL~\citep{gpl} tackles AO in ad hoc teamwork with graph neural nets and type inference for variable team sizes but trains only a single agent while others have fixed policies, thus not a fully adaptive multi-agent training approach. In contrast, MOHITO~\citep{anilmohito} is specifically designed for openness, formalizing task-open Markov games and employing hypergraph-based actor–critics to handle dynamic task–action sets, though it focuses solely on TO and assumes fixed agents and frames. Our investigations show that PLATO achieves strong performance under combined openness (S2--S3) and more consistent zero-shot generalization than baselines, with advantages that widen at larger grid scales.
Also, while some may argue that transformers \citep{vaswani2017attention} naturally handle unbounded input sizes, imposing bounds is necessary to maintain training efficiency and scalability. That is, though transformers theoretically support variable-length inputs by attending over all tokens---making them attractive to OASYS, to enable efficient parallelization, \emph{inputs are typically padded to fixed lengths}.  This practice limits transformer-based applications in OASYS to a maximum number of agents or tasks. PLATO, on the other hand, circumvents this limitation via the attention mechanism and team-based statistics to handle variability in agents and tasks.  

Related fields such as multi-task reinforcement learning~\citep{zhang2021survey} and curriculum learning~\citep{narvekar2020curriculum} focus on adapting to varying tasks or progressively harder scenarios but typically assume a fixed number of entities and do not address AO and TO.  Also, they may suffer from catastrophic forgetting when learning new tasks~\citep{ribeiro2019multitask}, high prediction errors in volatile environments, and do not effectively handle openness due to their reliance on predefined task spaces or curriculum structures~\citep{kravchenko2024limitations, abadi2025challenges}. 

\section{Conclusion}\label{sec:Conclusion}
Open agent systems require policies that remain well-defined as both the active agent set and the task set change online. We introduced \textbf{PLATO}, an approach that addresses this challenge by combining (i) a pointer-network actor that outputs a distribution directly over the \emph{current} task set, conditioned on a permutation-invariant summary of the current team, and (ii) a centralized GNN critic that evaluates open global states via a bipartite agent--task graph whose structure changes with system composition. PLATO's design avoids fixed action indexing, padding-based bounds, and retraining to handle openness. We formalize the setting as a TaAgO-MG and prove that PLATO is well-defined and permutation-invariant over the resulting unbounded state and action spaces (Appendix~\ref{app:well-defined}).

Across wildfire configurations with different types of openness (i.e., endogenous and exogenous TO and temporary AO), PLATO achieves strong performance on seen grids and more consistent zero-shot generalization to unseen grid sizes across the evaluated settings. Our ablation studies reveal an interaction between encoder and scorer: (lstm+add) achieves the strongest zero-shot generalization under combined AO+TO settings, while (mlp+dot) leads in simpler zero-shot conditions, with neither factor being universally dominant.

We plan to (1) extend TaAgO-MG and the pointer interface to support tasks with multiple actions (i.e., per-task action sets) to address the one-action-per-task limitation of our current TaAgO-MG; (2) investigate the interactions between temporal memory and additive attention scoring more deeply, as the ablation results suggest that recurrent history may be particularly valuable in out-of-distribution settings; (3) relax the full observability of tasks and agents assumption to study partially observable variants with latent-state or belief-based extensions; and (4) consider additional domains outlined by \cite{aimag_eck} and frame openness to further evaluate and improve PLATO.


\begin{ack}

This research was supported, in part, by a collaborative NSF Grant \#IIS-2312657 (to Doshi), \#IIS-2312658 (to Soh), and \#IIS-2312659 (to Eck). Some of the computing occurred on the Holland Computing Center of the University of Nebraska, which receives support from the university’s Office of Research and Economic Development and the Nebraska Research Initiative.

\end{ack}

\bibliographystyle{plainnat}
\bibliography{references}

\appendix

\section{Proofs and Complexity}

\subsection{Lemma 1: Permutation Invariance of Team Statistics}\label{app:stats-perm}

Here we show that the set of statistics we use to approximate the state of present agents is permutation invariant. 

\setcounter{lem}{0}
\begin{lem}[Encoder Permutation Invariance]
    \label{lem:stats_appendix}
    Each team statistic is permutation invariant, $MLP_s$ operates on the aggregate, and $MLP_x$ operates on task observations independently. Therefore, the agent and task encodings are permutation invariant.
\end{lem}

\begin{proof}\label{prop:stats-perm}
Let $o^t_N:[o^t_1,...o^t_{n}] \in \mathbb{R}^{n \times f_N}$ be the agent-feature matrix for a team of $n = |N^t|$ agents, and let $\sigma$ be any permutation of $\{1,\ldots,n\}$. Then the team-statistics vector
\[
  d^t = \bigl[\mu(o^t_N),\;\mathrm{Var}(o^t_N),\;\min(o^t_N),\;\max(o^t_N)\bigr]
  \in \mathbb{R}^{4f_N}
\]
is invariant to agent ordering $d^t\!\bigl(\sigma(o^t_N)\bigr) = d^t\!\bigl(o^t_N\bigr)$ if all statistics in $d^t$ are permutation invariant.

\emph{Mean.} Because $\sigma$ is a bijection on $\{1,\ldots,n\}$, the index set $\{\sigma(1),\ldots,\sigma(n)\}$ equals $\{1,\ldots,n\}$, so
\[
  \mu_j\!\bigl(\sigma(o^t_N)\bigr)
  = \frac{1}{n}\sum_{i=1}^n o^t_{\sigma(i),j}
  = \frac{1}{n}\sum_{i=1}^n o^t_{i,j}
  = \mu_j\!\bigl(o^t_N\bigr).
\]
\emph{Variance.} Since $\mu_j$ is invariant (above), the same bijection argument gives
\[
    \mathrm{Var}_j\!\bigl(\tau(o^t_N)\bigr)
  = \frac{1}{n}\sum_{i=1}^n \bigl(o^t_{\sigma(i),j}-\mu_j\bigr)^2
  = \frac{1}{n}\sum_{i=1}^n \bigl(o^t_{i,j}-\mu_j\bigr)^2
  = \mathrm{Var}_j\!\bigl(o^t_N\bigr).
\]
\emph{Min and max.} The minimum (resp.\ maximum) of a finite multiset depends only on the set of values, not their order:
\[
    \min_{i}\,o^t_{\tau(i),j} = \min_{i}\,o^t_{i,j},
  \qquad
    \max_{i}\,o^t_{\tau(i),j} = \max_{i}\,o^t_{i,j}.
\]
All four statistics are feature-wise invariant for every $j$, so $d^t\!\bigl(\tau(o^t_N)\bigr) = d^t\!\bigl(o^t_N\bigr)$. Since this holds for all permutations $\tau$, $d^t$ is permutation-invariant in the agent ordering.
\end{proof}


\subsection{Theorem 1: Well Defined PLATO} \label{app:well-defined}

A precondition to learning a rational policy in openness is producing \emph{consistent actor queries} from finite parameters conditioned on a countably infinite observation space. In other words, the policy must be a well defined function. GNN actors, such as MOHITO \citep{anilmohito}, meet this inherently because graph convolution is permutation invariant to node ordering and produce observation specific action output from its hyperedges. Here we prove that PLATO also fulfills this requirement.

\setcounter{theorem}{0}
\begin{theorem}[PLATO is well defined]
    \label{theorem:well-defined_appendix}
    Because all components of the actor and critic are permutation invariant, as shown in Lemma~\ref{lem:stats_appendix} and in the proof below, all countably infinite possible observations have only one categorical distribution in the actor and one value estimation in the critic.
\end{theorem}

To prove PLATO is a well defined function, we will show that the actor, $\pi$, and critic, $V$, have values for each state in this unbounded space,
$$
\begin{aligned}
    & (1) \forall s \in \mathbf{S}, \exists\Delta A \subseteq \mathbf{A}: (s,\Delta A) \in \pi
    \qquad (2) \forall s \in \mathbf{S} \exists v \in \mathbb{R}: (s,v) \in V,
\end{aligned}
$$
and that they consistently produce only one finite action distribution and $v$ approximation for each state,
$$
\begin{aligned}
    &(3) \forall s \in \mathbf{S}, \quad \forall \Delta A_1, \Delta A_2 \subseteq \mathbf{A}\ s.t.\ \Delta A_1 \neq \Delta  A_2: (s,\Delta A_2) \notin \pi \text{ if } (s,\Delta A_1) \in \pi\\
    &(4) \forall s \in \mathbf{S}, \quad \forall v_1,  v_2 \in \mathbb{R}\ s.t.\ v_1\neq v_2: (s,v_2)\notin  V\ \text{if}\  (s,v_1) \in V.
\end{aligned}
$$



\begin{proof}[Proof (1)]
    \label{proof:definedpi}
    Assume there exists a state, $\hat{s} \in \mathbf{S}$ where $\pi(\hat{s})\neq \Delta A \subseteq \mathbf{A}$. It must then be true that either $\pi(\hat{s})$
 is undefined, in other words $\hat{s} \notin \pi$, or $\pi(\hat{s})$ must have support over an action not in $\mathbf{A}$.
 
    PLATO's actor preprocessor factors agent observations into agent and task features of length $|X^t_i|$ and $|N^t|$ respectively. $X^t_i$ refers to tasks agent $i$ can act on; $X_i^t\subseteq X^t$. We assume this factorization is given by the environment and defined over $\Psi_\Omega$. Regarding agent observations, team statistics produce $d^t \in \mathbb{R}^{4f_N}$, a \emph{fixed}-length vector regardless of $|N^t|$. The subsequent MLP is defined over $\mathbb{R}^{4f_n}$. On the task observations, the task encoding applies individually to tasks and is thus defined over $\mathbb{R}^{f_x}$, and produces keys, $K^t \in \mathbb{R}^{(|X^t|+1)f_h}$. Each score $u_x^t = \mathbf{v}^\top\!\tanh(\mathbf{W}_K k_x^t + \mathbf{W}_q q^t)$ applies the same fixed-parameter function regardless of $|X^t_i|$ or $|N^t|$. Because the only input constraints come from $\Psi_\Omega$, and $\Psi_\Omega$  is defined over all possible task and agent sets, $s\in \pi \iff s \in \mathbf{S}$. It must be the case that $\hat{s} \in \pi$.

There exists one score $u_x^t \forall X_i^t \in M^t$ and one NO-OP score. PLATO assumes that there exists one action per task, $a_X \forall X_i$. There are also additional task-agnostic actions $A_*$, here a NO-OP action.  If the environment's action space is defined as, $\Psi_A(...) = \prod_{i}\{a_x  | x \in X_i, X_i\subseteq X\} + \{a_{NOOP}\}$, then any task score set must only include $a_x \forall X_i$ and NO-OP by definition. It must be the case that $\pi(\hat{s})\subseteq A$. Thus, by contradiction  $\forall s \in \mathbf{S}, \exists\Delta A \subseteq \mathbf{A}: (s,\Delta A) \in \pi$. 
\end{proof}

\begin{proof}[Proof (2)]
    \label{proof:definedq}
I    The global graph is a partition over the state, similar to the preprocessor's output, but encompassing all agents' local observations. If the factored agent, $\mathcal{N}$  and task nodes $\mathcal{X}$, have fixed feature dimensions, $f_\mathcal{N}, f_{\mathcal{X}}$, the critic's MLPs can encode them into a shared latent space, $\mathbb{R}^{f_h}$.  The subsequent layers are GCNs over these nodes which are defined for a fixed $f_h$, global mean pooling which has no spatial requirements, and a decoder MLP, $MLP_v: \mathbb{R}^{f_h}\to \mathbb{R}$. Thus, if tasks and agents can be factored into state feature representations of fixed length regardless of change in the overall task and agent sets, then  $\forall s \in \mathbf{S} \exists v \in \mathbb{R}: (s,v) \in V$.
\end{proof}

\begin{proof}[Proof (3)]
    \label{proof:mappingpi}
        Let $\sigma(o_i^t)$ represent a permutation of task and agent observations given by the environment. The preprocessor is not assumed to be permutation invariant, in that \textbf{only the} order of factors may change by the order of input observations. The agent encoder is permutation invariant because the team statistics are comprised of permutation invariant operations, (mean, variance, minimum, maximum). The task encoder operates on all $o_x^t \in o_X^t$ independently and like the preprocessor persists ordering, but $K^t$'s values are unaffected by the order. The pointer decoder does not use a positional embedding, and more specifically,
$$
\mathbf{v}^\top\!\tanh(\mathbf{W}_K \sigma(k_x^t) + \mathbf{W}_q q^t) = \mathbf{v}^\top\!\tanh(\mathbf{W}_K k_x^t + \mathbf{W}_q q^t) \quad \forall \sigma, x.
    $$
    Therefore,  $\pi$ is permutation invariant, $\pi_\theta(\cdot \mid o_i^t) = \pi_\theta(\cdot \mid \sigma(o_i^t))$.  Because there are no stochastic operations in $\pi$ and we've established permutation invariance,
$$
    \forall s \in \mathbf{S}, \quad \forall \Delta A_1, \Delta A_2 \subseteq \mathbf{A}: (s,\Delta A_1), (s,\Delta A_2) \in \pi.
$$
\end{proof}
\begin{proof}[Proof (4)]
    \label{proof:mappingq}
    Let $\sigma(\mathcal{G})$ represent a reordering of nodes describing the state given by the environment. Critic encoders, $MLP_x$ and $MLP_v$ , encode nodes independently and are thus permutation invariant to node ordering. GCN message passing only produce a different output graph if the input graph is different. By definition graph vertices are a set. Permuting their order does not change the graph. Because none of the remaining operations are stochastic and we have shown the critic is permutation invariant,  $\forall s \in \mathbf{S}, \quad \forall v_1,  v_2 \in \mathbb{R}: (s,v_1), (s,v_2) \in V.$
\end{proof}

\subsection{How the Pointer Policy Handles Task and Agent Openness}\label{app:ao-to-policy}

We characterize how the pointer actor handles each form of openness. We call the mechanism of Lemma~\ref{lem:to-pointer}, which points over tasks to select the action, the \emph{task pointer}, and the mechanism of Lemma~\ref{lem:ao-pointer}, which points over agents to build the query, the \emph{agent pointer}. Lemma~\ref{lem:to-pointer} concerns task openness (TO): the actor's action set is exactly the currently available tasks together with a \textsc{no-op}, and the relative preference between any two tasks is independent of which other tasks are present, so the policy adapts as tasks appear or disappear. Lemma~\ref{lem:ao-pointer} concerns agent openness (AO): the same pointing mechanism builds the query from the current agent set, with each observed agent encoded as a key and the deciding agent's own key acting as the query, so the policy adapts as agents arrive or depart.

\begin{lem}[Task openness: the task pointer is support-adaptive and rank-consistent]\label{lem:to-pointer}
Fix parameters $\theta$ and let $q$ be any fixed query vector, and assume PLATO's
one-action-per-task setting: each task $x\in X^t$ has a single associated action
$a_x$ plus a persistent \textsc{no-op} (in wildfire domain for example, $a_x=\textsf{suppress}(x)$).
Let $\pi(\cdot)=\mathrm{softmax}\bigl([u_a^t]_{a\in X^t\cup\{\textsc{no-op}\}}\bigr)$
with per-task keys $k_x^t=\mathrm{MLP}_x(o_x^t)$ and scores
$u_x^t=(\mathbf{v}^{\mathrm{t}})^{\top}\!\tanh(\mathbf{W}_K^{\mathrm{t}} k_x^t+\mathbf{W}_q^{\mathrm{t}} q)$, where $\mathbf{W}_K^{\mathrm{t}}$ and
$\mathbf{W}_q^{\mathrm{t}}$ are learnable projection matrices for the keys and the query, respectively
(note: the subscripts $K$ and $q$ name each matrix's role and are not indices; the upright superscripts $\mathrm{t}$ and $\mathrm{a}$ tag the task-side and agent-side scorers and are not the time index $t$), and
$\mathbf{v}^{\mathrm{t}}$ is a learnable vector mapping the compatibility to a scalar score (cf.\ the Pointer Actor section of the main paper). Then (i) the support
of $\pi$ equals $\{\textsc{no-op}\}\cup\{a_x\mid x\in X^t\}$; and (ii) writing
$\pi(x)$ for the mass on $a_x$, for any $x,x'\in X^t$ the log-odds
$\log\frac{\pi(x)}{\pi(x')}=u_x^t-u_{x'}^t$ depend only on $(k_x^t,k_{x'}^t,q)$ and are
invariant to $X^t\setminus\{x,x'\}$. Hence a task arriving or departing changes
only the normalization and preserves the ranking of every remaining task.
\end{lem}
\begin{proof}
The task encoder maps each available task $x\in X^t$ independently to a key $k_x^t=\mathrm{MLP}_x(o_x^t)$, and a persistent $\textsc{no-op}$ key is always present, so the score set is exactly $\{u_a^t:a\in X^t\cup\{\textsc{no-op}\}\}$ and the softmax over these $|X^t|+1\ge 1$ scores is a categorical whose support is $\{\textsc{no-op}\}\cup\{a_x\mid x\in X^t\}$, giving (i).

For (ii), the additive scorer \citep{bahdanau2014neural} $u_x^t=(\mathbf{v}^{\mathrm{t}})^{\top}\!\tanh(\mathbf{W}_K^{\mathrm{t}} k_x^t+\mathbf{W}_q^{\mathrm{t}} q)$ depends only on that task's key $k_x^t$ and the fixed query $q$. With softmax normalizer $Z=\sum_{a\in X^t\cup\{\textsc{no-op}\}}e^{u_a^t}$ each task's
probability is $\pi(x)=e^{u_x^t}/Z$ and $\pi(x')=e^{u_{x'}^t}/Z$. Forming the log-odds
with the normalizer still in place and then cancelling the common $Z$,
\[
\log\frac{\pi(x)}{\pi(x')}
=\log\frac{e^{u_x^t}/Z}{\,e^{u_{x'}^t}/Z\,}
=\log\frac{e^{u_x^t}}{e^{u_{x'}^t}}
=u_x^t-u_{x'}^t.
\]
Because each key $k_y^t=\mathrm{MLP}_x(o_y^t)$ is a function of task $y$'s features alone and $q$ is
fixed, $u_x^t$ and $u_{x'}^t$ do not depend on which other tasks are present; adding or removing a task
$x''\ne x,x'$ only inserts or deletes the term $e^{u_{x''}^t}$ in $Z$, leaving $u_x^t,u_{x'}^t$ and hence the log-odds unchanged. Hence $\log\frac{\pi(x)}{\pi(x')}=u_x^t-u_{x'}^t$ is invariant to $X^t\setminus\{x,x'\}$, establishing (ii); equivalently, each remaining task's score is unchanged, so their ranking is preserved and only the normalization shifts (content-based pointing \citep{ptr}). The restriction $x''\notin\{x,x'\}$ is the scope of this pairwise claim, not a limitation:
$\log\frac{\pi(x)}{\pi(x')}$ is defined only while both $x$ and $x'$ lie in the support
of~(i). Removing either leaves the policy well defined and rank-consistent among the tasks
that remain, and since each key $k_x^t=\mathrm{MLP}_x(o_x^t)$ is recomputed from current
features, a task that departs and later reappears re-enters under the same log-odds
$u_x^t-u_{x'}^t$, carrying no state during its absence.
\end{proof}

\begin{lem}[Agent openness: the agent pointer is support-adaptive and rank-consistent]\label{lem:ao-pointer}
Fix parameters $\theta$, let $i$ be the deciding agent, and let the \emph{team} $N^t$ be the set of
agents currently observed by $i$: each element $j\in N^t$ is one currently present
agent, with observed features $o_j^t$ accounting for one row of the agent-feature matrix $o_N^t$.
Both are part of agent $i$'s local observation $o_i^t$, so a different deciding agent generally
has a different agent-feature matrix; since the whole lemma is computed by the fixed deciding
agent $i$, we omit the extra subscript $i$ on $o_j^t$, $o_N^t$, and all derived quantities
($k_j^t$, $e_j^t$, $\alpha^t$, $q^t$) for readability.
An agent always observes itself, so $i\in N^t$ and $|N^t|\ge 1$.
Let per-agent keys $k_j^t=\mathrm{MLP}_n(o_j^t)\in\mathbb{R}^{f_h}$
for $j\in N^t$, scores $e_j^t=(\mathbf{v}^{\mathrm{a}})^{\top}\!\tanh(\mathbf{W}_K^{\mathrm{a}} k_j^t+\mathbf{W}_q^{\mathrm{a}} k_i^t)$
with the deciding agent's own key $k_i^t$ as the query, where
$(\mathbf{W}_K^{\mathrm{a}},\mathbf{W}_q^{\mathrm{a}},\mathbf{v}^{\mathrm{a}})$ are the agent-side counterparts of the task-side
$(\mathbf{W}_K^{\mathrm{t}},\mathbf{W}_q^{\mathrm{t}},\mathbf{v}^{\mathrm{t}})$ of Lemma~\ref{lem:to-pointer}: learnable projection
matrices for the keys and the query, and a learnable scoring vector (cf.\ the Pointer Actor section of the main paper), attention
$\alpha^t=\mathrm{softmax}\bigl([e_j^t]_{j\in N^t}\bigr)$, and the query, the
attention-weighted read of the agent keys,
$q^t=\sum_{j\in N^t}\alpha_j^t k_j^t$. Then (i) $\alpha^t$ is a well-defined
categorical whose support is exactly $N^t$, for every
$|N^t|\ge 1$; (ii) $q^t\in\mathbb{R}^{f_h}$ is a
fixed-length, permutation-invariant function of the agent set, for every
$|N^t|\ge 1$; and (iii) for any $j,j'\in N^t$ the log-odds
$\log\frac{\alpha_j^t}{\alpha_{j'}^t}=e_j^t-e_{j'}^t$ depend only on
$(k_j^t,k_{j'}^t,k_i^t)$ and are invariant to $N^t\setminus\{j,j',i\}$. Hence an
agent arriving or departing changes only the normalization and preserves the
\emph{relative weighting} of every remaining agent, that is, the ratio
$\alpha_j^t/\alpha_{j'}^t$ between the attention weights of any two agents that remain. Moreover, the agent set enters each
task score $u_x^t=(\mathbf{v}^{\mathrm{t}})^{\top}\!\tanh(\mathbf{W}_K^{\mathrm{t}} k_x^t+\mathbf{W}_q^{\mathrm{t}} q^t)$, with $k_x^t=\mathrm{MLP}_x(o_x^t)$ the key of task $x$ and $(\mathbf{W}_K^{\mathrm{t}},\mathbf{W}_q^{\mathrm{t}},\mathbf{v}^{\mathrm{t}})$ the task-side learnable projections and scoring vector of Lemma~\ref{lem:to-pointer}, where the subscript $x$ indexes a task just as $j$ indexes an agent ($u_x^t$ is the task-side score, $e_j^t$ its agent-side counterpart),
only through $q^t$, and because Lemma~\ref{lem:to-pointer} holds for any fixed
query, its conclusions apply verbatim with $q=q^t$ at each $t$.
\end{lem}
\begin{proof}
For (i), the agent encoder maps each observed agent $j\in N^t$ independently to a key $k_j^t=\mathrm{MLP}_n(o_j^t)$, and the deciding agent always observes itself, so $i\in N^t$ and the score set $\{e_j^t: j\in N^t\}$ contains at least the ego score; the softmax over these $|N^t|\ge 1$ scores is a categorical whose support is exactly $N^t$. 

For (ii), permutation invariance of the normalizer $Z_N=\sum_{l\in N^t}e^{e_l^t}$ and the read $q^t=\sum_{j\in N^t}\alpha_j^t k_j^t$ is exactly Lemma~\ref{lem:stats_appendix}. Since $q^t$ is a convex combination of keys in $\mathbb{R}^{f_h}$, its length is $f_h$ regardless of $|N^t|$, giving (ii).

For (iii), the additive scorer \citep{bahdanau2014neural} $e_j^t=(\mathbf{v}^{\mathrm{a}})^{\top}\!\tanh(\mathbf{W}_K^{\mathrm{a}} k_j^t+\mathbf{W}_q^{\mathrm{a}} k_i^t)$ depends only on that agent's key $k_j^t$ and the deciding agent's key $k_i^t$, which is always available since $i\in N^t$. With softmax normalizer $Z_N=\sum_{l\in N^t}e^{e_l^t}$ each agent's
weight is $\alpha_j^t=e^{e_j^t}/Z_N$ and $\alpha_{j'}^t=e^{e_{j'}^t}/Z_N$. Forming the log-odds
with the normalizer still in place and then cancelling the common $Z_N$,
\[
\log\frac{\alpha_j^t}{\alpha_{j'}^t}
=\log\frac{e^{e_j^t}/Z_N}{\,e^{e_{j'}^t}/Z_N\,}
=\log\frac{e^{e_j^t}}{e^{e_{j'}^t}}
=e_j^t-e_{j'}^t.
\]
Because each key $k_l^t=\mathrm{MLP}_n(o_l^t)$ is a function of agent $l$'s features alone, $e_j^t$ and $e_{j'}^t$ do not depend on which other agents are present; adding or removing an agent
$j''\notin\{j,j',i\}$ only inserts or deletes the term $e^{e_{j''}^t}$ in $Z_N$, leaving $e_j^t,e_{j'}^t$ and hence the log-odds unchanged. Hence $\log\frac{\alpha_j^t}{\alpha_{j'}^t}=e_j^t-e_{j'}^t$ is invariant to $N^t\setminus\{j,j',i\}$, establishing (iii); equivalently, each remaining agent's score is unchanged, so their relative weighting is preserved and only the normalization shifts (content-based pointing \citep{ptr}). Here \emph{unchanged} means at the same timestep with the same observed features: the scorer reads only $(k_j^t,k_i^t)$, never the rest of $N^t$, so a change in the team membership by itself cannot re-rank the remaining agents. The policy still adapts at once, because the departed key drops out of $q^t$ and the weights renormalize, shifting the task distribution immediately; at later steps the departure also changes what agents observe (in wildfire, an unattended fire grows and suppressant use shifts), and those new features change the scores. The restriction $j''\notin\{j,j',i\}$ is the scope of this pairwise claim, not a limitation: $i$ never departs and its key is part of the stated dependence, while
$\log\frac{\alpha_j^t}{\alpha_{j'}^t}$ is defined only while both $j$ and $j'$ lie in the support
of~(i). Removing either leaves the attention well defined and rank-consistent among the agents
that remain, and since each key $k_j^t=\mathrm{MLP}_n(o_j^t)$ is recomputed from current
features, an agent that departs (for example, after running out of suppressant) and later
returns is re-encoded from its current features and re-enters under the same scoring rule: its
log-odds against any current teammate are determined by its current features and the current
ego key $k_i^t$. The policy carries no memory of the agent through its absence: nothing is
saved at departure and nothing is restored at re-entry; any persistence (for example, its
suppressant level) lives in the environment state and reaches the policy only through the
agent's current observed features.

For the final claim, each task key $k_x^t=\mathrm{MLP}_x(o_x^t)$ is a function of task features alone, so the agent set enters the score $u_x^t=(\mathbf{v}^{\mathrm{t}})^{\top}\!\tanh(\mathbf{W}_K^{\mathrm{t}} k_x^t+\mathbf{W}_q^{\mathrm{t}} q^t)$ only through $q^t$.
At any time $t$, the read $q^t$ is one fixed vector in $\mathbb{R}^{f_h}$. Lemma~\ref{lem:to-pointer} was proved for an arbitrary fixed query, so its conclusions hold verbatim with $q=q^t$. In particular, $q^t$ is computed from the agent features $o_N^t$ alone, so it does not change when tasks arrive or depart, and Lemma~\ref{lem:to-pointer}'s invariance to the rest of the task set applies. Note that $q^t$ itself does change when the team changes: the weights renormalize and keys enter or leave the sum. That is the intended adaptation to the current team; claim (iii) asserts invariance of the relative weighting among the remaining agents, not constancy of $q^t$.

\end{proof}

\subsection{Implications of the Openness Lemmas}\label{app:openness-implications}

Together, Lemmas~\ref{lem:to-pointer} and~\ref{lem:ao-pointer} pin down how the policy reacts
when the task set or the team changes: adaptation flows through the observed features and the
query, never through the membership of the sets themselves. We illustrate with an example. Let
the team be $N^t=\{n_1,n_2,n_3\}$, three agents as defined in Lemma~\ref{lem:ao-pointer}, and
let the task set $X^t$ contain the fires $x_1$, $x_2$, and $x_5$, where fire $x_1$ requires all
three agents while the smaller fires $x_2$ and $x_5$ can each be handled by one agent. Suppose
$n_2$ departs, for example after running out of suppressant. At that instant, by
Lemma~\ref{lem:ao-pointer}(iii), the relative weighting between $n_1$ and $n_3$ is unchanged:
the departure by itself does not re-rank the remaining agents. The policy nevertheless adapts
immediately: $n_2$'s key $k_{n_2}^t$ drops out of the query $q^t$ and the attention weights
$\alpha^t$ re-normalize, so each remaining agent's task distribution shifts, and by
Lemma~\ref{lem:to-pointer} this shift re-weights the available tasks without disturbing the
support or requiring retraining; $n_1$ may now prefer $x_2$ and $n_3$ may prefer $x_5$, so the
joint attack on $x_1$ dissolves. In the following steps the consequences of the departure enter
through the features: $x_1$ grows unattended and suppressant use shifts, and these new
observations change both the agent scores and the task scores. If $n_2$ later returns, or a
brand-new agent arrives, it is encoded from its current features and joins the same
computation, with no reserved slot, no reset, and no retraining. The two lemmas thus guarantee
that openness changes the policy only in the intended way: the distribution is always well
defined over exactly the current tasks and the current team, and it reacts to openness through
content rather than through identity or count.

\subsection{Theorem 2: Computational complexity.}\label{app:complexity}
Here we analyze the space and time complexity of PLATO. Our time complexity bound is non-deterministic, in that we consider parallelization where possible. We treat all fixed hyperparameters ($f_N$, $f_X$, $f_h$, number of GCN layers), including the feature size of factored tasks and agents as constants. We do this to express complexity solely in terms of the quantities that vary under openness, the number of active agents $|N^t|$ and active tasks $|X^t|$.

\begin{theorem}[Complexity Analysis for PLATO]
    \label{thm:analysis_appendix}
    Assuming neural network hyperparameters, feature dimensions, and graph convolutional layers are constants, the per-timestep complexity of PLATO is $\mathcal{O}(|N^t|^2+|N^t||X|^t).$
\end{theorem}

\paragraph{Actor.}
The preprocessor and team statistics scan $o_N^t$ and $o_X^t$ in $O(|N^t|+|X^t|)$.
The team-statistics step outputs a \emph{fixed}-length descriptor $d^t \in \mathbb{R}^{4f_N}$ regardless of $|N^t|$, so $\mathrm{MLP}_s$ contributes an $O(1)$ constant.
$\mathrm{MLP}_x$ is applied independently to each of the $|X^t|+1$ task features, costing $O(|X^t|)$; because tasks are encoded without shared state, this step parallelizes fully over the task dimension.
One-step additive attention \citep{bahdanau2014neural,ptr} computes a scalar score per key in $O(1)$ and runs over all keys in $O(|X^t|)$; \textsc{Softmax} normalization is likewise $O(|X^t|)$.
The per-agent actor cost is $O(|N^t|+|X^t|)$, and for all $|N^t|$ agents sharing the same policy the per-timestep actor cost is $O(|N^t|(|N^t|+|X^t|))$.

\paragraph{Critic.}
Node encoding via separate MLPs costs $O(|N^t|+|X^t|)$.
The bipartite graph $\mathcal{G}^t$ has $|\mathcal{E}^t|=|N^t||X^t|$ edges; each GCN layer \citep{kipf2016semi} aggregates over $|\mathcal{E}^t|$ edges in $O(|N^t||X^t|)$, which dominates the per-node projection $O(|N^t|+|X^t|)$.
With $L=2$ layers and a residual connection \citep{he2016deep}, total propagation is $O(|N^t||X^t|)$.
Mean pooling and the value head are $O(|N^t|+|X^t|)$ and $O(1)$, respectively.
The critic cost per timestep is therefore $O(|N^t||X^t|)$, which dominates the actor during training.
At execution time the critic is absent, and the cost across all agents reduces to $O(|N^t|(|N^t|+|X^t|))$ per timestep.

\subsection{Reward bounds in Wildfire}

 In Wildfire, the only positive rewards are $r_{putout}^{\text{size}}$ for successfully suppressing a fire, and the only negative reewards are $r_{burnout}^{\text{size}}$ when a fire burns out.  With stochastic random ignition, in the worst case  all fires immediately ignite once they burnout. Therefore, trivially, the worst penalty over any trajectory $\tau$ is bounded by the rate fires burnout,
 \begin{equation}
 \label{eq:lowerrbound}
 R(\tau)\geq \left\lfloor\frac{h}{\text{int}_0+1}\right\rfloor\sum_x r^{size_x}_{burnout},
 \end{equation}
 where $h$ is the time horizon, and $\text{int}_0$ is the initial intensity of fires. The $+1$ is required because fires only ignite when they have $\text{int}=0$, so there is a timestep delay. 
 
 For the upper bound, fire sizes are very important. All fires start at the same intensity, $\text{int}_0$, so agents, $N$, strictly benefit from collaborating if $r^{size}_{putout}\geq  \text{size}\cdot r_{putout}$.  When fires are all $\text{size}=1$, the maximum reward can similarly be bounded by the maximum rate fires can be put out,
 \begin{equation}
 \label{eq:simplerbound}
      R(\tau) \leq  \frac{h}{\text{int}_0} \cdot r_{putout}^{\text{size}_x}\cdot \min(|N|,  \lceil\text{number of unique reachable fires}/2\rceil).
 \end{equation}
 In this bound, the set of \emph{unique reachable fires} are all fires reachable by at least one unique agent. This approach will overestimate the impact of interleaving fires.  Size 1 fires can be suppressed by one agent, so the optimal strategy is repeatedly put out either the one fire available to that agent ($|N|$), or if  agents can reach two unique fires, alternate fighting them to avoid the one timestep ignition delay, ($\lceil \text{number of unique reachable fires}/2\rceil$).  When  grid size is sufficiently large, increasing it will not change \emph{unique reachable fires} if agent position and reach are fixed.

We address both fire size, and \emph{fuel}, a mechanic added to some recent uses of the Wildfire environment, \citep{patino2025inaugural}, through optimization. Fuel involves putting out or letting a fire burnout deterministically decreases the number of times it can ignite.  Our algorithm solves this by calculating optimal cycles. Let agents repeatedly put out fires in \emph{cycles}, $c \in C$. Agents may pause one cycle while waiting for reignition to work on another. Each cycle contains one fire, with fire size monotonically increasing by cycle index. 

We determine which cells will be completed as an integer program solution, similar to the  Maximum Temporal Matching Problem \citep{mtmp},

\begin{alignat}{2}
\text{Maximize} \quad & \sum_{c \in C} \sum_{x \in X} f_{x,c} \cdot r_x \notag \\
\text{Such That} \quad & \sum_{n \in N} a_{n,x,c} = f_{x,c} \cdot \text{size}_{x} && \forall x \in X, c \in C \label{const:res_alloc} \\
& \sum_{x \in X} a_{n,x,c}\cdot\text{size}_x  \geq \sum_{x \in X} a_{n,x,c+1}  \cdot\text{size}_x && \forall n \in N, c < |C| \label{const:lex_order} \\
& \sum_{c \in C} a_{n,x,c} \leq 1 && \forall n \in N, x \in X \label{const:unique_res} \\
& \sum_{x \in X} a_{n,x,c} \leq 1 && \forall n \in N, c \in C \label{const:capacity} \\
& f_{x,c}, a_{n,x,c} \in \{0, 1\} && \forall n, x, c \notag.
\end{alignat}

The constraints on this integer program (IP) solution ensure that (3) fires have sufficient collaboration, (4) monotonic cycle fire size increase, (5) agents attack unique fires across cycles, and (6) agents attack only one fire per cycle. The number of cycles is bounded, $|C| \leq \text{ num reachable fires}$. Now we determine how many times each cycle occurs. We handle this post to avoid solving the integer program over the whole horizon. Let $\rho_c$ be the number of times cycle $c$ produces a reward.  If, 
\begin{equation}
\begin{aligned}
        \rho_{x,1} =& \min( \text{fuel}_x, \left\lfloor\frac{h}{\text{int}_0}\right\rfloor)\\
        &\rho_{x,c} = \min(\text{fuel}_x, \left\lfloor\frac{h }{{\text{int}_0}}\right\rfloor- \sum_{c'=1}^{c-1}\rho_{x,c'}),
\end{aligned}
\end{equation}
then $R(\tau) \leq \sum_{n,x,c} f_{x,c} \cdot r_x \cdot \rho_{x,c}$, for any stochastic Wildfire environment. 

\section{Reproducibility}\label{app:reprod}

Here we formally define the Wildfire environment. We also include all hyper parameters, our seeding approach, and hyperparameter tuning methodology.

Table~\ref{tab:network_architecture} lists the best tuned hyperparameters used for each learned method.

\begin{table*}[ht]
\centering
\small
\setlength{\tabcolsep}{5pt}
\begin{tabular}{|l|c|c|c|c|c|}
\hline
\textbf{Parameter} & \textbf{PLATO (MLP)} & \textbf{PLATO (LSTM)} & \textbf{MOHITO} & \textbf{DGN} & \textbf{DICG} \\
\hline
\multicolumn{6}{|c|}{\textit{Network Architecture}} \\
\hline
Hidden Dimension (Actor) & 128 & 128 & 32 & 128 & 128 \\
Hidden Dimension (Critic) & 128 & 128 & 32 & 128 & 128 \\
Activation Function & ReLU, Leaky ReLU & ReLU, Leaky ReLU & ReLU & ReLU & Tanh \\
\hline
\multicolumn{6}{|c|}{\textit{Learning Rates \& Optimization}} \\
\hline
Actor Learning Rate & 5e-4 & 5e-4 & 1e-4 & 5e-3 & 1e-4 \\
Critic Learning Rate & 9e-4 & 9e-4 & 5e-4 & 1e-4 & 1e-4 \\
Optimizer & Adam & Adam & Adam & RMSprop & Adam \\
Optimizer Epsilon & 1e-5 & 1e-5 & 1e-8 & 1e-5 & 1e-5 \\
Gradient Clipping & 1.0 & 1.0 & 1.0 & 10.0 & 10.0 \\
\hline
\multicolumn{6}{|c|}{\textit{Input Processing}} \\
\hline
Observation Padding & No & No & No & Yes & Yes \\
Padding Value & - & - & - & -1.0 & -1 \\
\hline
\multicolumn{6}{|c|}{\textit{Data Collection}} \\
\hline
Batch Size & 2048 & 2048 & 16 & 1024 & 1024 \\
Parallel Environments & 1 & 1 & 1 & 1 & 1 \\
\hline
\end{tabular}
\caption{Configuration: hyper parameters for all methods. This shows the best hyper parameters from tuning.}
\label{tab:network_architecture}
\end{table*}

\paragraph{Compute resources.\label{app:compute}}
All experiments were conducted on a high-performance computing cluster. The experiments were executed on Linux-based compute nodes running a 64-bit x86 architecture with dual Intel Xeon Gold 6348 CPUs at 2.60\,GHz, for a total of 56 CPU cores. The cluster nodes used the Linux kernel version 4.18.0-553.89.1.el8\_10.x86\_64. Experiments were performed strictly on CPU compute using the free-range-zoo library introduced by \citet{patino2025inaugural}. Each individual training run (one method, one grid, one setup) takes approximately 3--5 days of CPU time. The full experimental suite—covering all methods, grids, setups, and Optuna tuning trials—required approximately 6{,}000--10{,}000 CPU-days in total.

\paragraph{Hyperparameter tuning.}\label{app:hp}
We tune all methods with Optuna~\citep{akiba2019optuna} using a Tree-structured Parzen Estimator (TPE) sampler with a shared budget of 50 trials per method.
Each trial trains for 2{,}000 episodes and is evaluated on 15 validation episodes; the trial objective is the mean validation return on the $3\times 3$ grid under S3, the most challenging setup.
For each method, the search space is anchored to the default hyperparameters reported in its published codebase, with ranges that bracket those defaults on a log scale for learning rates and a linear scale for bounded scalars.
Specifically, learning rates are searched over $[10^{-5},\,10^{-2}]$ (PLATO, DICG) or $[10^{-5},\,10^{-3}]$ (MOHITO) and $[10^{-4},\,10^{-2}]$ (DGN); hidden dimensions over $\{64, 128, 256\}$ (PLATO, DGN, DICG) or $\{32, 64, 128\}$ (MOHITO); batch sizes over $\{512, 1024, 2048, 4096\}$ (PLATO), $\{256, 512, 1024, 2048\}$ (DGN), $\{512, 1024, 2048\}$ (DICG), or $\{16, 32, 64, 128\}$ (MOHITO); and entropy coefficients, gradient clipping, and exploration schedules similarly bracketing each method's reported defaults.
This ensures each baseline is given a fair and equal opportunity to find its best configuration in our environment, independently of its original domain.

\paragraph{Seeding protocol.}\label{app:random_seed}
We use a three-tier seeding protocol, following established reproducibility practice in deep RL~\citep{henderson2018deep} with respect to openness.
The \textbf{global seed} (42) initializes all network parameters and algorithm-level randomness (PyTorch, NumPy, Python random) once per run.
The \textbf{train seed} (300) resets the environment at the start of every training episode, so the policy is always trained from a consistent initial fire configuration; within-episode stochasticity (fire spread, random ignitions, and stochastic refills) still varies across episodes, exposing the policy to diverse dynamics under a fixed start.
The \textbf{execution seed} jointly seeds both the environment and all random-number generators at the start of each evaluation episode, making every episode exactly reproducible from its seed alone; we assign $N{=}50$ distinct sequential seeds per checkpoint ($200, 201, \ldots, 249$).
All three seed ranges are disjoint by construction, so the policy is never evaluated on conditions seen during training.
Sharing the same execution seeds across all methods yields seed-matched episode-level observations, which controls for seed-induced variance and increases the sensitivity of the Shapiro-Wilk-adaptive two-sided tests~\citep{wilcoxon1945individual,shapiro1965analysis} used for significance.

\section{Additional Results}

\subsection{Wildfire Suppression}
\label{wildfire}

We simulate \emph{task openness} and \emph{agent openness} using the domain of wildfire suppression. The \textsc{Wildfire} domain~\citep{patino2025inaugural} was developed and validated within OASYS, provides a publicly available codebase with documented configurations, and natively exhibits both AO (agents exit and re-enter as suppressant depletes) and both forms of TO (endogenous spread and exogenous random ignitions) without artificial modification, while supporting varied grid sizes for systematic generalization analysis.

\begin{figure}[htbp]
  \centering
  \includegraphics[width=0.7\linewidth]{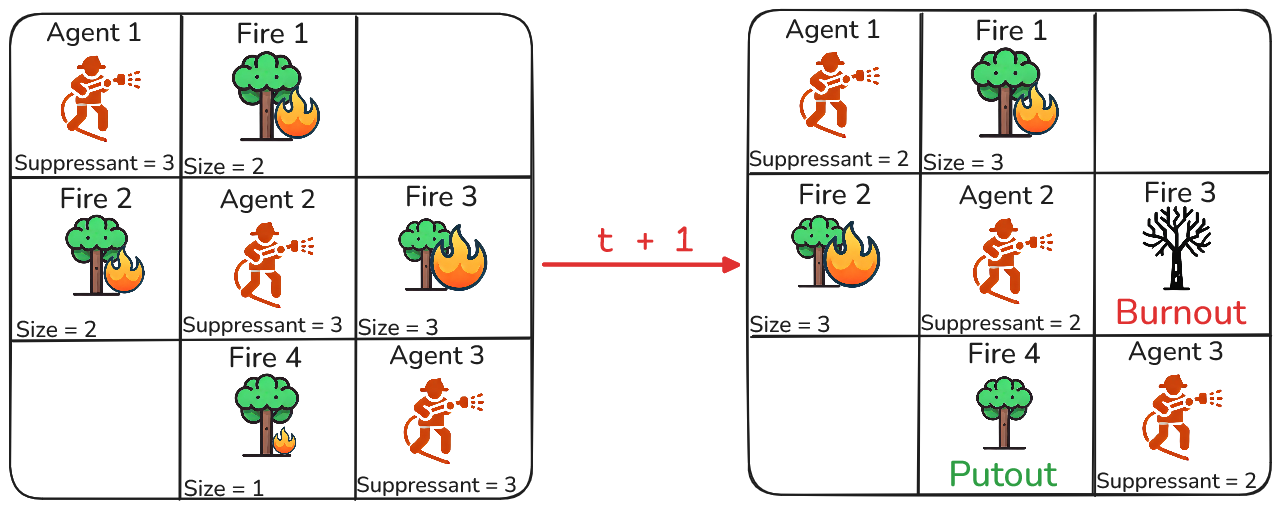}
  \caption{ \textsc{Wildfire} example ($3\times 3$). Fires (tasks $X_j$) ignite, spread, burn out, or are put out (i.e., extinguished) while agents ($N_i$) deplete suppressant or leave for repairs. Both sets evolve over time, creating an open multi-agent system.}
  \label{fig:wildfire}
\end{figure}
Each firefighter, $N_i$, has a suppression range, a finite suppressant level, and equipment health. Agents cannot move between cells; they act only on fires within range. When agent $i$ depletes its suppressant, it exits to recharge suppressant or repair its equipment and may later return, introducing \emph{agent openness}.

\subsection{Setup Figures}\label{app:Setup_figures}

We define \textit{native} (i.e., training and testing on the same grid; see Figs.~\ref{fig:setups_2x3}--\ref{fig:setups_4x4}) and \textit{zero-shot} (i.e., training and testing on different grids; see Fig.~\ref{fig:zs_example}). Zero-shot evaluates generalization to unseen environments, where agents must handle unseen configurations, new spatial layouts, and different coordination patterns without retraining. The same team from training is used, but on a different grid with different number of initial and possible fires and novel positions: $2\times 3$ (3~agents, 2~fires) $\to$ $3\times 3$ (3~agents, 5~fires) and $4\times 4$ (3~agents, 7~fires); $3\times 3$ (3~agents, 4~fires) $\to$ $4\times 4$ (3~agents, 7~fires) and $5\times 5$ (3~agents, 8~fires).

\begin{figure}[ht]
\centering
\includegraphics[width=1\columnwidth]{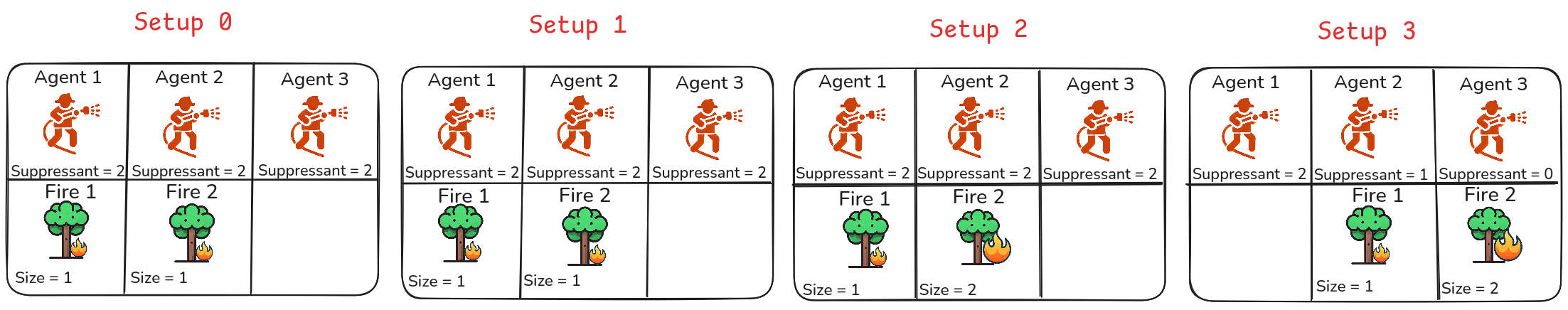}
\caption{Setups (S0--S3) on $2\times 3$ showing the initial suppressant level that each firefighter agent has and the fire size. Setup 1 enables fire spread (endogenous TO). Setup 2 replaces spread with random ignitions (exogenous TO) and introduces a larger fire type, testing a qualitatively different openness mechanism. Setup 3 is the most demanding: it reintroduces fire spread alongside random ignitions and adds uneven initial suppressant levels, combining both TO mechanisms with AO resource constraints. Setup 0 uses the same initial conditions as Setup 1 but disables both task openness and agent openness, yielding a static, non-open setting.}

\label{fig:setups_2x3}
\end{figure}

\begin{figure}[ht]
\centering
\includegraphics[width=1\columnwidth]{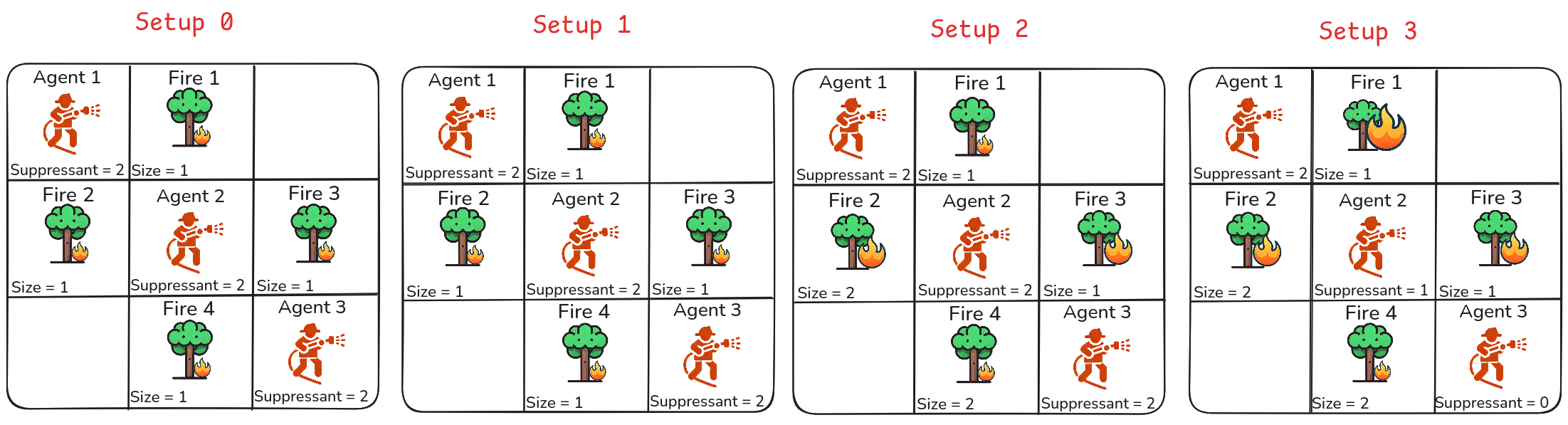}
\caption{Setups (S0--S3) on $3\times 3$ showing the suppressant level that each firefighter agent has and the fire size. Setup 1 enables fire spread (endogenous TO). Setup 2 replaces spread with random ignitions (exogenous TO) and introduces two larger fire types, testing a distinct openness mechanism. Setup 3 reintroduces fire spread alongside random ignitions and adds uneven initial suppressant levels, combining both TO mechanisms with AO resource constraints.}
\label{fig:setups_3x3}
\end{figure}

\begin{figure}[ht]
\centering
\includegraphics[width=1\columnwidth]{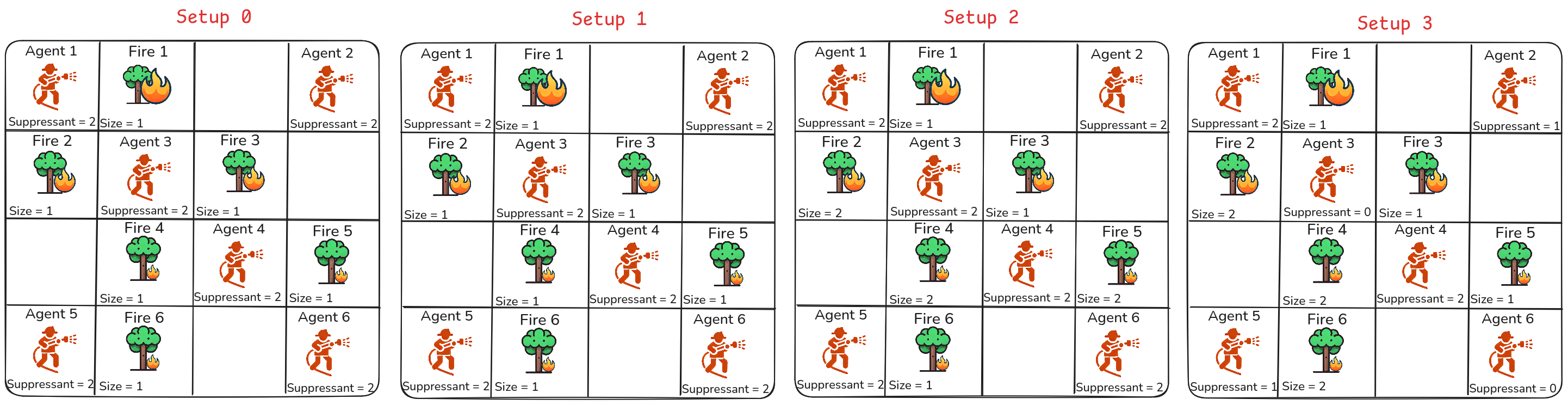}
\caption{Setups (S0--S3) on $4\times 4$ showing the suppressant level that each firefighter agent has and the fire size. Setup 1 enables fire spread (endogenous TO). Setup 2 replaces spread with random ignitions (exogenous TO) and introduces two larger fire types, testing a distinct openness mechanism. Setup 3 reintroduces fire spread alongside random ignitions and adds uneven initial suppressant across all six agents, combining both TO mechanisms with more heterogeneous AO.}
\label{fig:setups_4x4}
\end{figure}
\clearpage

\subsubsection{Zero-Shot Evaluation}\label{app:zs}

\begin{figure}[ht]
\centering
\includegraphics[width=1\columnwidth]{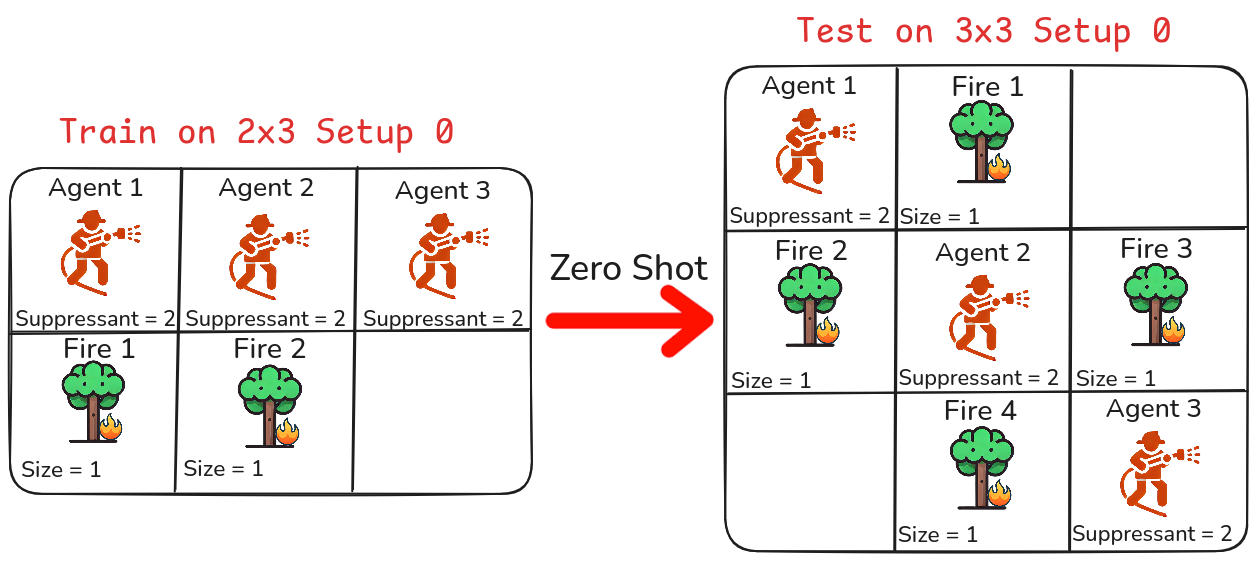}
\caption{Example zero-shot evaluation. Agents are trained on a $2\times 3$ Setup 0 environment and tested without retraining on a larger $3\times 3$ Setup 0 environment.}
\label{fig:zs_example}
\end{figure}

\subsection{Domain and Openness Setups}\label{app:domain_setups}

\begin{table} [ht]
\centering
\footnotesize
\setlength{\tabcolsep}{3pt}
\renewcommand{\arraystretch}{1.05}

\begin{tabular}{lccc}
\toprule
\textbf{Setup} & \textbf{Fire spread} & \textbf{Random ignitions} & \textbf{Refill success} \\
\midrule
S0 & $\times$ (0.0) & $\times$ (0.0) & - \\
S1 & $\checkmark$ (0.6) & $\times$ (0.0) & 1.0 \\
S2 & $\times$ (0.0) & $\checkmark$ (0.6) & 0.8 \\
S3 & $\checkmark$ (0.6) & $\checkmark$ (0.6) & 0.6 \\
\bottomrule
\end{tabular}
\caption{Configuration: Openness mechanisms across setups. Endogenous TO is fire spread, exogenous TO is random ignitions, and AO is stochastic refill success. Numbers indicate the probability per step. In S0 agents have unlimited suppressant (suppressant decrease disabled).}
\label{tab:setup_summary_app}
\end{table}

\subsection{Setup Details and Openness Statistics}\label{app:A1}

Table~\ref{tab:openness_events} summarizes how frequently openness events occur within the 100-step horizon, and Table~\ref{tab:grid_minimal} summarizes the grid-scale configuration differences.

\begin{table} [ht]
\centering
\small
\setlength{\tabcolsep}{3pt}
\renewcommand{\arraystretch}{1.05}
\begin{tabular}{l|c|c|c|c}
\toprule
 & \multicolumn{4}{c}{\textbf{Setup}} \\
\cline{2-5}
\textbf{Grid} & \multicolumn{1}{c|}{\textbf{S0}} & \multicolumn{1}{c|}{\textbf{S1}} & \multicolumn{1}{c|}{\textbf{S2}} & \multicolumn{1}{c}{\textbf{S3}} \\
\midrule
\multicolumn{5}{c}{\textbf{Zero-suppressant count (agents ran out of suppressant)}} \\
\midrule
$2\times 3$ & 0.00$\pm$0.00 & 16.25$\pm$12.35 & 56.05$\pm$28.56 & 54.06$\pm$34.85 \\
$3\times 3$ & 0.00$\pm$0.00 & 10.69$\pm$8.07  & 64.71$\pm$27.87 & 61.78$\pm$42.16 \\
$4\times 4$ & 0.00$\pm$0.00 & 14.91$\pm$10.98 & 60.52$\pm$55.31 & 54.40$\pm$60.45 \\
\midrule
\multicolumn{5}{c}{\textbf{New fires count}} \\
\midrule
$2\times 3$ & 0.00$\pm$0.00 & 3.16$\pm$6.46  & 30.46$\pm$15.74 & 19.53$\pm$15.81 \\
$3\times 3$ & 0.00$\pm$0.00 & 2.17$\pm$2.92  & 34.27$\pm$16.93 & 22.85$\pm$18.64 \\
$4\times 4$ & 0.00$\pm$0.00 & 1.97$\pm$2.92  & 33.60$\pm$34.40 & 21.97$\pm$28.37 \\
\bottomrule
\end{tabular}
\caption{Openness Event Statistics: Execution openness event frequency within the fixed 100-step episode horizon, aggregated across all models (mean $\pm$ std). S0 is closed, hence no ignitions and no zero-suppressant events. In S1--S3, both task openness (new fires) and agent openness (zero-suppressant leading to exit/refill) occur frequently.}
\label{tab:openness_events}
\end{table}

\begin{table} [ht]
\centering
\small
\setlength{\tabcolsep}{2.5pt}
\renewcommand{\arraystretch}{1.05}
\begin{tabular}{l|c|c|c}
\toprule
 & $2\times 3$ & $3\times 3$ & $4\times 4$ \\
\midrule
Cells ($rc$)                 & 6  & 9  & 16 \\
Agents ($|N_0|$)             & 3  & 3  & 6 \\
Initial fires ($|X_0|$)      & 2  & 4  & 6 \\
Init suppressant (S0--S2)& [2,2,2] & [2,2,2] & [2,2,2,2,2,2] \\
Init suppressant (S3)   & [2,1,0] & [2,1,0] & [2,1,0,2,1,0] \\
Fire types (S0--S1)      & $1\times 2$ & $1\times 4$ & $1\times 6$ \\
Fire types (S2)         & $1\times 1,\,2\times 1$ & $1\times 2,\,2\times 2$ & $1\times 4,\,2\times 2$ \\
Fire types (S3)         & $1\times 1,\,2\times 1$ & $1\times 2,\,2\times 2$ & $1\times 3,\,2\times 3$ \\
\bottomrule
\end{tabular}

\caption{Configuration: Grid details across sizes.}
\label{tab:grid_minimal}
\end{table}

\subsection{Configuration Tables}\label{app:configuration_tables}

Table~\ref{tab:configuration_details} enumerates the full grid-specific instantiations for each setup.

\begin{table*}[ht]
\centering
\small
\setlength{\tabcolsep}{2pt}
\begin{tabular}{lccc}
\toprule
\textbf{Setup} & \textbf{2$\times$3} & \textbf{3$\times$3} & \textbf{4$\times$4} \\
\midrule
0 &
\begin{tabular}[c]{@{}l@{}}Agents: 3 \\ Initial suppressant: [2,2,2] \\ Initial fires: 2 \\ Fire types: 1$\times$2\end{tabular} & \begin{tabular}[c]{@{}l@{}}Agents: 3 \\ Initial suppressant: [2,2,2] \\ Initial fires: 4 \\ Fire types: 1$\times$4\end{tabular} & \begin{tabular}[c]{@{}l@{}}Agents: 6 \\ Initial suppressant: [2,2,2,2,2,2] \\ Initial fires: 6 \\ Fire types: 1$\times$6\end{tabular} \\

1 &
\begin{tabular}[c]{@{}l@{}}Agents: 3 \\ Initial suppressant: [2,2,2] \\ Initial fires: 2 \\ Fire types: 1$\times$2\end{tabular} & \begin{tabular}[c]{@{}l@{}}Agents: 3 \\ Initial suppressant: [2,2,2] \\ Initial fires: 4 \\ Fire types: 1$\times$4\end{tabular} & \begin{tabular}[c]{@{}l@{}}Agents: 6 \\ Initial suppressant: [2,2,2,2,2,2] \\ Initial fires: 6 \\ Fire types: 1$\times$6\end{tabular} \\

2 &
\begin{tabular}[c]{@{}l@{}}Agents: 3 \\ Initial suppressant: [2,2,2] \\ Initial fires: 2 \\ Fire types: 1$\times$1, 2$\times$1\end{tabular} & \begin{tabular}[c]{@{}l@{}}Agents: 3 \\ Initial suppressant: [2,2,2] \\ Initial fires: 4 \\ Fire types: 1$\times$2, 2$\times$2\end{tabular} & \begin{tabular}[c]{@{}l@{}}Agents: 6 \\ Initial suppressant: [2,2,2,2,2,2] \\ Initial fires: 6 \\ Fire types: 1$\times$4, 2$\times$2\end{tabular} \\

3 &
\begin{tabular}[c]{@{}l@{}}Agents: 3 \\ Initial suppressant: [2,1,0] \\ Initial fires: 2 \\ Fire types: 1$\times$1, 2$\times$1\end{tabular} & \begin{tabular}[c]{@{}l@{}}Agents: 3 \\ Initial suppressant: [2,1,0] \\ Initial fires: 4 \\ Fire types: 1$\times$2, 2$\times$2\end{tabular} & \begin{tabular}[c]{@{}l@{}}Agents: 6 \\ Initial suppressant: [2,1,0,2,1,0] \\ Initial fires: 6 \\ Fire types: 1$\times$3, 2$\times$3\end{tabular} \\
\bottomrule
\end{tabular}
\caption{Configuration: Grid-specific instantiations for each setup, including the number of agents, initial suppressant levels, initial fires, and initial fire types.}
\label{tab:configuration_details}
\end{table*}

\clearpage





\subsection{Convergence, checkpoint selection, and statistical testing.}\label{sec:convergence_appendix}
In closed settings, a common stability criterion is
\begin{equation}
\left| \bar{R}_{t} - \bar{R}_{t-W} \right| < \varepsilon ,
\end{equation}
\noindent\textit{where} $\bar{R}_{t}$ is the moving-average return at episode $t$; $W$ is a smoothing window; $\varepsilon$ is a small stability threshold.

In open settings, episode returns vary under different realizations of the task and agent sets, so convergence is assessed over an evaluation distribution.
For a saved checkpoint $\theta_k$, we estimate its expected performance as
\begin{equation}
\hat{\mu}_k = \frac{1}{|\mathcal{S}|N}\sum_{s\in \mathcal{S}}\sum_{i=1}^{N} R(\theta_k; s, i),
\end{equation}
\noindent\textit{where} $\mathcal{S}$ is the set of setups; $N$ is the number of evaluation seeds per setup; $R(\theta_k; s,i)$ is the episode return of checkpoint $\theta_k$ on setup $s$ and seed $i$.

We select the top three checkpoints for each method by $\hat{\mu}_k$ over the second half of training and report mean $\pm$ std over their aggregated evaluation episodes.

\paragraph{Statistical test procedure.}\label{sec:stat_testing}
We use a \textbf{Shapiro-Wilk-adaptive two-sided test}~\citep{wilcoxon1945individual,shapiro1965analysis} for all significance claims: for each comparison, we apply a Shapiro-Wilk normality pre-test (at the same Bonferroni-corrected $\alpha$ as the primary test) to the per-seed differences; if the test fails to reject normality ($p \geq \alpha$), we use the two-sided paired $t$-test; if it rejects normality ($p < \alpha$), we use the two-sided Wilcoxon signed-rank test.
Because every method is evaluated on the same fixed seed set, each seed produces one matched observation per method under identical stochastic conditions; pairing by seed removes this shared variance and is strictly more powerful than an unpaired test on the same data~\citep{montgomery2017design}.
For the empirically best model $a$ and any competitor $b$, we form per-seed differences over their shared seed set $\mathcal{S}_{ab}$,
\begin{equation}
t \;=\; \frac{\bar{\delta}}{\,s_{\delta}/\sqrt{|\mathcal{S}_{ab}|}},
\qquad
\bar{\delta} = \frac{1}{|\mathcal{S}_{ab}|}\sum_{s \in \mathcal{S}_{ab}}\delta_s,
\quad
\delta_s = r_a(s) - r_b(s),
\label{eq:paired_t}
\end{equation}
\noindent\textit{where} $s_{\delta}$ is the sample standard deviation of $\{\delta_s\}$; differences are negated for cost metrics (lower is better).
A model is \textbf{bolded} only when it significantly outperforms \emph{every} other model in the comparison family; no cell is bolded when no model qualifies.

\paragraph{Bonferroni correction and comparison families.}
We apply Bonferroni correction $\alpha = 0.05/k$~\citep{hochberg1987multiple}, where $k$ is the number of models in the \emph{comparison family} for each table.
Bonferroni requires no assumptions on the dependence structure among tests and is the most conservative choice; with the small $k$ values used here, the power cost is negligible.
We define three families, each scoped to the scientific question of the corresponding table or figure (Table~\ref{tab:bonferroni_families}):

\begin{itemize}
    \item \textbf{Main results: all native grids} (Tables~\ref{tab:native_reward}, \ref{tab:native_rpf}, and appendix metric tables): all $k=8$ models (PLATO (MLP-Additive), DGN, MOHITO, DICG, NOOP, Random, Weakest, Strongest) giving $\alpha = 0.05/8 = 0.00625$. The family is per table (per metric, per grid), consistent with standard practice in empirical MARL~\citep{demsar2006statistical}.

    \item \textbf{Zero-shot generalization} (Table~\ref{tab:zs_reward}, Appendix~\ref{sec:zeroshot_appendix}): naive baselines have no learned policy and do not participate in the generalization question; the family is $k=4$ smart models (PLATO, DGN, MOHITO, DICG), giving $\alpha = 0.05/4 = 0.0125$.

    \item \textbf{Ablation study} (Table~\ref{tab:ablation_reward} in main text; additional metrics in Appendix~\ref{sec:ablation_appendix}): only PLATO variants are compared; the $k=4$ variants are (mlp+add) (as reference), (mlp+dot), (lstm+add), and (lstm+dot), giving $\alpha = 0.05/4 = 0.0125$.

\end{itemize}

\begin{table}[ht]
\centering
\small
\setlength{\tabcolsep}{5pt}
\renewcommand{\arraystretch}{1.1}
\begin{tabular}{p{2.5cm}p{6.5cm}ccr}
\toprule
\textbf{Family} & \textbf{Models included} & $k$ & $\alpha=0.05/k$ & $p$-threshold \\
\midrule
Native results & PLATO, DGN, DICG, MOHITO, NOOP, Random, Weakest, Strongest & 8 & $0.05/8$ & $0.00625$ \\
Zero-shot & PLATO, DGN, DICG, MOHITO (smart models only; naive baselines excluded) & 4 & $0.05/4$ & $0.01250$ \\
Ablation & 4 PLATO variants: mlp+add, mlp+dot, lstm+add, lstm+dot & 4 & $0.05/4$ & $0.0125$ \\
\bottomrule
\end{tabular}
\caption{Bonferroni correction families. For each family, $\alpha = 0.05/k$ is the per-test significance threshold applied per setup per table. A bold entry means the empirically best model significantly outperforms \emph{all} others in the family at that threshold.}
\label{tab:bonferroni_families}
\end{table}



\subsection{Detailed Results}\label{Detailed_Results}

This appendix provides detailed results on performance and efficiency (burnouts, putouts, NOOP\%, and reward per fight) for all experimental configurations.
These metrics complement the episode returns reported in the main text and offer insight into the behavioral differences between methods. Putouts measure successful task completion. Burnouts measure failures (fires reaching terminal intensity). NOOP\% measures action efficiency and implicit idleness induced by AO or conservative policies.

\subsubsection{Native Training: All Metrics}\label{Additional_Results}

PLATO refers to the mlp+add variant unless otherwise noted. Tables~\ref{tab:app_native_reward_2x3} and~\ref{tab:app_native_rpf_2x3} report episode return and reward per fight on $2\times 3$ (referred to from the main text). Tables~\ref{tab:app_native_noop_2x3}--\ref{tab:app_native_putout_2x3} report NOOP, burnouts, and putouts on $2\times 3$. Tables~\ref{tab:app_native_noop_3x3_4x4}--\ref{tab:app_native_putout_3x3_4x4} report the same metrics on $3\times 3$ and $4\times 4$.

\begin{table*}[!t]
\centering\setlength{\tabcolsep}{1.4pt}\renewcommand{\arraystretch}{1.02}
\resizebox{0.55\linewidth}{!}{%
\begin{tabular}{@{\hspace{2pt}}l|r|r|r|r@{\hspace{2pt}}}
\toprule
 & \multicolumn{4}{c}{$2\times3$} \\
\cmidrule(lr){2-5}
Model & \multicolumn{1}{c|}{S0} & \multicolumn{1}{c|}{S1} & \multicolumn{1}{c|}{S2} & \multicolumn{1}{c}{S3} \\
\midrule
PLATO    & 4.00$\pm$0.00 & \textbf{52.84$\pm$25.62} & \textbf{68.17$\pm$17.80} & \textbf{76.07$\pm$10.73} \\
DICG     & 4.00$\pm$0.00 & 18.92$\pm$14.81          & 51.76$\pm$12.71 & 70.80$\pm$7.81           \\
DGN      & 3.92$\pm$0.56 & 14.41$\pm$11.03          & 45.77$\pm$15.29 & 69.05$\pm$10.93          \\
MOHITO   & 4.00$\pm$0.00 & 9.67$\pm$7.30            & 8.08$\pm$14.66  & 58.20$\pm$14.46          \\
\midrule
NOOP     & $-$4.00$\pm$0.00 & $-$4.00$\pm$0.00 & $-$6.00$\pm$0.00 & $-$6.00$\pm$0.00 \\
Random   & 3.92$\pm$0.57 & 4.36$\pm$2.27   & 3.56$\pm$11.24  & 34.48$\pm$18.00 \\
Weakest  & 4.00$\pm$0.00 & 3.68$\pm$0.74   & 64.00$\pm$0.00  & 64.00$\pm$0.00  \\
Strongest& 4.00$\pm$0.00 & 5.24$\pm$3.68   & 24.44$\pm$25.75 & 63.12$\pm$1.08  \\
\bottomrule
\end{tabular}%
}
\caption{Performance: Episode return (mean $\pm$ std), trained and tested on $2\times3$. Bold indicates statistically significant best performance within each setup (Shapiro-Wilk-adaptive two-sided test; Bonferroni-corrected, $p{<}0.00625$).}
\label{tab:app_native_reward_2x3}
\end{table*}

\begin{table*}[!t]
\centering\setlength{\tabcolsep}{1.4pt}\renewcommand{\arraystretch}{1.02}
\resizebox{0.55\linewidth}{!}{%
\begin{tabular}{@{\hspace{2pt}}l|r|r|r|r@{\hspace{2pt}}}
\toprule
 & \multicolumn{4}{c}{$2\times3$} \\
\cmidrule(lr){2-5}
Model & \multicolumn{1}{c|}{S0} & \multicolumn{1}{c|}{S1} & \multicolumn{1}{c|}{S2} & \multicolumn{1}{c}{S3} \\
\midrule
PLATO    & 1.27$\pm$0.45 & 0.52$\pm$0.12 & \textbf{0.67$\pm$0.17} & \textbf{0.58$\pm$0.10} \\
DICG     & 1.12$\pm$0.37 & 0.48$\pm$0.31 & 0.50$\pm$0.14          & 0.56$\pm$0.07 \\
DGN      & 1.15$\pm$0.54 & 0.41$\pm$0.41 & 0.46$\pm$0.19          & 0.55$\pm$0.10 \\
MOHITO   & 1.04$\pm$0.22 & 0.58$\pm$0.50 & $-$0.07$\pm$0.65       & 0.49$\pm$0.11 \\
\midrule
NOOP     & $-$4.00$\pm$0.00 & $-$4.00$\pm$0.00 & $-$6.00$\pm$0.00 & $-$6.00$\pm$0.00 \\
Random   & 1.15$\pm$0.54 & 0.79$\pm$0.68 & $-$0.08$\pm$0.33 & 0.38$\pm$0.13 \\
Weakest  & 1.21$\pm$0.28 & 1.15$\pm$0.43 & 0.53$\pm$0.00    & 0.43$\pm$0.00 \\
Strongest& 1.21$\pm$0.28 & \textbf{1.19$\pm$0.34} & 0.10$\pm$0.30 & 0.43$\pm$0.01 \\
\bottomrule
\end{tabular}%
}
\caption{Efficiency: Reward per fight (mean $\pm$ std), trained and tested on $2\times3$. Bold indicates statistically significant best performance within each setup (Shapiro-Wilk-adaptive two-sided test; Bonferroni-corrected, $p{<}0.00625$).}
\label{tab:app_native_rpf_2x3}
\end{table*}

\begin{table*}[!t]
\centering\setlength{\tabcolsep}{1.4pt}\renewcommand{\arraystretch}{1.02}
\resizebox{0.55\linewidth}{!}{%
\begin{tabular}{@{\hspace{2pt}}l|r|r|r|r@{\hspace{2pt}}}
\toprule
 & \multicolumn{4}{c}{$2\times3$} \\
\cmidrule(lr){2-5}
Model & \multicolumn{1}{c|}{S0} & \multicolumn{1}{c|}{S1} & \multicolumn{1}{c|}{S2} & \multicolumn{1}{c}{S3} \\
\midrule
PLATO    & 17.26$\pm$19.15 & \textbf{64.29$\pm$17.82} & 64.82$\pm$4.45 & 56.25$\pm$2.30 \\
DICG     & 11.53$\pm$19.80 & 82.86$\pm$14.28          & 64.82$\pm$3.05 & 57.91$\pm$2.48 \\
DGN      & 37.55$\pm$23.75 & 81.59$\pm$13.89          & 65.69$\pm$3.83 & 58.17$\pm$3.33 \\
MOHITO   & 55.56$\pm$9.10  & 89.81$\pm$9.23           & 67.50$\pm$17.62& 60.46$\pm$6.10 \\
\midrule
NOOP     & 100.00$\pm$0.00 & 100.00$\pm$0.00 & 100.00$\pm$0.00 & 100.00$\pm$0.00 \\
Random   & 37.92$\pm$24.17 & 94.83$\pm$7.96  & 63.25$\pm$9.93  & 63.97$\pm$2.39  \\
Weakest  & 3.56$\pm$8.23   & 98.57$\pm$0.99  & 60.00$\pm$0.00  & 50.67$\pm$0.00  \\
Strongest& 3.56$\pm$8.23   & 97.61$\pm$4.04  & \textbf{44.87$\pm$11.81} & 50.67$\pm$0.00 \\
\bottomrule
\end{tabular}%
}
\caption{Efficiency: NOOP\% (mean $\pm$ std), trained and tested on $2\times3$. Bold indicates statistically significant minimum (lowest idle rate) within each setup (Shapiro-Wilk-adaptive two-sided test; Bonferroni-corrected, $p{<}0.00625$).}
\label{tab:app_native_noop_2x3}
\end{table*}

\begin{table*}[!t]
\centering\setlength{\tabcolsep}{1.4pt}\renewcommand{\arraystretch}{1.02}
\resizebox{0.55\linewidth}{!}{%
\begin{tabular}{@{\hspace{2pt}}l|r|r|r|r@{\hspace{2pt}}}
\toprule
 & \multicolumn{4}{c}{$2\times3$} \\
\cmidrule(lr){2-5}
Model & \multicolumn{1}{c|}{S0} & \multicolumn{1}{c|}{S1} & \multicolumn{1}{c|}{S2} & \multicolumn{1}{c}{S3} \\
\midrule
PLATO    & 0.00$\pm$0.00 & 0.01$\pm$0.08 & 1.07$\pm$0.25 & 1.01$\pm$0.08 \\
DICG     & 0.00$\pm$0.00 & 0.36$\pm$0.48 & 1.05$\pm$0.23 & 1.00$\pm$0.00 \\
DGN      & 0.02$\pm$0.14 & 0.80$\pm$0.54 & 1.10$\pm$0.30 & 1.01$\pm$0.12 \\
MOHITO   & 0.00$\pm$0.00 & 0.45$\pm$0.55 & 1.87$\pm$0.33 & 1.03$\pm$0.18 \\
\midrule
NOOP     & 2.00$\pm$0.00 & 2.00$\pm$0.00 & 2.00$\pm$0.00 & 2.00$\pm$0.00 \\
Random   & 0.02$\pm$0.14 & 0.62$\pm$0.57 & 1.94$\pm$0.24 & 1.38$\pm$0.49 \\
Weakest  & 0.00$\pm$0.00 & 0.16$\pm$0.37 & 1.00$\pm$0.00 & 1.00$\pm$0.00 \\
Strongest& 0.00$\pm$0.00 & 0.00$\pm$0.00 & 1.40$\pm$0.49 & 1.00$\pm$0.00 \\
\bottomrule
\end{tabular}%
}
\caption{Performance: Burnouts per episode (mean $\pm$ std), trained and tested on $2\times3$. Bold indicates statistically significant minimum within each setup (Shapiro-Wilk-adaptive two-sided test; Bonferroni-corrected, $p{<}0.00625$).}
\label{tab:app_native_burnout_2x3}
\end{table*}

\begin{table*}[!t]
\centering\setlength{\tabcolsep}{1.4pt}\renewcommand{\arraystretch}{1.02}
\resizebox{0.55\linewidth}{!}{%
\begin{tabular}{@{\hspace{2pt}}l|r|r|r|r@{\hspace{2pt}}}
\toprule
 & \multicolumn{4}{c}{$2\times3$} \\
\cmidrule(lr){2-5}
Model & \multicolumn{1}{c|}{S0} & \multicolumn{1}{c|}{S1} & \multicolumn{1}{c|}{S2} & \multicolumn{1}{c}{S3} \\
\midrule
PLATO    & 2.00$\pm$0.00 & \textbf{26.43$\pm$12.80} & \textbf{35.65$\pm$8.83} & \textbf{39.73$\pm$5.28} \\
DICG     & 2.00$\pm$0.00 & 9.82$\pm$7.36            & 27.87$\pm$6.18  & 37.40$\pm$3.91          \\
DGN      & 1.98$\pm$0.14 & 8.01$\pm$5.49            & 24.90$\pm$7.44  & 36.54$\pm$5.40          \\
MOHITO   & 2.00$\pm$0.00 & 5.29$\pm$3.70            & 6.03$\pm$6.82   & 31.13$\pm$7.15          \\
\midrule
NOOP     & 0.00$\pm$0.00 & 0.00$\pm$0.00 & 0.00$\pm$0.00   & 0.00$\pm$0.00   \\
Random   & 1.98$\pm$0.14 & 2.80$\pm$1.28 & 4.70$\pm$5.53   & 19.62$\pm$8.60  \\
Weakest  & 2.00$\pm$0.00 & 2.00$\pm$0.00 & 34.00$\pm$0.00  & 34.00$\pm$0.00  \\
Strongest& 2.00$\pm$0.00 & 2.62$\pm$1.84 & 13.14$\pm$12.60 & 33.54$\pm$0.50  \\
\bottomrule
\end{tabular}%
}
\caption{Performance: Putouts per episode (mean $\pm$ std), trained and tested on $2\times3$. Bold indicates statistically significant best performance within each setup (Shapiro-Wilk-adaptive two-sided test; Bonferroni-corrected, $p{<}0.00625$).}
\label{tab:app_native_putout_2x3}
\end{table*}

\begin{table*}[!t]
\centering\setlength{\tabcolsep}{1.4pt}\renewcommand{\arraystretch}{1.02}
\resizebox{\linewidth}{!}{%
\begin{tabular}{@{\hspace{2pt}}l|r|r|r|r||r|r|r|r@{\hspace{2pt}}}
\toprule
 & \multicolumn{4}{c||}{$3\times3$} & \multicolumn{4}{c}{$4\times4$} \\
\cmidrule(lr){2-5}\cmidrule(lr){6-9}
Model & \multicolumn{1}{c|}{S0} & \multicolumn{1}{c|}{S1} & \multicolumn{1}{c|}{S2} & \multicolumn{1}{c||}{S3} & \multicolumn{1}{c|}{S0} & \multicolumn{1}{c|}{S1} & \multicolumn{1}{c|}{S2} & \multicolumn{1}{c}{S3} \\
\midrule
PLATO    & 30.39$\pm$10.46 & 94.38$\pm$2.78  & 49.15$\pm$3.22  & 57.76$\pm$7.07  & 37.31$\pm$11.69 & 90.62$\pm$5.36 & 47.42$\pm$3.41 & \textbf{62.22$\pm$5.14} \\
DICG     & 53.96$\pm$12.54 & 93.36$\pm$2.99  & 55.52$\pm$6.65  & 62.17$\pm$5.49  & 72.51$\pm$11.87 & 95.95$\pm$1.61 & 57.67$\pm$5.21 & 69.73$\pm$4.95  \\
DGN      & 45.43$\pm$8.54  & 94.17$\pm$3.68  & 60.94$\pm$5.95  & 65.35$\pm$4.70  & 72.53$\pm$6.12  & 94.83$\pm$2.76 & 67.97$\pm$5.31 & 73.00$\pm$4.37  \\
MOHITO   & \textbf{7.41$\pm$10.51}  & 94.37$\pm$3.11  & 59.17$\pm$6.93  & 59.78$\pm$6.07  & 44.44$\pm$4.55  & 94.71$\pm$3.21 & 77.00$\pm$5.26 & 79.86$\pm$4.59  \\
\midrule
NOOP     & 100.00$\pm$0.00 & 100.00$\pm$0.00 & 100.00$\pm$0.00 & 100.00$\pm$0.00 & 100.00$\pm$0.00 & 100.00$\pm$0.00& 100.00$\pm$0.00& 100.00$\pm$0.00 \\
Random   & 46.45$\pm$15.84 & 94.33$\pm$3.02  & 70.37$\pm$4.43  & 67.39$\pm$7.99  & 61.09$\pm$15.01 & 94.79$\pm$1.98 & 77.10$\pm$3.38 & 80.07$\pm$2.14  \\
Weakest  & 12.00$\pm$11.19 & 96.36$\pm$1.50  & 52.25$\pm$8.51  & 64.63$\pm$0.10  & 28.67$\pm$13.49 & 94.33$\pm$1.95 & 57.83$\pm$3.79 & 80.50$\pm$0.00  \\
Strongest& 12.00$\pm$11.19 & 93.24$\pm$5.15  & 51.50$\pm$11.60 & 54.59$\pm$13.51 & 28.67$\pm$13.49 & 92.61$\pm$2.89 & 49.37$\pm$2.92 & 67.66$\pm$3.31  \\
\bottomrule
\end{tabular}%
}
\caption{Efficiency: NOOP\% (mean $\pm$ std), trained and tested on $3\times3$ and $4\times4$. Bold indicates statistically significant minimum (lowest idle rate) within each setup (Shapiro-Wilk-adaptive two-sided test; Bonferroni-corrected, $p{<}0.00625$).}
\label{tab:app_native_noop_3x3_4x4}
\end{table*}

\begin{table*}[!t]
\centering\setlength{\tabcolsep}{1.4pt}\renewcommand{\arraystretch}{1.02}
\resizebox{\linewidth}{!}{%
\begin{tabular}{@{\hspace{2pt}}l|r|r|r|r||r|r|r|r@{\hspace{2pt}}}
\toprule
 & \multicolumn{4}{c||}{$3\times3$} & \multicolumn{4}{c}{$4\times4$} \\
\cmidrule(lr){2-5}\cmidrule(lr){6-9}
Model & \multicolumn{1}{c|}{S0} & \multicolumn{1}{c|}{S1} & \multicolumn{1}{c|}{S2} & \multicolumn{1}{c||}{S3} & \multicolumn{1}{c|}{S0} & \multicolumn{1}{c|}{S1} & \multicolumn{1}{c|}{S2} & \multicolumn{1}{c}{S3} \\
\midrule
PLATO    & 0.00$\pm$0.00 & 0.12$\pm$0.36 & \textbf{2.07$\pm$0.25} & \textbf{2.62$\pm$0.55} & 0.00$\pm$0.00 & \textbf{0.39$\pm$0.49} & \textbf{2.17$\pm$0.40} & \textbf{3.57$\pm$0.58} \\
DICG     & 0.00$\pm$0.00 & 0.51$\pm$0.63 & 2.97$\pm$0.35          & 3.19$\pm$0.44          & 0.00$\pm$0.00 & 0.62$\pm$0.67          & 3.77$\pm$0.57          & 4.59$\pm$0.65          \\
DGN      & 0.09$\pm$0.29 & 0.62$\pm$0.70 & 3.09$\pm$0.39          & 3.30$\pm$0.47          & 0.05$\pm$0.21 & 1.11$\pm$0.84          & 4.68$\pm$0.52          & 4.85$\pm$0.43          \\
MOHITO   & 0.00$\pm$0.00 & 1.24$\pm$0.91 & 3.01$\pm$0.47          & 3.11$\pm$0.47          & 0.00$\pm$0.00 & 2.17$\pm$0.84          & 4.97$\pm$0.68          & 4.99$\pm$0.57          \\
\midrule
NOOP     & 4.00$\pm$0.00 & 4.00$\pm$0.00 & 4.00$\pm$0.00 & 4.00$\pm$0.00 & 6.00$\pm$0.00 & 6.00$\pm$0.00 & 6.00$\pm$0.00 & 6.00$\pm$0.00 \\
Random   & 0.50$\pm$0.61 & 0.88$\pm$0.87 & 3.74$\pm$0.44 & 3.88$\pm$0.33 & 0.70$\pm$0.79 & 2.64$\pm$1.08 & 5.14$\pm$0.40 & 5.30$\pm$0.46 \\
Weakest  & 0.00$\pm$0.00 & 0.06$\pm$0.24 & 2.50$\pm$0.51 & 3.00$\pm$0.00 & 0.00$\pm$0.00 & 0.82$\pm$0.69 & 4.42$\pm$0.57 & 5.00$\pm$0.00 \\
Strongest& 0.00$\pm$0.00 & 0.18$\pm$0.44 & 3.28$\pm$0.64 & 3.32$\pm$0.59 & 0.00$\pm$0.00 & 1.88$\pm$0.96 & 3.18$\pm$0.48 & 4.92$\pm$0.27 \\
\bottomrule
\end{tabular}%
}
\caption{Performance: Burnouts per episode (mean $\pm$ std), trained and tested on $3\times3$ and $4\times4$. Bold indicates statistically significant minimum within each setup (Shapiro-Wilk-adaptive two-sided test; Bonferroni-corrected, $p{<}0.00625$).}
\label{tab:app_native_burnout_3x3_4x4}
\end{table*}

\begin{table*}[!t]
\centering\setlength{\tabcolsep}{1.4pt}\renewcommand{\arraystretch}{1.02}
\resizebox{\linewidth}{!}{%
\begin{tabular}{@{\hspace{2pt}}l|r|r|r|r||r|r|r|r@{\hspace{2pt}}}
\toprule
 & \multicolumn{4}{c||}{$3\times3$} & \multicolumn{4}{c}{$4\times4$} \\
\cmidrule(lr){2-5}\cmidrule(lr){6-9}
Model & \multicolumn{1}{c|}{S0} & \multicolumn{1}{c|}{S1} & \multicolumn{1}{c|}{S2} & \multicolumn{1}{c||}{S3} & \multicolumn{1}{c|}{S0} & \multicolumn{1}{c|}{S1} & \multicolumn{1}{c|}{S2} & \multicolumn{1}{c}{S3} \\
\midrule
PLATO    & 4.00$\pm$0.00 & 3.88$\pm$0.36 & \textbf{60.48$\pm$6.22}  & \textbf{44.89$\pm$9.46}  & 6.00$\pm$0.00 & \textbf{10.13$\pm$4.50} & \textbf{118.58$\pm$10.79} & \textbf{81.19$\pm$13.19} \\
DICG     & 4.00$\pm$0.00 & 3.49$\pm$0.63 & 34.85$\pm$6.68           & 25.60$\pm$7.66           & 6.00$\pm$0.00 & 6.67$\pm$1.95           & 70.08$\pm$10.66           & 44.75$\pm$11.32          \\
DGN      & 3.91$\pm$0.29 & 3.38$\pm$0.70 & 27.39$\pm$7.24           & 21.08$\pm$7.99           & 5.94$\pm$0.26 & 6.53$\pm$2.12           & 48.17$\pm$9.98            & 39.90$\pm$9.40           \\
MOHITO   & 4.00$\pm$0.00 & 2.76$\pm$0.91 & 26.18$\pm$9.63           & 28.73$\pm$8.82           & 6.00$\pm$0.00 & 5.23$\pm$2.77           & 27.92$\pm$13.87           & 24.93$\pm$9.54           \\
\midrule
NOOP     & 0.00$\pm$0.00 & 0.00$\pm$0.00 & 0.00$\pm$0.00  & 0.00$\pm$0.00  & 0.00$\pm$0.00 & 0.00$\pm$0.00  & 0.00$\pm$0.00   & 0.00$\pm$0.00   \\
Random   & 3.50$\pm$0.61 & 3.12$\pm$0.87 & 11.48$\pm$7.72 & 6.32$\pm$6.24  & 5.30$\pm$0.79 & 4.24$\pm$1.25  & 31.96$\pm$10.25 & 26.06$\pm$10.86 \\
Weakest  & 4.00$\pm$0.00 & 3.94$\pm$0.24 & 39.68$\pm$10.83& 23.72$\pm$0.86 & 6.00$\pm$0.00 & 6.00$\pm$0.99  & 75.10$\pm$9.62  & 32.70$\pm$0.46  \\
Strongest& 4.00$\pm$0.00 & 3.82$\pm$0.44 & 17.40$\pm$13.95& 12.28$\pm$9.63 & 6.00$\pm$0.00 & 5.72$\pm$1.65  & 85.80$\pm$5.94  & 48.24$\pm$5.76  \\
\bottomrule
\end{tabular}%
}
\caption{Performance: Putouts per episode (mean $\pm$ std), trained and tested on $3\times3$ and $4\times4$. Bold indicates statistically significant best performance within each setup (Shapiro-Wilk-adaptive two-sided test; Bonferroni-corrected, $p{<}0.00625$).}
\label{tab:app_native_putout_3x3_4x4}
\end{table*}

\subparagraph{Episode Return (Table~\ref{tab:app_native_reward_2x3} together with the main-text native return table).}
PLATO achieves statistically significant best performance in Setups~1--3 on $2\times 3$, demonstrating clear differentiation across all open configurations; DICG, DGN, and MOHITO lag substantially under task and agent openness. On $3\times 3$, PLATO achieves the highest returns with statistical significance in Setups~2--3. On the $4\times 4$ grid, PLATO achieves statistically significant highest returns in Setups~1--3, demonstrating its advantage when task diversity and agent openness co-occur. In Setup~0, all methods achieve equivalent returns across grids.

\subparagraph{Burnouts (Tables~\ref{tab:app_native_burnout_2x3} and~\ref{tab:app_native_burnout_3x3_4x4}).}
On $2\times 3$, burnout counts are similar across smart methods in all open setups; MOHITO and DGN tend to incur more in Setup~1. On $3\times 3$, PLATO achieves statistically significant fewest burnouts in Setups~2--3, with the margin over baselines widening under combined openness. On $4\times 4$, PLATO achieves statistically significant fewest burnouts in Setups~1--3, reflecting more effective task prioritization. In Setup~0, all smart methods achieve near-zero burnouts on all grids.

\subparagraph{Putouts (Tables~\ref{tab:app_native_putout_2x3} and~\ref{tab:app_native_putout_3x3_4x4}).}
PLATO achieves statistically significant best task completions in Setups~1--3 on $2\times 3$, with the advantage persisting through combined task and agent openness. On $3\times 3$, PLATO achieves statistically significant best putouts in Setups~2--3, with the gap widening under higher openness. On $4\times 4$, PLATO achieves statistically significant best putouts in Setups~1--3. MOHITO achieves the fewest putouts under open setups at $4\times 4$, consistent with its higher NOOP rates.

\subparagraph{NOOP Usage (Tables~\ref{tab:app_native_noop_2x3} and~\ref{tab:app_native_noop_3x3_4x4}).}
PLATO achieves the statistically significant lowest idle rate in Setup~1 on $2\times 3$; Strongest achieves the statistically significant lowest idle rate in Setup~2 ($44.87\%$); no method reaches significance in Setup~0 or Setup~3. On $3\times 3$, MOHITO achieves the statistically significant lowest NOOP in Setup~0 due to its aggressive suppression policy. On $4\times 4$, PLATO achieves the statistically significant lowest idle rate in Setup~3. All methods show elevated NOOP in Setup~1 on $3\times 3$ due to the higher agent-to-fire ratio.

\subparagraph{Reward per Fight (Table~\ref{tab:app_native_rpf_2x3} together with the main-text native reward-per-fight table).}
On $2\times 3$, Strongest achieves the statistically significant highest reward per fight in Setup~1, and PLATO achieves it in Setups~2--3. On $3\times 3$, PLATO achieves the statistically significant highest reward per fight in Setups~0, 2--3, reflecting efficient suppressant allocation; Weakest achieves it in Setup~1. On $4\times 4$, PLATO leads S2--S3, MOHITO leads S0, and DICG leads S1, all with statistical significance.

\subsubsection{Zero-Shot: Detailed Results}\label{sec:zeroshot_appendix}

The main-text zero-shot episode-return table reports episode return for zero-shot transfer. Tables~\ref{tab:app_zs_noop}--\ref{tab:app_zs_putout} provide NOOP\%, burnouts, and putouts for policies trained on $2\times 3$ (tested on $3\times 3$ and $4\times 4$) and trained on $3\times 3$ (tested on $4\times 4$ and $5\times 5$).

\begin{table*}[!t]
\centering\setlength{\tabcolsep}{1.4pt}\renewcommand{\arraystretch}{1.02}
\resizebox{\linewidth}{!}{%
\begin{tabular}{@{\hspace{2pt}}l|r|r|r|r||r|r|r|r||r|r|r|r||r|r|r|r@{\hspace{2pt}}}
\toprule
 & \multicolumn{8}{c||}{Trained on $2\times3$} & \multicolumn{8}{c}{Trained on $3\times3$} \\
\cmidrule(lr){2-9}\cmidrule(lr){10-17}
 & \multicolumn{4}{c||}{$3\times3$} & \multicolumn{4}{c||}{$4\times4$} & \multicolumn{4}{c||}{$4\times4$} & \multicolumn{4}{c}{$5\times5$} \\
\cmidrule(lr){2-5}\cmidrule(lr){6-9}\cmidrule(lr){10-13}\cmidrule(lr){14-17}
Model & \multicolumn{1}{c|}{S0} & \multicolumn{1}{c|}{S1} & \multicolumn{1}{c|}{S2} & \multicolumn{1}{c||}{S3} & \multicolumn{1}{c|}{S0} & \multicolumn{1}{c|}{S1} & \multicolumn{1}{c|}{S2} & \multicolumn{1}{c||}{S3} & \multicolumn{1}{c|}{S0} & \multicolumn{1}{c|}{S1} & \multicolumn{1}{c|}{S2} & \multicolumn{1}{c||}{S3} & \multicolumn{1}{c|}{S0} & \multicolumn{1}{c|}{S1} & \multicolumn{1}{c|}{S2} & \multicolumn{1}{c}{S3} \\
\midrule
PLATO  & 18.43$\pm$14.55 & \textbf{87.04$\pm$7.65} & \textbf{51.19$\pm$5.34} & \textbf{53.80$\pm$12.92} & 24.03$\pm$15.80 & 86.51$\pm$8.30 & \textbf{49.32$\pm$4.96} & 56.16$\pm$8.73 & 26.71$\pm$14.04 & 90.18$\pm$6.05 & 50.44$\pm$9.25 & 58.87$\pm$4.14 & 28.81$\pm$14.54 & 92.03$\pm$3.80 & \textbf{46.36$\pm$7.08} & \textbf{56.28$\pm$7.37} \\
DICG   & \textbf{6.58$\pm$12.55} & 93.04$\pm$2.78 & 53.99$\pm$7.23 & 67.41$\pm$3.89 & 22.87$\pm$27.10 & 88.20$\pm$7.15 & 52.61$\pm$3.83 & 56.72$\pm$4.41 & \textbf{13.73$\pm$20.82} & 90.03$\pm$6.28 & 51.13$\pm$3.76 & 54.93$\pm$3.95 & 17.09$\pm$16.61 & 91.08$\pm$3.28 & 59.92$\pm$6.90 & 64.71$\pm$4.49 \\
DGN    & 48.57$\pm$14.00 & 93.41$\pm$3.25 & 58.38$\pm$5.18 & 66.10$\pm$4.33 & 49.86$\pm$12.57 & 87.51$\pm$6.82 & 59.13$\pm$5.61 & 57.24$\pm$4.60 & 48.29$\pm$13.27 & 91.39$\pm$5.04 & 58.22$\pm$3.64 & 59.94$\pm$4.21 & 36.10$\pm$16.07 & 90.83$\pm$3.50 & 63.75$\pm$5.51 & 67.44$\pm$3.63 \\
MOHITO & 55.43$\pm$18.20 & 93.54$\pm$5.11 & 64.49$\pm$9.46 & 72.50$\pm$6.75 & 52.54$\pm$8.29 & 92.32$\pm$4.89 & 73.04$\pm$8.92 & 64.10$\pm$5.39 & 39.16$\pm$10.77 & 91.30$\pm$6.50 & 55.21$\pm$8.25 & \textbf{51.56$\pm$2.86} & 20.83$\pm$9.11 & 92.20$\pm$3.30 & 54.50$\pm$9.04 & 64.17$\pm$5.23 \\
\bottomrule
\end{tabular}%
}
\caption{Generalizability: NOOP\% (mean $\pm$ std); trained on $2\times3$, tested on $3\times3$ and $4\times4$; trained on $3\times3$, tested on $4\times4$ and $5\times5$. Bold indicates statistically significant minimum (lowest idle rate) within each setup (Shapiro-Wilk-adaptive two-sided test; Bonferroni-corrected, $p{<}0.0125$).}
\label{tab:app_zs_noop}
\end{table*}

\begin{table*}[!t]
\centering\setlength{\tabcolsep}{1.4pt}\renewcommand{\arraystretch}{1.02}
\resizebox{\linewidth}{!}{%
\begin{tabular}{@{\hspace{2pt}}l|r|r|r|r||r|r|r|r||r|r|r|r||r|r|r|r@{\hspace{2pt}}}
\toprule
 & \multicolumn{8}{c||}{Trained on $2\times3$} & \multicolumn{8}{c}{Trained on $3\times3$} \\
\cmidrule(lr){2-9}\cmidrule(lr){10-17}
 & \multicolumn{4}{c||}{$3\times3$} & \multicolumn{4}{c||}{$4\times4$} & \multicolumn{4}{c||}{$4\times4$} & \multicolumn{4}{c}{$5\times5$} \\
\cmidrule(lr){2-5}\cmidrule(lr){6-9}\cmidrule(lr){10-13}\cmidrule(lr){14-17}
Model & \multicolumn{1}{c|}{S0} & \multicolumn{1}{c|}{S1} & \multicolumn{1}{c|}{S2} & \multicolumn{1}{c||}{S3} & \multicolumn{1}{c|}{S0} & \multicolumn{1}{c|}{S1} & \multicolumn{1}{c|}{S2} & \multicolumn{1}{c||}{S3} & \multicolumn{1}{c|}{S0} & \multicolumn{1}{c|}{S1} & \multicolumn{1}{c|}{S2} & \multicolumn{1}{c||}{S3} & \multicolumn{1}{c|}{S0} & \multicolumn{1}{c|}{S1} & \multicolumn{1}{c|}{S2} & \multicolumn{1}{c}{S3} \\
\midrule
PLATO  & 0.01$\pm$0.08 & \textbf{0.74$\pm$0.73} & 3.97$\pm$0.20 & 4.63$\pm$0.48 & 0.57$\pm$0.63 & 3.63$\pm$0.89 & 6.01$\pm$0.12 & 6.53$\pm$0.50 & 1.09$\pm$0.67 & 3.83$\pm$0.92 & 6.05$\pm$0.61 & 6.01$\pm$0.18 & 2.06$\pm$0.74 & 4.53$\pm$1.00 & \textbf{6.29$\pm$0.47} & \textbf{6.89$\pm$0.55} \\
DICG   & 0.00$\pm$0.00 & 1.64$\pm$0.79 & \textbf{3.87$\pm$0.38} & 4.33$\pm$0.47 & 0.68$\pm$0.50 & 3.78$\pm$0.93 & 6.01$\pm$0.22 & 6.07$\pm$0.25 & 1.00$\pm$0.37 & 3.93$\pm$0.67 & 6.00$\pm$0.12 & 6.07$\pm$0.25 & 1.69$\pm$0.46 & 5.03$\pm$0.89 & 6.99$\pm$0.37 & 7.33$\pm$0.47 \\
DGN    & 0.73$\pm$0.69 & 1.77$\pm$0.61 & 3.99$\pm$0.20 & 4.25$\pm$0.43 & 2.47$\pm$0.91 & 4.64$\pm$0.87 & 6.16$\pm$0.37 & 6.11$\pm$0.31 & 1.80$\pm$0.97 & 4.66$\pm$0.94 & 6.07$\pm$0.25 & 6.17$\pm$0.37 & 2.33$\pm$0.94 & 5.33$\pm$0.94 & 7.03$\pm$0.29 & 7.34$\pm$0.48 \\
MOHITO & 1.53$\pm$0.72 & 2.25$\pm$0.79 & 4.40$\pm$0.60 & 4.77$\pm$0.42 & 2.65$\pm$0.52 & 4.77$\pm$0.75 & 6.80$\pm$0.40 & 6.29$\pm$0.46 & \textbf{0.70$\pm$0.46} & 5.03$\pm$0.70 & 6.27$\pm$0.48 & 6.03$\pm$0.16 & 1.62$\pm$0.49 & 4.59$\pm$0.96 & 6.80$\pm$0.53 & 7.27$\pm$0.49 \\
\bottomrule
\end{tabular}%
}
\caption{Generalizability: Burnouts per episode (mean $\pm$ std); trained on $2\times3$, tested on $3\times3$ and $4\times4$; trained on $3\times3$, tested on $4\times4$ and $5\times5$. Bold indicates statistically significant minimum within each setup (Shapiro-Wilk-adaptive two-sided test; Bonferroni-corrected, $p{<}0.0125$).}
\label{tab:app_zs_burnout}
\end{table*}

\begin{table*}[!t]
\centering\setlength{\tabcolsep}{1.4pt}\renewcommand{\arraystretch}{1.02}
\resizebox{\linewidth}{!}{%
\begin{tabular}{@{\hspace{2pt}}l|r|r|r|r||r|r|r|r||r|r|r|r||r|r|r|r@{\hspace{2pt}}}
\toprule
 & \multicolumn{8}{c||}{Trained on $2\times3$} & \multicolumn{8}{c}{Trained on $3\times3$} \\
\cmidrule(lr){2-9}\cmidrule(lr){10-17}
 & \multicolumn{4}{c||}{$3\times3$} & \multicolumn{4}{c||}{$4\times4$} & \multicolumn{4}{c||}{$4\times4$} & \multicolumn{4}{c}{$5\times5$} \\
\cmidrule(lr){2-5}\cmidrule(lr){6-9}\cmidrule(lr){10-13}\cmidrule(lr){14-17}
Model & \multicolumn{1}{c|}{S0} & \multicolumn{1}{c|}{S1} & \multicolumn{1}{c|}{S2} & \multicolumn{1}{c||}{S3} & \multicolumn{1}{c|}{S0} & \multicolumn{1}{c|}{S1} & \multicolumn{1}{c|}{S2} & \multicolumn{1}{c||}{S3} & \multicolumn{1}{c|}{S0} & \multicolumn{1}{c|}{S1} & \multicolumn{1}{c|}{S2} & \multicolumn{1}{c||}{S3} & \multicolumn{1}{c|}{S0} & \multicolumn{1}{c|}{S1} & \multicolumn{1}{c|}{S2} & \multicolumn{1}{c}{S3} \\
\midrule
PLATO  & 4.99$\pm$0.08 & \textbf{4.15$\pm$0.74} & \textbf{46.35$\pm$6.24} & 11.02$\pm$13.65 & 6.43$\pm$0.63 & 6.29$\pm$4.51 & \textbf{41.06$\pm$7.73} & 17.37$\pm$17.82 & 5.91$\pm$0.67 & 6.45$\pm$4.51 & 33.63$\pm$12.27 & 38.69$\pm$7.04 & 5.94$\pm$0.74 & 5.77$\pm$2.64 & \textbf{50.69$\pm$10.43} & \textbf{36.09$\pm$12.64} \\
DICG   & 5.00$\pm$0.00 & 3.36$\pm$0.79 & 36.18$\pm$7.47 & 20.99$\pm$8.72 & 6.31$\pm$0.49 & \textbf{7.87$\pm$4.65} & 38.53$\pm$8.33 & \textbf{34.41$\pm$9.03} & 6.00$\pm$0.37 & \textbf{8.00$\pm$4.28} & \textbf{39.94$\pm$6.49} & 35.75$\pm$7.95 & 6.31$\pm$0.46 & 5.06$\pm$1.67 & 31.74$\pm$8.77 & 22.08$\pm$9.19 \\
DGN    & 4.27$\pm$0.69 & 3.23$\pm$0.61 & 31.91$\pm$6.06 & 21.83$\pm$8.88 & 4.53$\pm$0.91 & 6.75$\pm$4.22 & 28.97$\pm$12.54 & 32.55$\pm$10.34 & 5.20$\pm$0.97 & 4.37$\pm$2.09 & 32.99$\pm$9.54 & 28.13$\pm$13.01 & 5.67$\pm$0.94 & 3.93$\pm$1.39 & 25.18$\pm$7.01 & 18.45$\pm$8.73 \\
MOHITO & 3.47$\pm$0.72 & 2.73$\pm$0.78 & 18.59$\pm$12.95 & 6.46$\pm$6.40 & 4.35$\pm$0.52 & 3.94$\pm$1.82 & 3.21$\pm$4.33 & 21.90$\pm$11.07 & \textbf{6.30$\pm$0.46} & 4.17$\pm$2.89 & 26.61$\pm$15.08 & 37.91$\pm$7.51 & 6.38$\pm$0.49 & 4.38$\pm$1.31 & 31.19$\pm$10.67 & 23.95$\pm$10.53 \\
\bottomrule
\end{tabular}%
}
\caption{Generalizability: Putouts per episode (mean $\pm$ std); trained on $2\times3$, tested on $3\times3$ and $4\times4$; trained on $3\times3$, tested on $4\times4$ and $5\times5$. Bold indicates statistically significant best performance within each setup (Shapiro-Wilk-adaptive two-sided test; Bonferroni-corrected, $p{<}0.0125$).}
\label{tab:app_zs_putout}
\end{table*}

\subparagraph{Episode Return.}
Results are mixed across zero-shot target grids. In the bottom half of the main-text zero-shot episode-return table, policies trained on $3\times 3$ show that PLATO leads S1--S3 at $5\times 5$ with statistical significance, while DICG leads S1--S2 and MOHITO leads S0 at $4\times 4$. In the top half of that table, policies trained on $2\times 3$ show that PLATO leads S1--S2 at $3\times 3$ and S2 at $4\times 4$; DICG leads S1 and S3 at $4\times 4$; MOHITO shows negative returns in S3 at $3\times 3$ and S2 at $4\times 4$. DGN degrades consistently under higher openness across all training sources.

\subparagraph{Burnouts (Table~\ref{tab:app_zs_burnout}).}
Trained on $2\times 3$, PLATO achieves statistically significant fewest burnouts in S1 at $3\times 3$; DICG achieves statistically significant fewest in S2 at $3\times 3$. Trained on $3\times 3$, PLATO achieves statistically significant fewest burnouts in S2--S3 at $5\times 5$, while MOHITO achieves statistically significant fewest in S0 at $4\times 4$. All methods accumulate near-maximal burnouts in S2--S3 at $4\times 4$ regardless of training source.

\subparagraph{Putouts (Table~\ref{tab:app_zs_putout}).}
Trained on $2\times 3$, PLATO leads S1--S2 at $3\times 3$ and S2 at $4\times 4$ with statistical significance; DICG leads S1 and S3 at $4\times 4$ and MOHITO collapses in S2 ($3.21$ putouts). Trained on $3\times 3$, PLATO leads at $5\times 5$ in S2--S3 with statistical significance; DICG leads S1--S2 and MOHITO leads S0 at $4\times 4$.

\subparagraph{NOOP Usage (Table~\ref{tab:app_zs_noop}).}
When trained on $2\times 3$, PLATO achieves the statistically significant lowest idle rates in S1--S3 at $3\times 3$ and S2 at $4\times 4$; DICG achieves the statistically significant lowest NOOP in S0 at $3\times 3$. When trained on $3\times 3$, PLATO leads at $5\times 5$ in S2--S3; DICG achieves the statistically significant lowest NOOP in S0 at $4\times 4$; MOHITO achieves the statistically significant lowest NOOP in S3 at $4\times 4$. All methods show very high NOOP in Setup~1 under both training regimes, reflecting task sparsity in the agent-openness phase.

\subparagraph{Zero-Shot Summary.}
PLATO generalizes most consistently across target grids and metrics, leading task completions and reducing NOOP across S2--S3. DICG competes significantly in putouts at $4\times 4$ from both training sources. MOHITO degrades severely in S2 from $2\times 3$ training (negative returns, near-zero putouts). Training on $3\times 3$ generally yields stronger and more stable transfer than $2\times 3$ in S2--S3.

\subsection{Ablation Studies: Detailed Results}\label{sec:ablation_appendix}

This section provides detailed NOOP\%, burnout, putout, and reward-per-fight results for the PLATO ablation study. The main-text ablation episode-return table reports episode return for both native and zero-shot settings.

\subsubsection{Variant Descriptions}
We ablate two design choices:

\textbf{MLP vs.\ LSTM encoder-decoder.} \textit{The actor encodes each observation independently, with no memory of prior timesteps.} In the MLP variant (our primary), a feedforward network maps the agent's current observation directly to a query vector. The LSTM variant replaces this with a recurrent unit that maintains a hidden state across timesteps, allowing the actor to condition task selection on the history of observations within an episode. Both variants use the same additive attention scorer unless otherwise stated.

\textbf{Additive vs.\ dot-product attention scorer.} \textit{The pointer scores agent-task compatibility via inner product rather than learned projections.} The additive (Bahdanau) scorer applies learned projection matrices $W_q$, $W_K$, and a learned vector $v$ to compute $u_x = v^\top \tanh(W_K k_x + W_q q)$, capturing asymmetric relationships at the cost of extra parameters. The dot-product scorer computes $u_x = q^\top k_x$ (after projection to a shared space), using direct inner-product similarity with no additional parameters. Both scorer variants are evaluated within the MLP and LSTM encoder settings.

\subsubsection{Detailed Native Results}

Table~\ref{tab:app_abl_native_reward_2x3} reports episode return on $2\times3$. Tables~\ref{tab:app_abl_native_noop}--\ref{tab:app_abl_native_rpf} report NOOP\%, burnouts, putouts, and reward per fight for the $2\times3$ native setting. Tables~\ref{tab:app_abl_native_noop_34}--\ref{tab:app_abl_native_rpf_34} report the corresponding native metrics on $3\times3$ and $4\times4$.

\begin{table}[!t]
\centering\setlength{\tabcolsep}{1.4pt}\renewcommand{\arraystretch}{1.02}
\resizebox{0.55\linewidth}{!}{%
\begin{tabular}{@{\hspace{2pt}}l|r|r|r|r@{\hspace{2pt}}}
\toprule
 & \multicolumn{4}{c}{$2\times3$} \\
\cmidrule(lr){2-5}
\shortstack[l]{PLATO\\Variation} & \multicolumn{1}{c|}{S0} & \multicolumn{1}{c|}{S1} & \multicolumn{1}{c|}{S2} & \multicolumn{1}{c}{S3} \\
\midrule
(mlp+add)      & 4.00$\pm$0.00 & 52.84$\pm$25.62 & 68.17$\pm$17.80 & 76.07$\pm$10.73 \\
(mlp+dot)      & 4.00$\pm$0.00 & 39.99$\pm$21.20 & 68.63$\pm$13.85 & 70.84$\pm$21.94 \\
(lstm+add)     & 4.00$\pm$0.00 & 66.95$\pm$17.77 & 70.72$\pm$11.14 & 75.84$\pm$9.24 \\
(lstm+dot)     & 4.00$\pm$0.00 & 33.15$\pm$19.01 & 68.99$\pm$14.56 & 74.93$\pm$12.55 \\

\bottomrule
\end{tabular}%
}
\caption{Ablation: Episode return (mean $\pm$ std), trained and tested on $2\times3$. No variant achieves statistically significant best performance within each setup (Shapiro-Wilk-adaptive two-sided test; Bonferroni-corrected, $p{<}0.0125$).}
\label{tab:app_abl_native_reward_2x3}
\end{table}

\begin{table}[!t]
\centering\setlength{\tabcolsep}{1.4pt}\renewcommand{\arraystretch}{1.02}
\resizebox{0.55\linewidth}{!}{%
\begin{tabular}{@{\hspace{2pt}}l|r|r|r|r@{\hspace{2pt}}}
\toprule
 & \multicolumn{4}{c}{$2\times3$} \\
\cmidrule(lr){2-5}
\shortstack[l]{PLATO\\Variation} & \multicolumn{1}{c|}{S0} & \multicolumn{1}{c|}{S1} & \multicolumn{1}{c|}{S2} & \multicolumn{1}{c}{S3} \\
\midrule
(mlp+add)      & 17.26$\pm$19.15 & 64.29$\pm$17.82 & \textbf{64.82$\pm$4.45} & 56.25$\pm$2.30 \\
(mlp+dot)      & 26.98$\pm$22.99 & 65.24$\pm$17.11 & 67.61$\pm$2.16 & 57.12$\pm$3.05 \\
(lstm+add)     & 14.04$\pm$17.09 & 57.97$\pm$10.03 & 68.31$\pm$2.42 & 56.53$\pm$2.49 \\
(lstm+dot)     & 19.76$\pm$19.45 & 80.26$\pm$13.01 & 68.47$\pm$1.96 & 56.34$\pm$2.54 \\

\bottomrule
\end{tabular}%
}
\caption{Ablation: NOOP\% (mean $\pm$ std), trained and tested on $2\times3$. Bold indicates statistically significant minimum (lowest idle rate) within each setup (Shapiro-Wilk-adaptive two-sided test; Bonferroni-corrected, $p{<}0.0125$).}

\label{tab:app_abl_native_noop}
\end{table}

\begin{table*}[!t]
\centering\setlength{\tabcolsep}{1.4pt}\renewcommand{\arraystretch}{1.02}
\resizebox{\linewidth}{!}{%
\begin{tabular}{@{\hspace{2pt}}l|r|r|r|r||r|r|r|r@{\hspace{2pt}}}
\toprule
 & \multicolumn{4}{c||}{$3\times3$} & \multicolumn{4}{c}{$4\times4$} \\
\cmidrule(lr){2-5}\cmidrule(lr){6-9}
\shortstack[l]{PLATO\\Variation} & \multicolumn{1}{c|}{S0} & \multicolumn{1}{c|}{S1} & \multicolumn{1}{c|}{S2} & \multicolumn{1}{c||}{S3} & \multicolumn{1}{c|}{S0} & \multicolumn{1}{c|}{S1} & \multicolumn{1}{c|}{S2} & \multicolumn{1}{c}{S3} \\
\midrule
(mlp+add)      & 30.39$\pm$10.46 & \textbf{94.38$\pm$2.78} & 49.15$\pm$3.22 & 57.76$\pm$7.07 & 37.31$\pm$11.69 & 90.62$\pm$5.36 & 47.42$\pm$3.41 & 62.22$\pm$5.14 \\
(mlp+dot)      & 20.16$\pm$10.90 & 95.61$\pm$2.00 & 46.64$\pm$3.49 & 57.03$\pm$7.04 & 34.73$\pm$9.90 & 90.46$\pm$5.30 & 47.14$\pm$3.25 & 61.88$\pm$5.30 \\
(lstm+add)     & \textbf{13.48$\pm$12.17} & 94.92$\pm$2.75 & 49.78$\pm$3.82 & 55.81$\pm$6.66 & 40.18$\pm$11.71 & 94.99$\pm$3.69 & 47.27$\pm$3.43 & 63.46$\pm$4.25 \\
(lstm+dot)     & 35.84$\pm$13.21 & 95.10$\pm$3.19 & 48.62$\pm$4.35 & 56.90$\pm$7.24 & 34.12$\pm$17.32 & 91.94$\pm$4.64 & 46.73$\pm$3.24 & 62.44$\pm$4.50 \\

\bottomrule
\end{tabular}%
}
\caption{Ablation: NOOP\% (mean $\pm$ std), trained and tested on $3\times3$ and $4\times4$. Bold indicates statistically significant minimum (lowest idle rate) within each setup (Shapiro-Wilk-adaptive two-sided test; Bonferroni-corrected, $p{<}0.0125$).}

\label{tab:app_abl_native_noop_34}
\end{table*}

\begin{table}[!t]
\centering\setlength{\tabcolsep}{1.4pt}\renewcommand{\arraystretch}{1.02}
\resizebox{0.55\linewidth}{!}{%
\begin{tabular}{@{\hspace{2pt}}l|r|r|r|r@{\hspace{2pt}}}
\toprule
 & \multicolumn{4}{c}{$2\times3$} \\
\cmidrule(lr){2-5}
\shortstack[l]{PLATO\\Variation} & \multicolumn{1}{c|}{S0} & \multicolumn{1}{c|}{S1} & \multicolumn{1}{c|}{S2} & \multicolumn{1}{c}{S3} \\
\midrule
(mlp+add)      & 0.00$\pm$0.00 & 0.01$\pm$0.08 & 1.07$\pm$0.25 & 1.01$\pm$0.08 \\
(mlp+dot)      & 0.00$\pm$0.00 & 0.24$\pm$0.44 & 1.03$\pm$0.18 & 1.07$\pm$0.26 \\
(lstm+add)     & 0.00$\pm$0.00 & 0.03$\pm$0.18 & 1.02$\pm$0.14 & 1.00$\pm$0.00 \\
(lstm+dot)     & 0.00$\pm$0.00 & 0.02$\pm$0.14 & 1.04$\pm$0.20 & 1.01$\pm$0.12 \\

\bottomrule
\end{tabular}%
}
\caption{Ablation: Burnouts per episode (mean $\pm$ std), trained and tested on $2\times3$. Bold indicates statistically significant minimum within each setup (Shapiro-Wilk-adaptive two-sided test; Bonferroni-corrected, $p{<}0.0125$).}
\label{tab:app_abl_native_burnout}
\end{table}

\begin{table*}[!t]
\centering\setlength{\tabcolsep}{1.4pt}\renewcommand{\arraystretch}{1.02}
\resizebox{\linewidth}{!}{%
\begin{tabular}{@{\hspace{2pt}}l|r|r|r|r||r|r|r|r@{\hspace{2pt}}}
\toprule
 & \multicolumn{4}{c||}{$3\times3$} & \multicolumn{4}{c}{$4\times4$} \\
\cmidrule(lr){2-5}\cmidrule(lr){6-9}
\shortstack[l]{PLATO\\Variation} & \multicolumn{1}{c|}{S0} & \multicolumn{1}{c|}{S1} & \multicolumn{1}{c|}{S2} & \multicolumn{1}{c||}{S3} & \multicolumn{1}{c|}{S0} & \multicolumn{1}{c|}{S1} & \multicolumn{1}{c|}{S2} & \multicolumn{1}{c}{S3} \\
\midrule
(mlp+add)      & 0.00$\pm$0.00 & 0.12$\pm$0.36 & 2.07$\pm$0.25 & 2.62$\pm$0.55 & 0.00$\pm$0.00 & 0.39$\pm$0.49 & 2.17$\pm$0.40 & 3.57$\pm$0.58 \\
(mlp+dot)      & 0.00$\pm$0.00 & 0.18$\pm$0.39 & 2.09$\pm$0.29 & 2.58$\pm$0.55 & 0.00$\pm$0.00 & 0.36$\pm$0.48 & 2.23$\pm$0.47 & 3.51$\pm$0.62 \\
(lstm+add)     & 0.00$\pm$0.00 & 0.13$\pm$0.34 & 2.09$\pm$0.28 & 2.52$\pm$0.58 & 0.00$\pm$0.00 & \textbf{0.17$\pm$0.38} & 2.21$\pm$0.45 & 3.40$\pm$0.54 \\
(lstm+dot)     & 0.00$\pm$0.00 & 0.20$\pm$0.40 & 2.11$\pm$0.31 & 2.47$\pm$0.51 & 0.00$\pm$0.00 & 0.32$\pm$0.48 & 2.25$\pm$0.48 & 3.48$\pm$0.60 \\

\bottomrule
\end{tabular}%
}
\caption{Ablation: Burnouts per episode (mean $\pm$ std), trained and tested on $3\times3$ and $4\times4$. Bold indicates statistically significant minimum within each setup (Shapiro-Wilk-adaptive two-sided test; Bonferroni-corrected, $p{<}0.0125$).}

\label{tab:app_abl_native_burnout_34}
\end{table*}

\begin{table}[!t]
\centering\setlength{\tabcolsep}{1.4pt}\renewcommand{\arraystretch}{1.02}
\resizebox{0.55\linewidth}{!}{%
\begin{tabular}{@{\hspace{2pt}}l|r|r|r|r@{\hspace{2pt}}}
\toprule
 & \multicolumn{4}{c}{$2\times3$} \\
\cmidrule(lr){2-5}
\shortstack[l]{PLATO\\Variation} & \multicolumn{1}{c|}{S0} & \multicolumn{1}{c|}{S1} & \multicolumn{1}{c|}{S2} & \multicolumn{1}{c}{S3} \\
\midrule
(mlp+add)      & 2.00$\pm$0.00 & 26.43$\pm$12.80 & 35.65$\pm$8.83 & 39.73$\pm$5.28 \\
(mlp+dot)      & 2.00$\pm$0.00 & 20.23$\pm$10.42 & 36.31$\pm$6.75 & 37.15$\pm$10.69 \\
(lstm+add)     & 2.00$\pm$0.00 & 33.51$\pm$8.84 & 37.32$\pm$5.45 & 39.59$\pm$4.57 \\
(lstm+dot)     & 2.00$\pm$0.00 & 16.59$\pm$9.52 & 36.52$\pm$7.11 & 39.21$\pm$6.13 \\

\bottomrule
\end{tabular}%
}
\caption{Ablation: Putouts per episode (mean $\pm$ std), trained and tested on $2\times3$. Bold indicates statistically significant best performance within each setup (Shapiro-Wilk-adaptive two-sided test; Bonferroni-corrected, $p{<}0.0125$).}
\label{tab:app_abl_native_putout}
\end{table}

\begin{table*}[!t]
\centering\setlength{\tabcolsep}{1.4pt}\renewcommand{\arraystretch}{1.02}
\resizebox{\linewidth}{!}{%
\begin{tabular}{@{\hspace{2pt}}l|r|r|r|r||r|r|r|r@{\hspace{2pt}}}
\toprule
 & \multicolumn{4}{c||}{$3\times3$} & \multicolumn{4}{c}{$4\times4$} \\
\cmidrule(lr){2-5}\cmidrule(lr){6-9}
\shortstack[l]{PLATO\\Variation} & \multicolumn{1}{c|}{S0} & \multicolumn{1}{c|}{S1} & \multicolumn{1}{c|}{S2} & \multicolumn{1}{c||}{S3} & \multicolumn{1}{c|}{S0} & \multicolumn{1}{c|}{S1} & \multicolumn{1}{c|}{S2} & \multicolumn{1}{c}{S3} \\
\midrule
(mlp+add)      & 4.00$\pm$0.00 & 3.88$\pm$0.36 & 60.48$\pm$6.22 & 44.89$\pm$9.46 & 6.00$\pm$0.00 & 10.13$\pm$4.50 & 118.58$\pm$10.79 & 81.19$\pm$13.19 \\
(mlp+dot)      & 4.00$\pm$0.00 & 3.82$\pm$0.39 & 58.92$\pm$6.67 & 46.77$\pm$10.00 & 6.00$\pm$0.00 & 10.24$\pm$4.60 & 114.75$\pm$11.71 & 79.81$\pm$12.51 \\
(lstm+add)     & 4.00$\pm$0.00 & 3.87$\pm$0.34 & 58.37$\pm$7.06 & 45.58$\pm$10.75 & 6.00$\pm$0.00 & 8.23$\pm$3.32 & 118.00$\pm$11.99 & 82.07$\pm$11.81 \\
(lstm+dot)     & 4.00$\pm$0.00 & 3.80$\pm$0.40 & 60.94$\pm$7.47 & 48.55$\pm$8.57 & 6.00$\pm$0.00 & 10.07$\pm$4.16 & 113.65$\pm$11.42 & 80.39$\pm$12.72 \\

\bottomrule
\end{tabular}%
}
\caption{Ablation: Putouts per episode (mean $\pm$ std), trained and tested on $3\times3$ and $4\times4$. Bold indicates statistically significant best performance within each setup (Shapiro-Wilk-adaptive two-sided test; Bonferroni-corrected, $p{<}0.0125$).}
\label{tab:app_abl_native_putout_34}
\end{table*}

\begin{table}[!t]
\centering\setlength{\tabcolsep}{1.4pt}\renewcommand{\arraystretch}{1.02}
\resizebox{0.55\linewidth}{!}{%
\begin{tabular}{@{\hspace{2pt}}l|r|r|r|r@{\hspace{2pt}}}
\toprule
 & \multicolumn{4}{c}{$2\times3$} \\
\cmidrule(lr){2-5}
\shortstack[l]{PLATO\\Variation} & \multicolumn{1}{c|}{S0} & \multicolumn{1}{c|}{S1} & \multicolumn{1}{c|}{S2} & \multicolumn{1}{c}{S3} \\
\midrule
(mlp+add)      & 1.27$\pm$0.45 & 0.52$\pm$0.12 & 0.67$\pm$0.17 & 0.58$\pm$0.10 \\
(mlp+dot)      & 1.29$\pm$0.49 & 0.40$\pm$0.11 & 0.72$\pm$0.12 & 0.53$\pm$0.24 \\
(lstm+add)     & 1.29$\pm$0.44 & 0.54$\pm$0.09 & 0.75$\pm$0.11 & 0.58$\pm$0.07 \\
(lstm+dot)     & 1.29$\pm$0.48 & \textbf{0.62$\pm$0.19} & 0.73$\pm$0.17 & 0.57$\pm$0.12 \\

\bottomrule
\end{tabular}%
}
\caption{Ablation: Reward per fight (mean $\pm$ std), trained and tested on $2\times3$. Bold indicates statistically significant best performance within each setup (Shapiro-Wilk-adaptive two-sided test; Bonferroni-corrected, $p{<}0.0125$).}
\label{tab:app_abl_native_rpf}
\end{table}

\begin{table*}[!t]
\centering\setlength{\tabcolsep}{1.4pt}\renewcommand{\arraystretch}{1.02}
\resizebox{\linewidth}{!}{%
\begin{tabular}{@{\hspace{2pt}}l|r|r|r|r||r|r|r|r@{\hspace{2pt}}}
\toprule
 & \multicolumn{4}{c||}{$3\times3$} & \multicolumn{4}{c}{$4\times4$} \\
\cmidrule(lr){2-5}\cmidrule(lr){6-9}
\shortstack[l]{PLATO\\Variation} & \multicolumn{1}{c|}{S0} & \multicolumn{1}{c|}{S1} & \multicolumn{1}{c|}{S2} & \multicolumn{1}{c||}{S3} & \multicolumn{1}{c|}{S0} & \multicolumn{1}{c|}{S1} & \multicolumn{1}{c|}{S2} & \multicolumn{1}{c}{S3} \\
\midrule
(mlp+add)      & 1.25$\pm$0.23 & 0.60$\pm$0.38 & 0.74$\pm$0.08 & 0.64$\pm$0.10 & 1.09$\pm$0.25 & 0.41$\pm$0.21 & 0.73$\pm$0.06 & 0.66$\pm$0.07 \\
(mlp+dot)      & 1.11$\pm$0.15 & 0.70$\pm$0.39 & 0.69$\pm$0.08 & 0.66$\pm$0.11 & 1.06$\pm$0.22 & 0.40$\pm$0.19 & 0.70$\pm$0.06 & 0.64$\pm$0.08 \\
(lstm+add)     & 1.04$\pm$0.16 & 0.64$\pm$0.37 & 0.72$\pm$0.09 & 0.62$\pm$0.13 & 1.07$\pm$0.24 & \textbf{0.72$\pm$0.35} & 0.72$\pm$0.06 & 0.69$\pm$0.08 \\
(lstm+dot)     & 1.27$\pm$0.30 & 0.67$\pm$0.41 & 0.74$\pm$0.09 & 0.68$\pm$0.09 & 1.26$\pm$0.35 & 0.49$\pm$0.25 & 0.68$\pm$0.06 & 0.66$\pm$0.07 \\

\bottomrule
\end{tabular}%
}
\caption{Ablation: Reward per fight (mean $\pm$ std), trained and tested on $3\times3$ and $4\times4$. Bold indicates statistically significant best performance within each setup (Shapiro-Wilk-adaptive two-sided test; Bonferroni-corrected, $p{<}0.0125$).}
\label{tab:app_abl_native_rpf_34}
\end{table*}

\paragraph{Burnouts (Tables~\ref{tab:app_abl_native_burnout} and~\ref{tab:app_abl_native_burnout_34}).}
Burnout rates are similar across all variants. All variants achieve zero burnouts in Setup~0. PLATO (lstm+add) achieves significantly fewest burnouts in Setup~1 at $4\times4$ ($0.17$).

\paragraph{Putouts (Tables~\ref{tab:app_abl_native_putout} and~\ref{tab:app_abl_native_putout_34}).}
Additive variants generally achieve more putouts than dot-product variants in Setup~1, with (lstm+add) achieving the highest putouts in S1 at $2\times3$ ($33.51$) among all variants. No variant achieves statistically significant best performance in any setup.

\paragraph{NOOP Usage (Tables~\ref{tab:app_abl_native_noop} and~\ref{tab:app_abl_native_noop_34}).}
On $2\times3$ S2, mlp+add achieves the statistically significant lowest NOOP ($64.82\%$). On $3\times3$ S0, lstm+add achieves statistically significant lowest NOOP ($13.48\%$); mlp+add achieves significantly lowest NOOP in $3\times3$ S1 ($94.38\%$), though all variants remain above $94\%$, reflecting task sparsity.

\paragraph{Reward per Fight (Tables~\ref{tab:app_abl_native_rpf} and~\ref{tab:app_abl_native_rpf_34}).}
Per-fight efficiency is similar across all variants. PLATO (lstm+dot) achieves statistically significant highest per-fight efficiency in $2\times3$ S1 ($0.62$). PLATO (lstm+add) leads in $4\times4$ S1 ($0.72$, significant).

\subsubsection{Detailed Zero-Shot Results}

Tables~\ref{tab:app_abl_zs_noop}--\ref{tab:app_abl_zs_putout} report the zero-shot NOOP\%, burnouts, and putouts for policies trained on $2\times3$ and tested on $3\times3$ and $4\times4$. Tables~\ref{tab:app_abl_zs_noop_from3x3}--\ref{tab:app_abl_zs_putout_from3x3} report the corresponding zero-shot metrics for policies trained on $3\times3$ and tested on $4\times4$ and $5\times5$.

\begin{table*}[!t]
\centering\setlength{\tabcolsep}{1.4pt}\renewcommand{\arraystretch}{1.02}
\resizebox{\linewidth}{!}{%
\begin{tabular}{@{\hspace{2pt}}l|r|r|r|r||r|r|r|r@{\hspace{2pt}}}
\toprule
 & \multicolumn{8}{c}{Trained on $2\times3$} \\
\cmidrule(lr){2-9}
 & \multicolumn{4}{c||}{$3\times3$} & \multicolumn{4}{c}{$4\times4$} \\
\cmidrule(lr){2-5}\cmidrule(lr){6-9}
\shortstack[l]{PLATO\\Variation} & \multicolumn{1}{c|}{S0} & \multicolumn{1}{c|}{S1} & \multicolumn{1}{c|}{S2} & \multicolumn{1}{c||}{S3} & \multicolumn{1}{c|}{S0} & \multicolumn{1}{c|}{S1} & \multicolumn{1}{c|}{S2} & \multicolumn{1}{c}{S3} \\
\midrule
(mlp+add)      & 18.43$\pm$14.55 & 87.04$\pm$7.65 & 51.19$\pm$5.34 & 53.80$\pm$12.92 & 24.03$\pm$15.80 & 86.51$\pm$8.30 & 49.32$\pm$4.96 & 56.16$\pm$8.73 \\
(mlp+dot)      & 16.01$\pm$12.32 & 86.31$\pm$7.93 & 53.17$\pm$6.62 & 54.31$\pm$12.95 & \textbf{17.49$\pm$16.76} & 83.86$\pm$9.84 & 47.80$\pm$4.28 & 55.50$\pm$7.80 \\
(lstm+add)     & 21.81$\pm$17.80 & 83.29$\pm$9.11 & 55.54$\pm$7.83 & 60.18$\pm$11.12 & 24.18$\pm$13.96 & 83.05$\pm$12.68 & 49.32$\pm$3.63 & 58.42$\pm$6.69 \\
(lstm+dot)     & 17.59$\pm$16.86 & 84.06$\pm$10.97 & \textbf{48.17$\pm$5.42} & 54.83$\pm$12.86 & 25.18$\pm$16.46 & 82.39$\pm$11.17 & 50.66$\pm$4.95 & 55.07$\pm$7.69 \\

\bottomrule
\end{tabular}%
}
\caption{Ablation: Zero-shot NOOP\% (mean $\pm$ std), trained on $2\times3$, tested on $3\times3$ and $4\times4$. Bold indicates statistically significant minimum (lowest idle rate) within each setup (Shapiro-Wilk-adaptive two-sided test; Bonferroni-corrected, $p{<}0.0125$).}

\label{tab:app_abl_zs_noop}
\end{table*}

\begin{table*}[!t]
\centering\setlength{\tabcolsep}{1.4pt}\renewcommand{\arraystretch}{1.02}
\resizebox{\linewidth}{!}{%
\begin{tabular}{@{\hspace{2pt}}l|r|r|r|r||r|r|r|r@{\hspace{2pt}}}
\toprule
 & \multicolumn{8}{c}{Trained on $3\times3$} \\
\cmidrule(lr){2-9}
 & \multicolumn{4}{c||}{$4\times4$} & \multicolumn{4}{c}{$5\times5$} \\
\cmidrule(lr){2-5}\cmidrule(lr){6-9}
\shortstack[l]{PLATO\\Variation} & \multicolumn{1}{c|}{S0} & \multicolumn{1}{c|}{S1} & \multicolumn{1}{c|}{S2} & \multicolumn{1}{c||}{S3} & \multicolumn{1}{c|}{S0} & \multicolumn{1}{c|}{S1} & \multicolumn{1}{c|}{S2} & \multicolumn{1}{c}{S3} \\
\midrule
(mlp+add)      & 26.71$\pm$14.04 & 90.18$\pm$6.05 & 50.44$\pm$9.25 & 58.87$\pm$4.14 & 28.81$\pm$14.54 & 92.03$\pm$3.80 & 46.36$\pm$7.08 & 56.28$\pm$7.37 \\
(mlp+dot)      & 36.33$\pm$20.78 & 91.28$\pm$3.99 & 49.82$\pm$9.14 & 53.23$\pm$3.30 & 22.62$\pm$14.49 & 91.60$\pm$3.95 & 46.09$\pm$7.19 & 58.09$\pm$7.07 \\
(lstm+add)     & \textbf{20.94$\pm$14.18} & \textbf{79.13$\pm$15.81} & 50.79$\pm$9.95 & 53.46$\pm$3.83 & 19.66$\pm$13.38 & 90.93$\pm$3.71 & \textbf{45.45$\pm$6.54} & 54.66$\pm$6.88 \\
(lstm+dot)     & 25.68$\pm$16.11 & 91.05$\pm$3.92 & 55.97$\pm$8.70 & 53.08$\pm$2.95 & 26.99$\pm$13.88 & \textbf{89.03$\pm$6.16} & 45.53$\pm$5.71 & 55.95$\pm$7.47 \\

\bottomrule
\end{tabular}%
}
\caption{Ablation: Zero-shot NOOP\% (mean $\pm$ std), trained on $3\times3$, tested on $4\times4$ and $5\times5$. Bold indicates statistically significant minimum (lowest idle rate) within each setup (Shapiro-Wilk-adaptive two-sided test; Bonferroni-corrected, $p{<}0.0125$).}
\label{tab:app_abl_zs_noop_from3x3}
\end{table*}

\begin{table*}[!t]
\centering\setlength{\tabcolsep}{1.4pt}\renewcommand{\arraystretch}{1.02}
\resizebox{\linewidth}{!}{%
\begin{tabular}{@{\hspace{2pt}}l|r|r|r|r||r|r|r|r@{\hspace{2pt}}}
\toprule
 & \multicolumn{8}{c}{Trained on $2\times3$} \\
\cmidrule(lr){2-9}
 & \multicolumn{4}{c||}{$3\times3$} & \multicolumn{4}{c}{$4\times4$} \\
\cmidrule(lr){2-5}\cmidrule(lr){6-9}
\shortstack[l]{PLATO\\Variation} & \multicolumn{1}{c|}{S0} & \multicolumn{1}{c|}{S1} & \multicolumn{1}{c|}{S2} & \multicolumn{1}{c||}{S3} & \multicolumn{1}{c|}{S0} & \multicolumn{1}{c|}{S1} & \multicolumn{1}{c|}{S2} & \multicolumn{1}{c}{S3} \\
\midrule
(mlp+add)      & 0.01$\pm$0.08 & 0.74$\pm$0.73 & 3.97$\pm$0.20 & 4.63$\pm$0.48 & 0.57$\pm$0.63 & 3.63$\pm$0.89 & 6.01$\pm$0.12 & 6.53$\pm$0.50 \\
(mlp+dot)      & 0.00$\pm$0.00 & 0.92$\pm$0.76 & 4.00$\pm$0.00 & 4.63$\pm$0.49 & \textbf{0.26$\pm$0.48} & 3.79$\pm$0.89 & 6.02$\pm$0.14 & 6.47$\pm$0.50 \\
(lstm+add)     & 0.04$\pm$0.23 & 0.63$\pm$0.69 & 4.00$\pm$0.48 & 4.45$\pm$0.50 & 0.73$\pm$0.69 & 4.08$\pm$0.79 & 6.01$\pm$0.08 & \textbf{6.23$\pm$0.42} \\
(lstm+dot)     & 0.03$\pm$0.20 & 0.83$\pm$0.86 & \textbf{3.77$\pm$0.45} & 4.61$\pm$0.49 & 0.92$\pm$0.74 & 3.57$\pm$0.81 & 6.03$\pm$0.21 & 6.36$\pm$0.50 \\

\bottomrule
\end{tabular}%
}
\caption{Ablation: Zero-shot burnouts per episode (mean $\pm$ std), trained on $2\times3$, tested on $3\times3$ and $4\times4$. Bold indicates statistically significant minimum within each setup (Shapiro-Wilk-adaptive two-sided test; Bonferroni-corrected, $p{<}0.0125$).}
\label{tab:app_abl_zs_burnout}
\end{table*}

\begin{table*}[!t]
\centering\setlength{\tabcolsep}{1.4pt}\renewcommand{\arraystretch}{1.02}
\resizebox{\linewidth}{!}{%
\begin{tabular}{@{\hspace{2pt}}l|r|r|r|r||r|r|r|r@{\hspace{2pt}}}
\toprule
 & \multicolumn{8}{c}{Trained on $3\times3$} \\
\cmidrule(lr){2-9}
 & \multicolumn{4}{c||}{$4\times4$} & \multicolumn{4}{c}{$5\times5$} \\
\cmidrule(lr){2-5}\cmidrule(lr){6-9}
\shortstack[l]{PLATO\\Variation} & \multicolumn{1}{c|}{S0} & \multicolumn{1}{c|}{S1} & \multicolumn{1}{c|}{S2} & \multicolumn{1}{c||}{S3} & \multicolumn{1}{c|}{S0} & \multicolumn{1}{c|}{S1} & \multicolumn{1}{c|}{S2} & \multicolumn{1}{c}{S3} \\
\midrule
(mlp+add)      & 1.09$\pm$0.67 & 3.83$\pm$0.92 & 6.05$\pm$0.61 & 6.01$\pm$0.18 & 2.06$\pm$0.74 & 4.53$\pm$1.00 & 6.29$\pm$0.47 & 6.89$\pm$0.55 \\
(mlp+dot)      & 0.96$\pm$0.57 & 3.63$\pm$0.89 & 6.05$\pm$0.58 & 6.01$\pm$0.14 & \textbf{1.62$\pm$0.71} & \textbf{4.13$\pm$0.85} & 6.29$\pm$0.48 & 7.06$\pm$0.49 \\
(lstm+add)     & 0.99$\pm$0.73 & 4.40$\pm$0.94 & 6.15$\pm$0.57 & 6.03$\pm$0.21 & 1.99$\pm$0.69 & 4.75$\pm$1.07 & 6.25$\pm$0.44 & 6.77$\pm$0.61 \\
(lstm+dot)     & 1.03$\pm$0.67 & 4.06$\pm$0.77 & 6.28$\pm$0.68 & 5.99$\pm$0.08 & 2.07$\pm$0.77 & 5.24$\pm$0.79 & 6.19$\pm$0.39 & 6.71$\pm$0.59 \\

\bottomrule
\end{tabular}%
}
\caption{Ablation: Zero-shot burnouts per episode (mean $\pm$ std), trained on $3\times3$, tested on $4\times4$ and $5\times5$. Bold indicates statistically significant minimum within each setup (Shapiro-Wilk-adaptive two-sided test; Bonferroni-corrected, $p{<}0.0125$).}
\label{tab:app_abl_zs_burnout_from3x3}
\end{table*}

\begin{table*}[!t]
\centering\setlength{\tabcolsep}{1.4pt}\renewcommand{\arraystretch}{1.02}
\resizebox{\linewidth}{!}{%
\begin{tabular}{@{\hspace{2pt}}l|r|r|r|r||r|r|r|r@{\hspace{2pt}}}
\toprule
 & \multicolumn{8}{c}{Trained on $2\times3$} \\
\cmidrule(lr){2-9}
 & \multicolumn{4}{c||}{$3\times3$} & \multicolumn{4}{c}{$4\times4$} \\
\cmidrule(lr){2-5}\cmidrule(lr){6-9}
\shortstack[l]{PLATO\\Variation} & \multicolumn{1}{c|}{S0} & \multicolumn{1}{c|}{S1} & \multicolumn{1}{c|}{S2} & \multicolumn{1}{c||}{S3} & \multicolumn{1}{c|}{S0} & \multicolumn{1}{c|}{S1} & \multicolumn{1}{c|}{S2} & \multicolumn{1}{c}{S3} \\
\midrule
(mlp+add)      & 4.99$\pm$0.08 & 4.15$\pm$0.74 & 46.35$\pm$6.24 & 11.02$\pm$13.65 & 6.43$\pm$0.63 & 6.29$\pm$4.51 & 41.06$\pm$7.73 & 17.37$\pm$17.82 \\
(mlp+dot)      & 5.00$\pm$0.00 & 4.05$\pm$0.76 & 45.09$\pm$6.01 & 10.79$\pm$13.33 & \textbf{6.74$\pm$0.48} & 5.83$\pm$3.83 & 40.61$\pm$8.12 & 19.39$\pm$17.95 \\
(lstm+add)     & 4.96$\pm$0.23 & 3.89$\pm$0.72 & 41.89$\pm$10.80 & \textbf{17.61$\pm$13.37} & 6.27$\pm$0.69 & \textbf{7.77$\pm$4.56} & 42.43$\pm$5.79 & \textbf{29.22$\pm$15.87} \\
(lstm+dot)     & 4.97$\pm$0.20 & 3.99$\pm$0.87 & 47.21$\pm$8.08 & 11.64$\pm$13.83 & 6.08$\pm$0.74 & 5.85$\pm$2.83 & 41.91$\pm$8.93 & 23.13$\pm$17.72 \\

\bottomrule
\end{tabular}%
}
\caption{Ablation: Zero-shot putouts per episode (mean $\pm$ std), trained on $2\times3$, tested on $3\times3$ and $4\times4$. Bold indicates statistically significant best performance within each setup (Shapiro-Wilk-adaptive two-sided test; Bonferroni-corrected, $p{<}0.0125$).}

\label{tab:app_abl_zs_putout}
\end{table*}

\begin{table*}[!t]
\centering\setlength{\tabcolsep}{1.4pt}\renewcommand{\arraystretch}{1.02}
\resizebox{\linewidth}{!}{%
\begin{tabular}{@{\hspace{2pt}}l|r|r|r|r||r|r|r|r@{\hspace{2pt}}}
\toprule
 & \multicolumn{8}{c}{Trained on $3\times3$} \\
\cmidrule(lr){2-9}
 & \multicolumn{4}{c||}{$4\times4$} & \multicolumn{4}{c}{$5\times5$} \\
\cmidrule(lr){2-5}\cmidrule(lr){6-9}
\shortstack[l]{PLATO\\Variation} & \multicolumn{1}{c|}{S0} & \multicolumn{1}{c|}{S1} & \multicolumn{1}{c|}{S2} & \multicolumn{1}{c||}{S3} & \multicolumn{1}{c|}{S0} & \multicolumn{1}{c|}{S1} & \multicolumn{1}{c|}{S2} & \multicolumn{1}{c}{S3} \\
\midrule
(mlp+add)      & 5.91$\pm$0.67 & 6.45$\pm$4.51 & 33.63$\pm$12.27 & 38.69$\pm$7.04 & 5.94$\pm$0.74 & 5.77$\pm$2.64 & 50.69$\pm$10.43 & 36.09$\pm$12.64 \\
(mlp+dot)      & 6.04$\pm$0.57 & 5.61$\pm$2.32 & 33.84$\pm$12.79 & 41.53$\pm$6.41 & \textbf{6.38$\pm$0.71} & 5.13$\pm$1.41 & 48.23$\pm$9.81 & 31.28$\pm$11.02 \\
(lstm+add)     & 6.01$\pm$0.73 & \textbf{8.39$\pm$5.27} & 29.99$\pm$13.09 & 39.35$\pm$8.70 & 6.01$\pm$0.69 & 5.40$\pm$2.34 & 50.95$\pm$10.61 & 38.76$\pm$12.87 \\
(lstm+dot)     & 5.97$\pm$0.67 & 5.71$\pm$2.23 & 28.13$\pm$13.99 & 41.15$\pm$4.99 & 5.93$\pm$0.77 & 7.21$\pm$3.36 & \textbf{54.33$\pm$9.87} & 37.95$\pm$12.36 \\

\bottomrule
\end{tabular}%
}
\caption{Ablation: Zero-shot putouts per episode (mean $\pm$ std), trained on $3\times3$, tested on $4\times4$ and $5\times5$. Bold indicates statistically significant best performance within each setup (Shapiro-Wilk-adaptive two-sided test; Bonferroni-corrected, $p{<}0.0125$).}
\label{tab:app_abl_zs_putout_from3x3}
\end{table*}

\paragraph{Zero-Shot Burnouts (Tables~\ref{tab:app_abl_zs_burnout} and~\ref{tab:app_abl_zs_burnout_from3x3}).}
Burnout counts are near-maximal across variants in S2--S3 at $4\times4$ ($\approx 6.0$) regardless of training source, reflecting the difficulty of fire control at scale. Trained on $2\times3$, mlp+dot achieves statistically significant fewest burnouts in S0 at $4\times4$ ($0.26$); lstm+dot achieves statistically significant fewest in S2 at $3\times3$ ($3.77$); lstm+add achieves statistically significant fewest in S3 at $4\times4$ ($6.23$). Trained on $3\times3$, mlp+dot achieves statistically significant fewest burnouts in S0--S1 at $5\times5$.

\paragraph{Zero-Shot Putouts (Tables~\ref{tab:app_abl_zs_putout} and~\ref{tab:app_abl_zs_putout_from3x3}).}
Trained on $2\times3$, mlp+dot achieves statistically significant best putouts in S0 at $4\times4$ ($6.74$); lstm+add achieves statistically significant best putouts in S3 at $3\times3$ ($17.61$) and S1 and S3 at $4\times4$. Trained on $3\times3$, mlp+dot achieves statistically significant best putouts in S0 at $5\times5$ ($6.38$); lstm+add achieves significantly most putouts at $4\times4$ S1 ($8.39$); lstm+dot achieves best putouts at $5\times5$ S2 ($54.33$, significant).

\paragraph{Zero-Shot NOOP Usage (Tables~\ref{tab:app_abl_zs_noop} and~\ref{tab:app_abl_zs_noop_from3x3}).}
NOOP rates are high in Setup~1 ($\approx 80\%$--$91\%$) across all variants and training sources, reflecting task sparsity. Trained on $2\times3$, mlp+dot achieves statistically significant lowest NOOP in S0 at $4\times4$ ($17.49\%$), and lstm+dot achieves lowest in S2 at $3\times3$ ($48.17\%$). Trained on $3\times3$, lstm+add achieves significantly lowest NOOP in S0--S1 at $4\times4$ and S2 at $5\times5$; lstm+dot achieves significantly lowest NOOP in S1 at $5\times5$ ($89.03\%$).


\end{document}